\documentclass[11pt]{amsart}
\usepackage{geometry}                
\usepackage{graphicx,calc}
\usepackage{amssymb,amsthm,amsmath,stmaryrd}
\usepackage{enumerate} 
\usepackage{hyperref}   
\usepackage{xcolor}   
\usepackage{epstopdf}
\usepackage[numbers]{natbib}
\usepackage{dirtytalk}
\usepackage{xcolor}
\newcommand{\modif}[1]{\textcolor{black}{#1}}
\newcommand{\modiff}[1]{\textcolor{black}{#1}}

\title[Identifiability of ReLU networks]{An embedding of ReLU networks\\
 and 
 an analysis of their identifiability}

\author{Pierre Stock and Rémi Gribonval}

\newtheorem{lemma}{Lemma}
\newtheorem{corollary}{Corollary}
\newtheorem{theorem}{Theorem}
\newtheorem{definition}{Definition}
\newtheorem{example}{Example}
\newtheorem{fact}{Fact}
\newtheorem{remark}{Remark}

\newcommand\RR[0]{\mathbb{R}}
\newcommand\pathset[1]{\mathcal{#1}}
\renewcommand\vec[1]{\boldsymbol{#1}}
\newcommand\mat[1]{\boldsymbol{#1}}
\newcommand\ReLU[0]{\mathtt{ReLU}}
\newcommand\indic[1]{\mathtt{1}_{#1}}

\newcommand\srmap[0]{{\Phi}}
\newcommand\rmap[0]{\vec{\Phi}}
\newcommand\rmapi[0]{\rmap^{\mathtt{i}}}
\newcommand\rmaph[0]{\rmap^{\mathtt{h}}}
\newcommand\permset[1]{\mathfrak{S}_{#1}}
\def\permmat{\mat{\Pi}}

\newcommand\xset[0]{\mathcal{X}}
\newcommand\xsetcont[0]{\xset_{\theta}}

\newcommand\zset[0]{\mathcal{Z}_{\theta}}
\def\va {\vec{a}}
\def\bva {\bar{\vec{a}}}
\def\vw {\vec{w}}
\def\vb {\vec{b}}
\def\vc {\vec{c}}
\def\vd {\vec{d}}
\def\vv {\vec{v}}
\def\vy {\vec{y}}
\def\vz {\vec{z}}
\def\vu {\vec{u}}

\newcommand\supp[1]{\mathtt{supp}(#1)}
\newcommand\sign[1]{\mathtt{sign}(#1)}
\newcommand\abs[1]{|#1|}

\DeclareMathOperator*{\Circ}{\bigcirc}
\def\dim {\mathtt{dim}}
\def\actdim {\mathtt{actdim}}

\def\ker {\mathtt{ker}}
\newcommand\linspan[1]{\mathtt{span}\left\{#1\right\}}

\def\actneuron{a}
\def\actlayer{\va}
\def\bactlayer{\bva}
\def\act{\va}
\def\actpath{\alpha}
\def\vactpath{\vec{\alpha}}
\def\bvactpath{\bar{\vec{\alpha}}}
\def\pathmatrix{\mat{P}}

\def\restmatrix{\mat{Q}}

\newcommand\linspace[1]{\mathtt{#1}}

\newcommand\Aspace[0]{\linspace{A}(\theta)}
\newcommand\Aspacebias[0]{\linspace{\bar{A}}(\theta)}

\newcommand\Aspacebiasperp[0]{\linspace{\bar{A}}^{\perp}(\theta)}
\newcommand\Aspaceperp[0]{\linspace{A}^{\perp}(\theta)}

\newcommand\twinsignature[1]{\mathtt{s}_{#1}}

\newcommand\rspace[0]{\RR^{\pathset{P}}}

\newcommand\card[1]{\mathtt{card}(#1)}

\newcommand\realiz[1]{\vec{R}_{#1}}
\newcommand\linopreal[0]{\mat{C}}

\newcommand\hidden[0]{H}
\newcommand\biasnodes[0]{\bar{\hidden}}
\newcommand\paramset[0]{E \cup \biasnodes}
\newcommand\Thetairred[0]{{\Theta_{\mathtt{irr}}}}

\begin{document}
\maketitle
%

\begin{abstract}
Neural networks with the Rectified Linear Unit (ReLU) nonlinearity are described by a vector of parameters $\theta$, and realized as a piecewise linear continuous function $\realiz{\theta}: x \in \RR^{d} \mapsto \realiz{\theta}(x) \in \RR^{k}$.
Natural scalings and permutations operations on the parameters $\theta$ leave the realization unchanged, leading to equivalence classes of parameters that yield the same realization. 
These considerations in turn lead to the notion of identifiability -- the ability to recover (the equivalence class of) $\theta$ from the sole knowledge of its realization $\realiz{\theta}$. 
The overall objective of this paper is to introduce an embedding for ReLU neural networks of any depth, $\rmap(\theta)$, that is invariant to scalings and that provides a  locally linear parameterization of the realization of the network. Leveraging these two key properties, we derive some conditions under which a deep ReLU network is indeed locally identifiable from the knowledge of the realization on a finite set of samples $x_{i} \in \RR^{d}$. We study the shallow case in more depth, establishing necessary and sufficient conditions for the network to be identifiable from a bounded subset $\xset \subseteq \RR^{d}$.  
\end{abstract}

\tableofcontents
\section{Introduction}

The empirical success of Deep Neural Networks (DNNs) for traditional machine learning tasks such as image classification is a well-known fact for the research community \cite{krizhevsky2012deep}. While this empirical success percolates to areas ranging from protein folding to symbolic mathematics, a second well-known fact is that the theoretical tools to grasp DNNs and uncover the reasons of their success are still lagging behind the fast-paced experimental results. We argue that a deeper understanding of the expressivity and stability properties of such networks could lead to practical improvements \cite{daubechies2019nonlinear,devore2020neural}. In this paper, we introduce an embedding, $\rmap(\theta)$, of the vector $\theta$ of network parameters (weights and biases) that exhibits interesting properties for networks based on the popular Rectified Linear Unit (ReLU): in particular, $\rmap(\theta)$ is invariant to natural rescalings of the parameters that leave unchanged the function implemented by the network. To showcase the potential of this tool, we leverage it to study the expressivity of DNNs from the perspective of their functional equivalence classes.

In the remainder of this section, we first list the papers tackling identifiability of neural networks with a given non-linearity function. We next present some related work that construct an embedding for ReLU networks. Finally, we list applications directly or indirectly derived from the two previous theoretical considerations. 

\pagebreak
\subsection*{Around Functional Identifiability}

First, various results dating back from the 90's identify conditions that allow to identify neural networks \emph{with one hidden layer} equipped with various non-linearities\footnote{It should be noted that the functional equivalence class generally depends on the considered non-linearity. For instance, with the hyperbolic tangent, the authors only consider permutations and sign flips.} like the hyperbolic tangent~\cite{sussman,kainen1994uniqueness,kurkova,albertini}. Such results \emph{do not} encompass the ReLU case. Simultaneously, Fefferman derived identifiability conditions for deep networks
 equipped with the $\tanh$ nonlinearity using complex analysis~\cite{Fefferman1994}.
More recently, the work of Rolnick and Kording~\cite{rolnick2019reverseengineering} reflects a renewed interest for this subject and its application to ReLU networks. The authors propose to reverse-engineer deep ReLU networks and present a constructive algorithm that samples network realizations $\realiz{\theta}(x)$ for carefully chosen input points $x$ to deduce the architecture of the network and its parameters, up to rescalings and permutations. The authors prove that their algorithm terminates, except for a measure-zero set of parameters\footnote{This measure-zero set of parameters is not explicitly described by the authors.}. Similarly, Fornasier \textit{et al.} \cite{fornasier2019robust} propose to recover the parameters of a \emph{two-hidden-layer} neural network with smooth nonlinearity by actively sampling finite difference approximations to Hessians of the network, and by combining the insights gained from the sampling with a heuristic for precise attribution of the parameters to the architecture. The authors demonstrate the empirical effectiveness of their approach and claim that the proposed method can be generalized to networks with any depth. 
Finally, Phuong and Lampert provide a result related to identifiability for ReLU networks under some more restrictive assumptions~\cite{phuong_functional_2020}.

\subsection*{Embeddings Zoology}

We list here the neural network embeddings in the literature that are the closest to our own embedding, $\rmap(\theta)$, which is introduced in Definition~\ref{def:Representation}. Its main property is that it is invariant under the action of rescalings, as stated more formally in Theorem~\ref{lem:EqualRepAndSignImpliesSEquiv}. Schematically, $\rmap(\theta)$ lives in the linear space indexed by network \emph{paths} and each coordinate is a product of weights and/or biases along a particular network path. In a similar fashion, Malgouyres and Landsberg~\cite{malgouyres2018multilinear} consider a particular class of \emph{linear} structured networks called \emph{Deep Structured Neural Networks}, without biases, and consider only layer-wise rescalings. In \cite[Section 6]{malgouyres2018multilinear}, the authors provide sufficient and necessary conditions for local identifiability by studying complex algebraic varieties leveraging the \emph{Segre embedding} of such networks. The Segre embedding bears a resemblance with $\rmap(\theta)$ since it is also made of product of network parameters, but it does not encompass the biases. Moreover, we consider neuron-wise rescalings in this paper as opposed to less general layer-wise rescalings considered by the authors, and also emcompass the ReLU non-linearity in our approach. Malgouyres later leverages the Segre embedding to study local stability properties of \modiff{sparse} neural networks \cite{malgouyres2020stable}. Finally, Neyshabur \emph{et al.} introduce a family of \emph{path regularizers} to derive an optimization procedure that takes the invariance of the realization $\realiz{\theta}$ under the action of the rescalings into account and called Path-SGD~\cite{neyshabur_norm-based_2015,neyshabur_path-sgd_2015}. Such path regularizers are scalars -- as opposed to vectorial embeddings -- that are obtained by summing, for any network path, the norm of the product of all weights along this paths. Moreover, this approach does not take the biases into account, as opposed to our embedding, $\rmap(\theta)$.

\subsection*{Applications}

As illustrated with Path-SGD~\cite{neyshabur_path-sgd_2015}, several papers attempt to perform the optimization in the space of network parameters \emph{quotiented} by the rescaling operation. Several work follow and perform the optimization by alternating between a standard SGD step and a projection step that modifies the rescaling coefficients without changing the function implemented by the network~\cite{meng_g-sgd_2019,yi_positively_2019,stock_equi-normalization_2019,yuan_scaling-based_2019}. The main difference between these papers is the projection step, that is performed either implicitly (with a regularizer) or explicitly (by computing the optimal\footnote{In the sense that such coefficients globally minimize a given objective function.} rescaling coefficients). In the latter case, the proposed empirical methods may not yield the optimal rescaling coefficients but rather more or less stable and good approximations. Another advantage of rescalings is to improve post-training scalar quantization of neural networks by carefully selecting the rescaling coefficients such that the \modif{dynamic} range of the weights \modif{within} a layer is relatively small, with as few outliers as possible~\cite{meller_same_2019,nagel_data-free_2019}. More related to the concept of (local) identifiability, Carlini \emph{et al.}~\cite{carlini2020cryptanalytic} design a differential attack to efficiently recover the parameters of remote model up to floating point precision, by sending carefully designed queries $x$ to the remote network and receiving only its output. 

\medskip 

After introducing the main notations, we define $\rmap(\theta)$ and state the main results of the paper in Section~\ref{sec:general}. Then, we formally state and prove the main properties of the embedding $\rmap(\theta)$ in Section~\ref{sec:embedding}. In particular, we prove that $\rmap(\theta)$ is invariant under the action of the rescalings. Next, we leverage this embedding to derive partial and local identifiability results for ReLU neural networks of any depth in Section~\ref{sec:technical_tools}. To further demonstrate the validity of our approach, we fully study the shallow case in Section~\ref{sec:shallowcase} and provide conditions under which a ReLU neural network with one hidden layer is identifiable. Finally, we argue that $\rmap(\theta)$ may be leveraged to tackle other open problems in the Machine Learning community. 
\section{General setting and main results}
\label{sec:general}

We consider fully-connected  feedforward ReLU neural networks with $L \geq 2$ affine layers. 
Each network is supported on a graph $G = (E,V)$ with vertex set $V$ composed of neurons $\nu$ and edge set $E$ composed of connections. The set of neurons $V$ is partitioned into the input layer $N_{0}$, $L-1$ hidden layers $N_{\ell}$, $1 \leq \ell \leq L-1$, and the output layer $N_{L}$. Hidden neurons compose the set $\hidden = \cup_{\ell=1}^{L-1} N_{\ell}$. Since we focus on fully-connected networks, the set of connections $E$ is made of all oriented edges $e = \nu \to \nu'$ between neurons belonging to consecutive layers, $\nu \in N_{\ell-1}$, $\nu' \in N_{\ell}$ for some $1 \leq \ell \leq L$. The subset of incoming edges of neuron $\nu$ is denoted $\bullet \to \nu$, while $\nu \to \bullet$ denotes its set of outgoing edges.

Each edge $e \in E$ is equipped with a weight $w_{e}$ and each hidden neuron $\nu \in H$ with a bias $b_{\nu}$. Output neurons, i.e. neurons from the last layer $\eta \in N_{L}$, are also equipped with a bias $b_{\eta}$, which is sometimes constrained to be zero. The set of all neurons equipped with biases is $\biasnodes := \hidden \cup N_{L}$. Parameters (weights and biases) are gathered in a parameter vector $\theta \in \RR^{\paramset}$ where  $\paramset$ indexes all possible weights and biases including biases on the output layer. For brevity we may denote $\theta_{e}$ for weights and $\theta_{\nu}$ for biases. 
When needed we also write $\theta = (\theta_{i})_{i \in \paramset}$.
Since we consider a fully connected architecture (this does not prevent some weights to possibly vanish on some edges), $\theta$ can also be represented as a set of $L$ matrices $\mat{W}_{\ell} = (w_{\nu \to \nu'})_{\nu' \in N_{\ell},\nu \in N_{\ell-1}} \in \RR^{N_{\ell} \times N_{\ell-1}}$, $1 \leq \ell \leq L$ and $L$ vectors $\vec{b}_{\ell} = (b_{\nu})_{\nu \in N_{\ell}} \in \RR^{N_{\ell}}$, $1 \leq \ell \leq L$.

\subsection{Network architectures}
Many of the notions of parameter identifiability or non-degeneracy that will be considered are relative to a choice of network ``architecture''. This is represented both by the graph $G$ (which determines how many layers there are, and how wide they are) but also by a possibly restricted set $\Theta \subseteq \RR^{\paramset}$ of network parameters, which may for example account for the following type of constraints:
\begin{itemize}
\item restricting to a convolutional structure;
\item restricting to sparse networks, possibly with structured sparsity patterns;
\item restricting to networks without output biases ($b_{\eta}=0$ for every $\eta \in N_{L}$);
\item restricting to networks without biases ($b_{\nu}=0$ for every $\nu \in H \cup N_{L}$).
\end{itemize}

\subsection{Realization of a network}
Given a parameter $\theta$ and an input vector $x \in \RR^{N_{0}}$, we sequentially define $\vec{y}_{0}(\theta,x) = x$ and for each $1 \leq \ell \leq L-1$ the pre-activation $\vec{z}_{\ell}(\theta,x) = \mat{W}_{\ell} \vec{y}_{\ell-1}(\theta,x)+\vec{b}_{\ell} \in \RR^{N_{\ell}}$, the post-activation $\vec{y}_{\ell}(\theta,x) = \ReLU(\vec{z}_{\ell}(\theta,x)) \in \RR^{N_{\ell}}$ where the rectified linear unit (ReLU) activation function, $\ReLU(t) = \max(t,0)$, is applied entrywise. Finally we define the realization of the network as the function $\realiz{\theta}: x \mapsto \realiz{\theta}(x) := \vec{z}_{L}(\theta,x)  = \mat{W}_{L} \vec{y}_{L-1}(\theta,x)+\vec{b}_{L} \in \RR^{N_{L}}$. When needed we will use \modif{neuron-wise} versions of these notations, e.g. $y_{\nu}(\theta,x) = (\vec{y}_{\ell}(\theta,x))_{\nu}$ where $\nu \in N_{\ell}$. Note the general convention to denote scalar-valued quantities in plain font to distinguish them from quantities that can be vector-valued, which are generally denoted in bold. 
\subsection{Invariance to permutation and scaling}

A well known fact \cite{neyshabur2015pathsgd} is that the realization of any ReLU-network is invariant to permutations and scalings of the parameter $\theta$. The invariance to permutations is not specific to ReLU-networks, while the scaling-invariance is due to the homogeneity of the ReLU: $\ReLU(\lambda \cdot) = \lambda \ReLU(\cdot)$ for every $\lambda>0$ and is also valid for other variants such as the leaky-ReLU. While various definitions coexist in the literature \cite{nagel2019datafree,meller2019different,yi2019positively},
it is convenient to focus first on the practical \emph{per-neuron} rescaling equivalence \cite{neyshabur2015pathsgd} as stated below.

\paragraph{\bf Rescaling equivalence}
Let $\nu\in H$ and $\lambda_\nu  > 0$. A neuron-wise scaling multiplies the incoming weights and the bias of $\nu$ by $\lambda_{\nu}$, and divides the outgoing weights by $\lambda_{\nu}$. It is formally defined as $s_{\nu, \lambda_\nu}: \theta = (w,b) \mapsto \theta' = (w',b')$ where for every connection $e \in E$, 
\begin{equation}\label{eq:DefNeuronScaling}
\forall e \in E, \quad
  w'_e = 
 \begin{cases}
    w_{e} \lambda_\nu & \mbox{if } e \in \bullet \to \nu \\
    \frac{1}{\lambda_\nu} w_{e} & \mbox{if } e \in \nu \to \bullet \\
    w_{e} & \mbox{otherwise},
  \end{cases}
  \qquad
  \forall \nu \in H,\quad b'_\nu =  b_\nu \lambda_\nu.
\end{equation}
Let $\mathcal S$ be the set of neuron-wise scalings. We observe that neuron-wise rescalings commute and are invertible, the inverse of $s_{\nu, \lambda_\nu}$ being $s_{\nu, 1/\lambda_\nu}$. Let $\langle \mathcal S \rangle$ be the commutative group generated by $\mathcal S$. Every $s \in \langle \mathcal S \rangle$ can be uniquely represented as the composition
\begin{equation*}\label{eq:resc_group}
  s = \Circ_{\nu \in H} s_{\nu, \lambda_\nu}
\end{equation*}
where the $\lambda_\nu$ are strictly positive. Note that in this representation, every hidden neuron $\nu$ is associated to exactly one neuron-wise rescaling $\lambda_\nu$.

\begin{definition}\label{def:resc_neuron}
  $\theta$ and $\theta'$ are rescaling equivalent if there exists $s\in \langle \mathcal S \rangle$ such that $\theta' = s(\theta)$. We then denote $\theta \sim_S \theta'$. 
\end{definition}
Notice that if $\theta' \sim_{S} \theta$, then the output biases are equal: $\theta'_{\eta}=\theta_{\eta}$ for all $\eta \in N_{L}$. 
\begin{fact}
$\theta' \sim_{S}\theta$ if, and only if, there exists diagonal matrices $\mat{\Lambda}_{\ell} \in \RR^{N_{\ell} \times N_{\ell}}$ with positive entries, $0 \leq \ell \leq L$ such that $\mat{\Lambda}_{0}=\mat{I}_{N_{0}}$,  $\mat{\Lambda}_{L}= \mat{I}_{N_{L}}$, and for every layer $1 \leq \ell \leq L$ 
\begin{equation}
\mat{W}'_{\ell} = \mat{\Lambda}_{\ell} \mat{W}_{\ell} \mat{\Lambda}_{\ell-1}^{-1}
\ \text{and}\ \vb'_{\ell} = \mat{\Lambda}_{\ell} \vb_{\ell}.
\end{equation}
\end{fact}
\paragraph{\bf Permutation equivalence}
Consider $\pi := (\pi_{1},\ldots,\pi_{\ell})$ where $\pi_{\ell} \in \permset{N_{\ell}}$ is a permutation of the $\ell$-th hidden layer (input and output layers are never permuted), $1 \leq \ell \leq L-1$. Denote $\permset{G} = \permset{N_{1}} \times \ldots \times \permset{N_{L-1}}$ the group of all such tuples of permutations. One can define a natural action of the group $\permset{G}$ on parameterizations via $\theta \mapsto \pi \circ \theta := \theta'$ where each weight matrix $\mat{W}'_{\ell}$ is obtained from  $\mat{W}_{\ell}$ by permuting rows according to $\pi_{\ell}$ and columns according to $\pi_{\ell-1}$, while bias vector $\vec{b}'_{\ell}$ is a permuted version of $\vec{b}_{\ell}$ according to $\pi_{\ell}$. 

\begin{definition}\label{def:PSequiv}
Two parameters $\theta,\theta'$ are \emph{permutation-equivalent}  if, and only if, there exists $\pi \in \permset{G}$ such that $\theta' = \pi \circ \theta$. This is denoted $\theta \sim_{P} \theta'$.\\
 The parameters are  \emph{permutation-scaling equivalent} if, and only if, there exists $\theta''$ such that $\theta \sim_{S} \theta'' \sim_{P} \theta'$. This is denoted $\theta \sim_{PS} \theta'$.\\
The parameters are  \emph{scaling-permutation equivalent} if, and only if, there exists $\theta''$ such that  $\theta \sim_{P} \theta'' \sim_{S} \theta'$. This is denoted $\theta \sim_{SP} \theta'$.\\
\end{definition}
\begin{fact}
$\theta' \sim_{PS} \theta$ if, and only if, $\theta' \sim_{SP} \theta$, if and only if there exists diagonal matrices $\mat{\Lambda}_{\ell} \in \RR^{N_{\ell}\times N_{\ell}}$ with positive entries and permutation matrices $\mat{\Pi}_{\ell} \in \RR^{N_{\ell} \times N_{\ell}}$, $0 \leq \ell \leq L$,
such that $\mat{\Pi}_{0} = \mat{\Lambda}_{0}=\mat{I}_{N_{0}}$, $\mat{\Pi}_{L} = \mat{\Lambda}_{L}= \mat{I}_{N_{L}}$, and for every layer $1 \leq \ell \leq L$ 
\begin{equation}\label{eq:PSequivalentParamsMatrixVersion}
\mat{W}'_{\ell} = \mat{\Pi}_{\ell}\mat{\Lambda}_{\ell} \mat{W}_{\ell} \mat{\Lambda}_{\ell-1}^{-1}\mat{\Pi}_{\ell-1}^{-1}
\ \text{and}\ \vb'_{\ell} = \mat{\Pi}_{\ell}\mat{\Lambda}_{\ell} \vb_{\ell}.
\end{equation}
\end{fact}
As widely documented
 \cite{neyshabur2015pathsgd,nagel2019datafree,meller2019different,yi2019positively}, PS-equivalent parameters share their realization as proven, e.g., in~\cite{rolnick2019reverseengineering}[Lemma 1].
\begin{lemma}\label{le:DirectResult}
For any $\theta,\theta' \in \RR^{\paramset}$, if $\theta' \sim_{PS} \theta$ then $\realiz{\theta'} = \realiz{\theta}$. 
\end{lemma}


A natural question is to determine 
 conditions for the {\em identifiability} of (the equivalence class up to scaling and permutation of) $\theta$ from $\realiz{\theta}$.
To be more specific, we consider identifiability with respect to a family of parameters $\Theta$, from a set $\xset$. A case of particular interest will be when $\xset$ is finite, in order to characterize whether $\theta$ can be recovered (up to scaling and permutations) from finitely many samples of the network realization $\realiz{\theta}$. 

\begin{definition}[PS-identifiability]\label{def:identif}
A parameter $\theta \in \Theta \subseteq \RR^{\paramset}$ is {\em PS-identifiable with respect to $\Theta$ from $\xset \subseteq \RR^{N_{0}}$} if for every $\theta' \in \Theta$, the equality $\realiz{\theta} = \realiz{\theta'}$ on $\xset$ implies $\theta' \sim_{PS} \theta$. When considering $\xset=\RR^{N_{0}}$, $\theta$ is simply said to be {\em PS-identifiable} with respect to $\Theta$.
When considering $\Theta = \RR^{\paramset}$, $\theta$ is simply said to be PS-identifiable (from $\xset$).
\end{definition}

A trivial observation is that if all outgoing weights of a hidden neuron are zero, then the realization of the network is unchanged under arbitrary modifications of the incoming weights and of the bias \modif{of this neuron}, hence the corresponding parameter $\theta$ \modif{cannot} be PS-identifiable with respect to $\Theta = \RR^{\paramset}$. A similar phenomenon occurs if all incoming weights to a hidden neuron are zero. This motivates the definition of admissible parameters and proves \autoref{lem:admnec} below.
\begin{definition}
$\theta$ is admissible if for each hidden neuron $\nu \in H$ we have $\vw_{\bullet \to \nu} \neq 0$ and $\vw_{\nu \to \bullet} \neq 0$. Equivalently, every hidden neuron belongs to a full path with nonzero weights. 
\end{definition}
\begin{lemma}\label{lem:admnec}
If $\theta$ is PS-identifiable \modiff{from $\xset$} with respect to $\Theta  = \RR^{\paramset}$, then it is admissible.
\end{lemma}

\subsection{An invariant embedding of ReLU networks}
\label{sec:representation}
The invariance with respect to (permutations and) scalings (\autoref{le:DirectResult}) calls for an invariant representation of equivalence classes of network parameters. A central tool is a representation $\rmap(\theta)$ mapping a network parameter $\theta \in \RR^{\paramset}$ to a vector $\rmap(\theta)$ in a space indexed by paths of the network, $\RR^{\pathset{P}}$.

Before going further let us formally introduce paths, as illustrated in Figure~\ref{fig:exPaths} 
\begin{definition}\label{def:pathsets}
The set $\pathset{P}_{\ell}$, $0 \leq \ell \leq L$ 
 (resp. $\pathset{Q}_{\ell}$, $1 \leq \ell \leq L-1$) consists of all partial paths from any neuron $\nu_{\ell} \in N_{\ell}$ to a neuron of the last (resp. {\em penultimate}) layer $\nu_{L} \in N_{L}$ (resp.  $\nu_{L-1} \in N_{L-1}$).
Any path $p \in \pathset{P}_{\ell}$ is written as a tuple $p=(p_\ell, \dots, p_L)$ where each $p_i \in V$ is a neuron. We say that $p$ is a full path if $\ell = 0$, that is, if $p$ connects the input and the output layers. 
We may write $p = p_\ell \to p_{\ell+1} \to \dots \to p_L$, as well as $p = \mu\to q\to\nu$ where $\mu = p_\ell \in N_{\ell}$, $\nu = p_L\in N_{L}$ and $q = (p_{\ell+1},\ldots,p_{L-1}) \in \pathset{Q}_{\ell+1}$. 
\end{definition}

\begin{figure}[htbp]
  \begin{center}
  \includegraphics[width=0.9\textwidth]{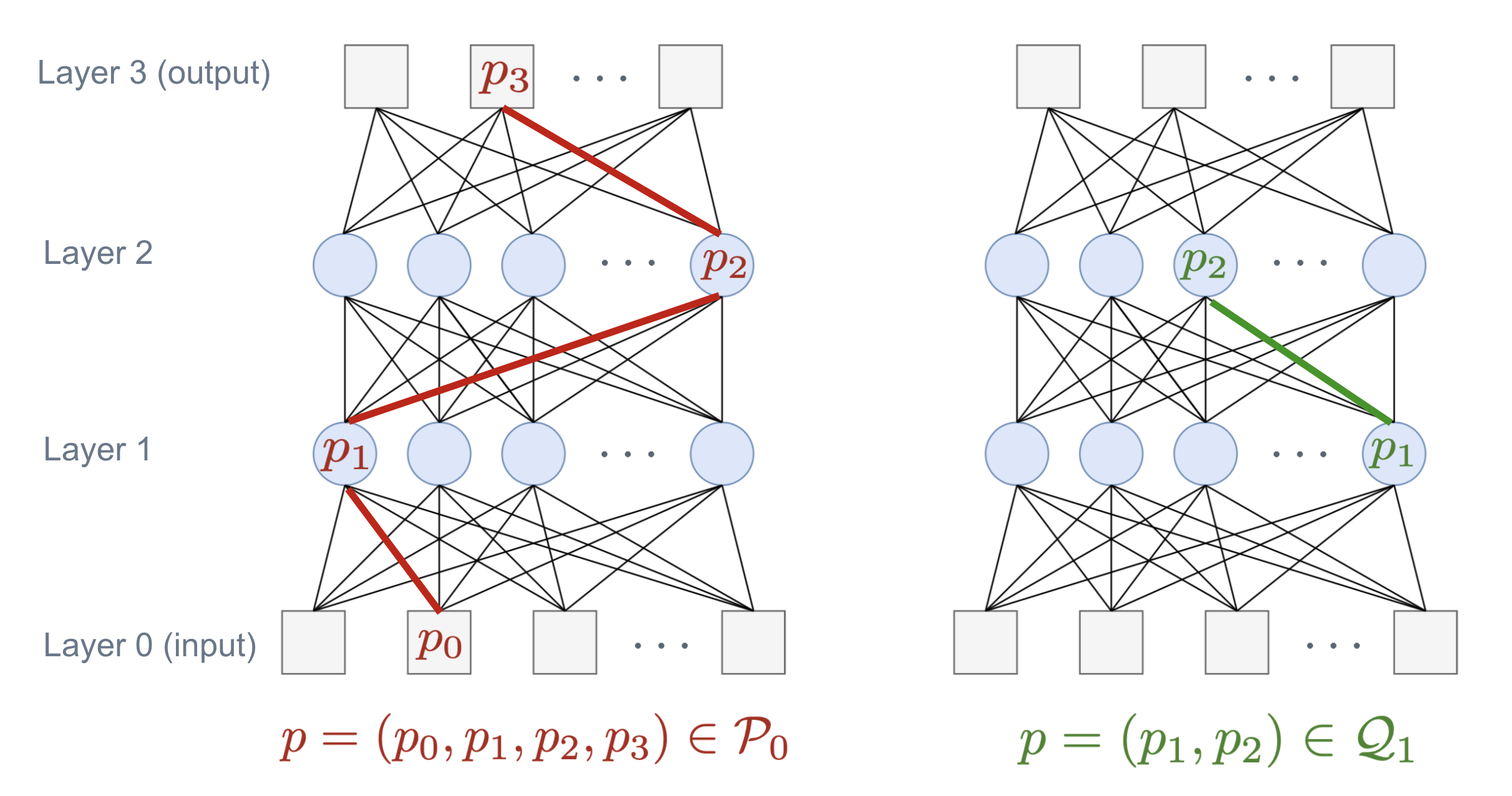}
  \caption{We consider a particular network architecture with $L=4$ layers (equivalently, with two hidden layers). Left: A particular path belonging to $\pathset{P}_0$. Right: A particular path belonging to $\pathset{Q}_1$.}
  \label{fig:exPaths}
  \end{center}
  \end{figure}
\modif{\begin{remark}\label{rk:belongstop}
  To streamline notations we say that an edge $e = \mu\to\nu \in E$ belongs to $p$ and also write $e \in p$ if there exists $\ell \leq i \leq L-1$ such that $\mu = p_i$ and $\nu = p_{i+1}$. Similarly, we choose to denote $\nu \in p$ if (and only if) the path $p$ \emph{starts} from neuron $\nu$, i.e., when $p = (p_{\ell},\ldots,p_{L}) \in \pathset{P}_{\ell}$,  if $p_{\ell} = \nu$. 
\end{remark}}

\modif{We next introduce the representation $\rmap(\cdot)$, which presents some connections with previous work \cite{malgouyres2018multilinear,malgouyres2020stable,neyshabur2015pathsgd} while being more generic as detailed in the introduction.
\begin{definition} \label{def:Representation}
Given $\theta \in \RR^{\paramset}$, the value of a path is
\begin{align}\label{eq:DefU}
\srmap_{p}(\theta) &= \Pi_{e \in p} \theta_{e},\ \text{for each full path}\ p \in \pathset{P}_{0},\\
\srmap_{p}(\theta) &= \theta_{p_{\ell}} \Pi_{e \in p} \theta_{e}, \text{for}\ p = (p_{\ell},\ldots,p_L) \in \pathset{P}_{\ell}, 1 \leq \ell \leq L.
\end{align}
For $p \in \pathset{P}_{L}$, $p = (\eta)$ with $\eta \in N_{L}$, $\srmap_{p}(\theta) = \theta_{p_{L}} = b_{\eta}$ is the corresponding output bias.\\
Define $\pathset{P} := \cup_{\ell=0}^{L}\pathset{P}_{\ell}$. For any $\theta \in \RR^{\paramset}$ we define
\begin{equation}\label{eq:DefRepresentation}
 \rmap(\theta) := (\srmap_{p}(\theta))_{p \in \pathset{P}} \in \rspace
 \end{equation}
\end{definition}}


This representation, \modif{combined with the entrywise sign of $\theta$ (with the convention $\sign{0} =0$),} characterizes the classes of scaling-equivalent admissible parameters.
\begin{theorem} \label{lem:EqualRepAndSignImpliesSEquiv}
Consider any $\theta',\theta \in \RR^{\paramset}$.
\begin{enumerate}[a)]
\item Assume that $\theta \sim_{S} \theta'$. Then $\rmap(\theta)= \rmap(\theta')$ and $\sign{\theta'}=\sign{\theta}$.
\item Assume that $\theta$ is admissible, that $\rmap(\theta')= \rmap(\theta)$, and that $\sign{\theta'_{E}}=\sign{\theta_{E}}$.\\ Then $\theta \sim_{S} \theta'$ and $\theta'$ is also admissible. 
\end{enumerate}
\end{theorem}
\modif{The proof is in Section~\ref{sec:embedding}}. A similar result is proven in \cite[Theorem 3.3]{meng2018mathcalgsgd} without considering the biases and by replacing the condition on the signs by a condition on the activation statuses of all partial paths, which depend on the input variable $x$ besides $\theta$.

\begin{remark}
The map $\theta \mapsto \rmap(\theta)$ will be referred to as an \emph{embedding} of network parameters. Stricto-sensu, as this map is not an injective function of network parameters, it does not match the definition of an embedding. However, since it characterizes equivalence classes of rescaling-equivalent admissible parameters, it can be used to define without ambiguity an embedding of these equivalence classes in $\RR^{\pathset{P}}$.
\end{remark}

\subsection{Some consequences of PS-identifiability}\label{sec:twoconsequences}
Using the embedding $\rmap(\cdot)$, we show that if $\theta$ is PS-identifiable then it is {\em locally} identifiable {\em up to scaling only}. 
Locality is measured in the sense of open balls $B(\vc,r) = \{\vc' : \|\vc'-\vc\|_{\infty} < r\}$, where the ambient linear space, equipped with the sup-norm, should always be clear from context.
\begin{definition}[
local S-identifiability]
\label{def:identif_local}
Given $\epsilon>0$, a parameter $\theta \in \Theta \subseteq \RR^{\paramset}$ is
 {\em $\epsilon$-locally S-identifiable from $\xset \subset \RR^{N_{0}}$ with respect to $\Theta$}, if for every $\theta' \in \Theta \cap B(\theta,\epsilon)$, the identity $\realiz{\theta} = \realiz{\theta'}$ on $\xset$ implies $\theta' \sim_{S} \theta$. 
 If there exists $\epsilon>0$ such that $\theta$ is 
 $\epsilon$-locally  S-identifiable  from $\xset$ then $\theta$ is {\em 
 locally  S-identifiable from $\xset$}. 
When $\xset=\RR^{N_{0}}$ and/or $\Theta = \RR^{\paramset}$ we adopt the same simplified terminology as with the notion of PS-identifiability.
\end{definition}
\begin{remark}
If $\theta$ is PS-identifiable (resp. locally S-identifiable) from $\xset \subseteq \RR^{N_{0}}$ with respect to $\Theta \subseteq \RR^{\paramset}$ then the same holds 
from any $\xset' \supseteq \xset$ with respect to any $\Theta' \subseteq \Theta$.
\end{remark}
Our first result is the following theorem. 
\begin{theorem}\label{th:IdentImpliesLocaIdent}
Consider $\Theta \subseteq \RR^{\paramset}$ and $\xset \subset \RR^{N_{0}}$.
If $\theta \in \Theta$ is admissible and PS-identifiable from $\xset$ with respect to $\Theta$ then it is locally S-identifiable from $\xset$ with respect to $\Theta$.
\end{theorem}
The proof is in \autoref{sec:global2local} and uses the embedding $\rmap(\cdot)$.
By~\autoref{lem:admnec}, PS-identifiability with respect to $\Theta = \RR^{\paramset}$ implies admissibility. Considering any $\Theta$ with a similar property, a direct corollary of \autoref{th:IdentImpliesLocaIdent} is that PS-identifiability with respect to $\Theta$ implies local S-identifiability with respect to $\Theta$. \modif{Note however that the assumption that $\theta$ is admissible cannot simply be skipped in \autoref{th:IdentImpliesLocaIdent}.}
\begin{figure}[htbp]
\begin{center}
\includegraphics[width=0.8\textwidth]{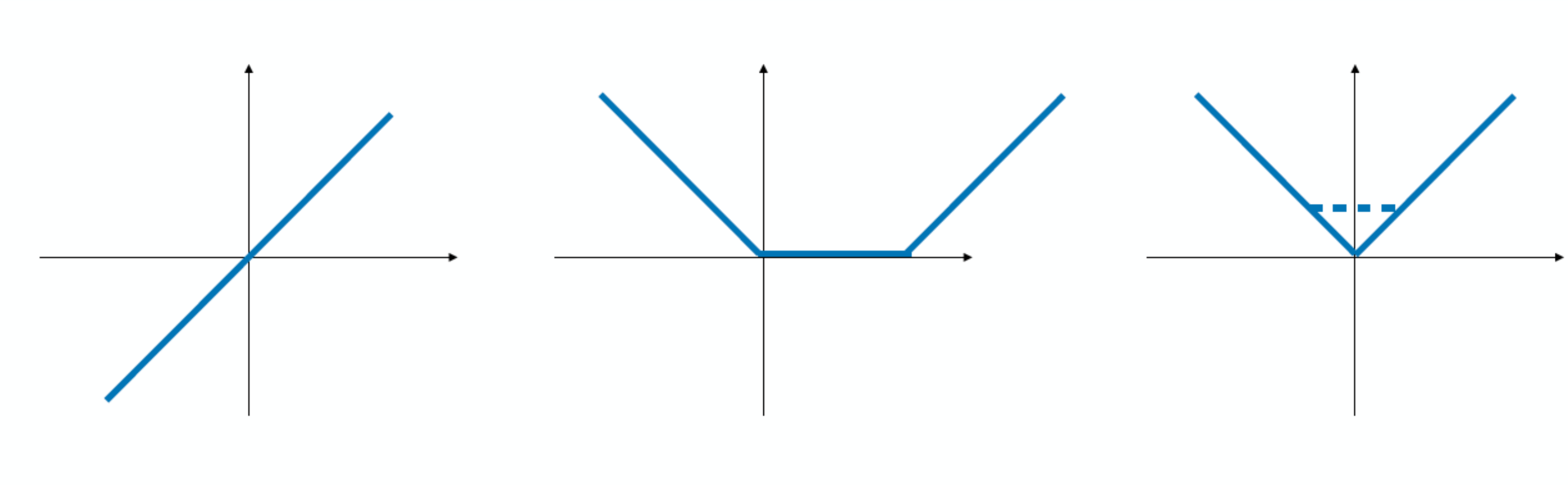}
\caption{Realizations of networks from (a) \autoref{ex:1twinpair}; (b) \autoref{ex:nonlocal} ; (c) \autoref{ex:LSInotND}}
\label{fig:exReLU}
\end{center}
\end{figure}

 An example shows that indeed, local S-identifiability depends on the constraint set $\Theta$.
\begin{example}[see \autoref{fig:exReLU}-(a)]\label{ex:1twinpair}
On a shallow network architecture with two hidden neurons $\nu_{1},\nu_{2}$, the identity $\mathtt{id}: \RR \to \RR, x \mapsto x$ can be written as  $x=\ReLU(x-t)-\ReLU(-(x-t))+t  = \realiz{\theta_{t}}$ with $\theta_{t} = (w_{\mu\to\nu_{1}}=1,w_{\mu\to\nu_{2}}=-1, b_{\nu_{1}}=-t,b_{\nu_{2}}=t, w_{\nu_{1}\to\eta}=1,w_{\nu_{2}\to\eta}=-1,b_{\eta}=t$) for every $t \in \RR$ ($\mu$ is the input neuron, $\eta$ the output neuron). Since $\theta_{t}$ and $\theta_{t'}$, $t \neq t'$ have different output bias, they are not PS-equivalent. This shows that, e.g., $\theta_{0}$ is not locally S-identifiable with respect to $\Theta = \RR^{\paramset}$. With respect to the set $\Theta$ of networks without output bias $(b_{\eta}=0)$, as detailed in \autoref{ex:1twinpairbis}, $\theta_{0}$ becomes PS-identifiable from $\xset = \RR$.
\end{example}
The above example includes two neurons which are \emph{twins} in the following sense.
\begin{definition}[Twin neurons]\label{def:shallowtwins}
Consider a parameter $\theta$ on a network architecture of any depth. Two hidden neurons $\nu \neq \nu'$ from the same layer are said to be twins if there exists $\lambda \in \RR$ such that $(\vw_{\bullet \to \nu},b_{\nu}) = \lambda (\vw_{\bullet \to \nu'},b_{\nu'})$. If $\theta$ is admissible then necessarily $\lambda \neq 0$, and $\nu,\nu'$ are said to be positive twins if $\lambda>0$, negative twins otherwise.\\
NB: Even though each hidden neuron $\nu \in H$ is (positive) twin to itself, such a neuron is abusively said to have ``no twin'' if it is not twin with any $\nu' \neq \nu$ from the same layer. We also say that $\theta$ has no twins if none of its neurons have any twin.
\end{definition}
Intuitively, if $\nu,\nu'$ are twins then the corresponding pre-activation functions 
$z_{\nu}(\theta,\cdot)$
$z_{\nu'}(\theta,\cdot)$ are collinear, and the resulting post-activation functions, $y_{\nu}(\theta,\cdot)$
$y_{\nu'}(\theta,\cdot)$ are also collinear for positive twins. For negative twins, 
there exists linear combinations of the post-activations that are simply proportional to the pre-activations, 
somehow bypassing the effect of the ReLU nonlinearity. 
As proved in~\autoref{app:PSIFromBoundedImpliesNoTwin}, twins \emph{always} prevent identifiability with respect to $\Theta = \RR^{\paramset}$.
\begin{lemma}\label{lem:PSIFromBoundedImpliesNoTwin}
Consider $\theta \in \Theta = \RR^{\paramset}$.
\begin{enumerate}[a)]
\item Assume that $\theta$ is locally S-identifiable with respect to $\Theta$. \\Then $\theta$ has no positive twins.
\item Assume that $\theta$ is PS-identifiable from some bounded set $\xset \subseteq \RR^{N_{0}}$ with respect to $\Theta$.\\
Then $\theta$ has no twins. 
\end{enumerate}
\end{lemma}
We will see in~\autoref{ex:LSInotND} (in~\autoref{sec:shallowcase}) that the absolute value function (see \autoref{fig:exReLU}-(c)) is the realization of a shallow network with two hidden neurons that are negative twins, yet it is PS-identifiable (hence locally S-identifiable) with respect to $\Theta = \RR^{\paramset}$. It is even locally S-identifiable from some finite set $F \subseteq \RR$. Of course, by \autoref{lem:PSIFromBoundedImpliesNoTwin} such a network cannot be PS-identifiable from any bounded set with respect to $\Theta = \RR^{\paramset}$. 

Twins are a form of {\em local} degeneracy. 
For shallow networks, we will show that this is the only form of local degeneracy (see the upcoming~\autoref{lem:NoTwinImpliesLSI} and~\autoref{th:MainTheorem}), but we will see other forms for deeper networks (see~\autoref{ex:stem}).
As illustrated next, there are also {\em non-local} degeneracies that can prevent identifiability.
\begin{example}[see \autoref{fig:exReLU}-(b)]\label{ex:nonlocal}
The function
\[
f(x) = 
\begin{cases}
-x, & \text{if}\ x\leq 0\\
0, & \text{if}\ 0\leq x \leq 1\\
x-1,& \text{if}\ x \geq 1
\end{cases}
\]
satisfies $f(x) = \ReLU(-x)+\ReLU(x-1)=\ReLU(x)+\ReLU(-(x-1))-1$. It is thus the realization of 
$\theta = (w_{\mu\to\nu_{1}}=-1,w_{\mu\to\nu_{2}}=1, b_{\nu_{1}}=0, b_{\nu_{2}}=-1, w_{\nu_{1}\to\eta}=w_{\nu_{2}\to\eta}=1,b_{\eta}=0$, but also of 
$\theta' = (w'_{\mu\to\nu_{1}}=1,w'_{\mu\to\nu_{2}}=-1, b'_{\nu_{1}}=0, b'_{\nu_{2}}=1, w'_{\nu_{1}\to\eta}=w'_{\nu_{2}\to\eta}=1,b'_{\eta}=-1$, which are not PS-equivalent since $b_{\eta} \neq b'_{\eta}$. Yet the theory we establish (see \autoref{lem:NoTwinImpliesLSI}) shows that $\theta$ and $\theta'$ are both locally S-identifiable from some finite set $F \subset \RR$.
\end{example}

It turns out that the above example fails to be \emph{irreducible} as we formalize next.
 \begin{definition}[Irreducibility]\label{def:irreducible}
A parameter $\theta$ is irreducible if for each hidden layer $1 \leq \ell \leq L-1$ and non-empty subset $T \subset N_{\ell}$ we have
 \begin{equation}\label{eq:Irreducible}
 \mat{W}_{\ell+1}\mat{I}_{T}\mat{W}_{\ell} \neq 0,\quad
\text{with}\  \mat{I}_{T} = \mathtt{diag}(\boldsymbol{\chi}_{T}),
 \end{equation}
 with $\vec{\chi}_{T} \in \{0,1\}^{N_{\ell}}$ the indicator function of $T$: $(\vec{\chi}_{T})_{\nu} = 1$ if, and only if, $\nu \in T$.
We denote $\Thetairred \subset \RR^{\paramset}$ the set of all irreducible parameters.  
\end{definition}  
\begin{fact}\label{rmk:LSIimpliesAdmissible} Each irreducible parameter is also admissible.
\end{fact}
In fact, as established in ~\autoref{app:irreduciblenecessary}, any PS-identifiable parameter with no twins must be irreducible. 
\begin{lemma}\label{lem:nonlocaldegeneracy}
If $\theta$ is PS-identifiable from $\xset \subseteq \RR^{N_{0}}$ with respect to  $\Theta = \RR^{\paramset}$ and has no twin, then it is irreducible.
\end{lemma}
In particular, in light of \autoref{lem:PSIFromBoundedImpliesNoTwin}, every \modiff{parameter that is PS-identifiable from a bounded $\xset$} is irreducible.
In the shallow case, a direct consequence of irreducibility can be obtained using an ``algebraic'' expression of the realization $\realiz{\theta}$ (\autoref{lem:RealizationAlgebraicBis} in \autoref{sec:technical_tools}):  for every input vector $x$ where $\realiz{\theta}$ is differentiable, the Jacobian of $\realiz{\theta}$ is given by $\mat{W}_{2}\mat{I}_{1} \mat{W}_{1}$ with $\mat{I}_{1} = \mathtt{diag}(\actlayer_{1}(\theta,x))$ \modif{(see \autoref{sec:technical_tools} for the introduction of notation $\actlayer_{1}(\theta,x)$)}. Irreducibility thus implies that this Jacobian can only vanish if $\actlayer_{1}(\theta,x) = \mat{0}$, i.e., if all neurons are inactive. As illustrated on \autoref{ex:nonlocal} (see \autoref{fig:exReLU}-(b)) this however does not \emph{characterize} irreducibility, and an intuitive characterization of irreducibility in terms of simple properties of $\realiz{\theta}$ is left to future work. 


\subsection{Identifiability conditions in the shallow case}
For shallow neural networks, we prove that admissible parameters with no twins are locally S-identifiable from a \emph{finite} set. Such results resonate with previous work on the identifiability of shallow networks equipped with various activation functions other than the ReLU \cite{sussman,kainen1994uniqueness,kurkova,albertini}. 
\begin{lemma}\label{lem:NoTwinImpliesLSI}
Consider a shallow architecture. If $\theta$ is admissible with no twins, then there is a finite $\xset \subseteq \RR^{N_{0}}$ with $\card{\xset} \leq (|N_{0}|+1)(|N_{1}|+1)$ from which $\theta$ is locally S-identifiable with respect to $\Theta = \RR^{\paramset}$.
\end{lemma}
The proof is in~\autoref{sec:proofMainTheoremLSI}. Combined with irreducibility, the absence of twins is further shown to be equivalent to PS-identifiability from a \emph{bounded} set. Whether this is also equivalent to PS-identifiability from a \emph{finite} set is left to future work, as well as a possible explicit control of the cardinality of such a finite set.
\begin{theorem}\label{th:MainTheorem}
Consider a shallow network architecture. The following are equivalent:
\begin{enumerate}[a)]
\item \label{it:main1} there is a bounded $\xset \subseteq \RR^{N_{0}}$  from which $\theta$ is PS-identifiable with respect to $\Theta = \RR^{\paramset}$;
\item \label{it:main2} $\theta$ has no twins and is irreducible.
\end{enumerate}
\end{theorem}
\begin{proof}
The implication~\ref{it:main1} $\Rightarrow$ \ref{it:main2} is a consequence of~\autoref{lem:PSIFromBoundedImpliesNoTwin} and~\autoref{lem:nonlocaldegeneracy}. The converse \ref{it:main2} $\Rightarrow$ \ref{it:main1} follows by \autoref{thm:PSIarchitecture}-\ref{it:PSIirred} in~\autoref{sec:proofMainTheorem}.
\end{proof}
As established with~\autoref{thm:PSIarchitecture}-\ref{it:PSIarchi} in~\autoref{sec:proofMainTheorem}, the \modiff{shallow} architecture itself is identifiable in the following sense for irreducible parameters with no twins.
\begin{theorem}
Consider two shallow network architectures with the same input and output layers, $N_{0}$ and $N_{2}$, and potentially distinct hidden layer $H= N_{1}$, $H' = N'_{1}$. Let $\theta$, $\theta'$ be parameters on each architecture. Assume that $\theta$ is irreducible with no twins, and that $\theta'$ is admissible with no twins. If $\realiz{\theta} = \realiz{\theta'}$ on $\RR^{N_{0}}$ then $\card{N_{1}} = \card{N'_{1}}$ and $\theta' \sim_{PS} \theta$.
\end{theorem}
As illustrated by~\autoref{ex:LSInotND} in \autoref{sec:shallowcase}, there are also shallow networks that are PS-identifiable from $\xset = \RR^{N_{0}}$ but not from any bounded set. They are of course irreducible by~\autoref{lem:nonlocaldegeneracy}, and have no positive twin by \autoref{lem:PSIFromBoundedImpliesNoTwin}, but they have \modif{one or more pairs of negative twins}. 
\subsection{A glimpse at the analysis of local identifiability}
Much of the local identifiability analysis, which is \modif{conducted in detail} in \autoref{sec:technical_tools}, relies on an important property of the embedding $\rmap$ (besides its ability to characterize scaling equivalence, see \autoref{lem:EqualRepAndSignImpliesSEquiv}):
  it provides a \emph{locally linear parameterization} of the realization of the network, in the sense that given $\theta$ and for ``most'' $x \in \RR^{N_{0}}$ we have,  for every $\theta'$ in a (small enough) neighborhood $\theta$ 
\begin{equation}\label{eq:LocallyLinearRealization}
\realiz{\theta'}(x)-\realiz{\theta}(x) = \linopreal_{\theta,x} \cdot \left (\rmap(\theta')-\rmap(\theta)\right)
\end{equation}
with $\linopreal_{\theta,x} \in \RR^{N_{L} \times \pathset{P}}$ some linear operator that is \emph{independent of $\theta'$}, \modiff{see Corrolary~\ref{cor:RexpressionbisSimple}} for a precise statement. This property holds provided $x$ is a point where the gradient of $\realiz{\theta}$ (and of all pre-activations at intermediate hidden layers) is well-defined and continuous, which motivates the following definition.
 
\begin{definition}\label{def:characXset}
Consider any network architecture. Given a parameter $\theta$ we define for each hidden 
neuron $\nu \in H$
the set $\Gamma_{\nu}(\theta)$ of input vectors where $z_{\nu}(\theta,x) =0$  and the gradient $\nabla z_{\nu}(\theta,x)$ is \emph{well-defined} and nonzero, 
\begin{align}
\label{def:GammaSet}
\Gamma_{\nu}(\theta) &:= \{x \in \RR^{N_{0}}: z_{\nu}(\theta,x) = 0\ \text{and}\ \nabla z_{\nu}(\theta,x) \neq 0\}.
\end{align}
We define $\xsetcont \subseteq \RR^{N_{0}}$ as the complement to $\cup_{\nu \in H} \Gamma_{\nu}(\theta)$.
\end{definition}


Definition~\ref{def:characXset} is extremely close to the definition of \emph{Bent Hyperplanes}~\cite{hanin2019deep} (except that we add the non-nullity condition on the gradient). Informally, and as previously stated~\cite{pascanu2013number,montufar,raghu}, bent hyperplanes separate the input space into \emph{linear regions} where the realization of the network $x\mapsto \realiz{\theta}(x)$ is affine, see Figure~\ref{fig:exHyperplane} for an illustration.

\begin{figure}[htbp]
  \begin{center}
  \includegraphics[width=0.8\textwidth]{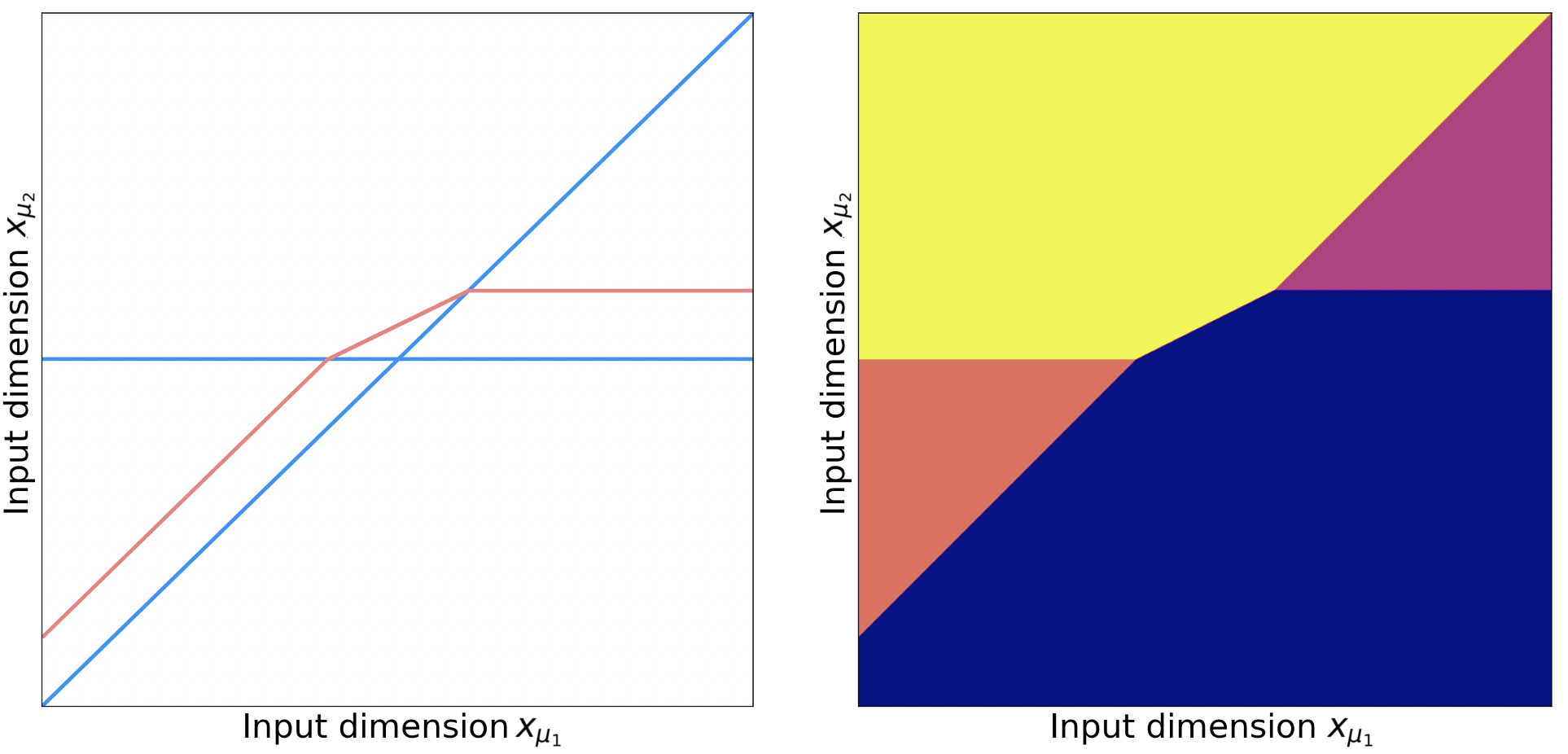}
  \caption{We consider a network architecture with $|N_1| = 2$ neurons on the first hidden layer and $|N_2| = 1$ neuron on the second hidden layer. The input $x$ is two-dimensional: $|N_0| = 2$ and the output is scalar: $|N_4| = 1$. Left: bent hyperplanes for the first hidden layer, $\Gamma_{\nu}(\theta)$, $\nu \in N_{1}$ (blue) and second hidden layer $\Gamma_{\nu}(\theta)$, $\nu \in N_{2}$ (red). Right: linear regions. All the weights and biases were initialized randomly.
  The figures are generated with a PyTorch script available at \url{https://github.
  com/pierrestock/linear-regions/blob/main/partition.ipynb.}}
  \label{fig:exHyperplane}
  \end{center}
\end{figure}

For our needs, we will provide in \autoref{lem:characXset}  an alternate characterization of $\xsetcont$ which we have not found elsewhere in the literature. It will be used in \autoref{cor:RexpressionbisSimple} to formalize Property~\eqref{eq:LocallyLinearRealization} for $x \in \xsetcont$,
which motivates the definition of non-degenerate parameters.
\begin{definition}[Non-degeneracy]\label{def:ND}
Consider the finite dimensional linear space
\begin{equation}
\linspace{V}(\theta) := \cap_{x \in \xsetcont} \ker(\linopreal_{\theta,x}) \subseteq \RR^{\pathset{P}},
\end{equation}
\modif{where $\linopreal_{\theta,x}$ is introduced formally in Corrolary~\ref{cor:RexpressionbisSimple}.}
A parameter $\theta  \in \Theta \subseteq \RR^{\paramset}$ is 
\emph{$\epsilon$-non-degenerate} with respect to $\Theta$, where $\epsilon>0$, if it is admissible and for every $\theta' \in \Theta \cap B(\theta,\epsilon)$ we have
\begin{equation}
\rmap(\theta')-\rmap(\theta) \in \linspace{V}(\theta) \Rightarrow \rmap(\theta')=\rmap(\theta). 
\end{equation}
It is \emph{non-degenerate} with respect to $\Theta$ if there exists $\epsilon>0$ such that it is $\epsilon$-
non-degenerate with respect to $\Theta$.
\end{definition}
Exploiting the fact that all considered spaces are finite dimensional, we characterize the space $\linspace{V}(\theta)$ in terms of certain \emph{activation spaces} (\autoref{def:actspace}) and prove that non-degeneracy is equivalent (see \autoref{le:SufficientConditionLocalIdentifiability}, the main result of \autoref{sec:technical_tools}) to the existence of some \emph{finite set} $F \subset \xsetcont$ 
such that $\theta$ is locally S-identifiable from $F$ (hence also locally S-identifiable from $\xset = \RR^{N_{0}}$). The cardinality of $F$ is bounded from above using the dimension of activation spaces.

\subsection{Non-degeneracy and irreducibility in shallow \emph{vs} deeper architectures}

An easy sufficient condition for non-degeneracy  is to have a trivial space $\linspace{V}(\theta) = \{0\}$. For \emph{scalar-valued} shallow networks ($L=2, |N_{L}|=1$), we prove (cf~\autoref{le:necessaryconditionLNDshallowscalar} and \autoref{cor:EquivFRA} that non-degeneracy with respect to $\Theta = \RR^{\paramset}$ is in fact \emph{equivalent} to $\linspace{V}(\theta) = \{0\}$, and for shallow (possibly vector-valued) networks, the latter is proved to hold if, and only if, there are no twins (by \autoref{cor:EquivFRA} and \autoref{lem:shallowFRAisNoTwin}). In light of \autoref{th:MainTheorem}, when combined with irreducibility, the fact that  $\linspace{V}(\theta) = \{0\}$ thus becomes equivalent (for shallow networks) to the PS-identifiability of $\theta$ from some bounded set.

For networks of depth $L \geq 3$, any parameter such that $\linspace{V}(\theta) = \{0\}$ is of course still non-degenerate (hence locally S-identifiable from a finite set, by \autoref{le:SufficientConditionLocalIdentifiability}), but this property is no longer equivalent to the absence of twins: further conditions between layers are required, as illustrated by the following example. 
\begin{example}\label{ex:stem}
In Figure~\ref{fig:stem}, we exhibit a two-hidden-layer architecture valued with a parameter $\theta$ that presents no twin neurons (see Definition~\ref{def:shallowtwins}) but such that $\theta$ is not locally S-identifiable (see Definition~\ref{def:identif_local}).
\begin{figure}[htbp]
  \includegraphics[width=\textwidth]{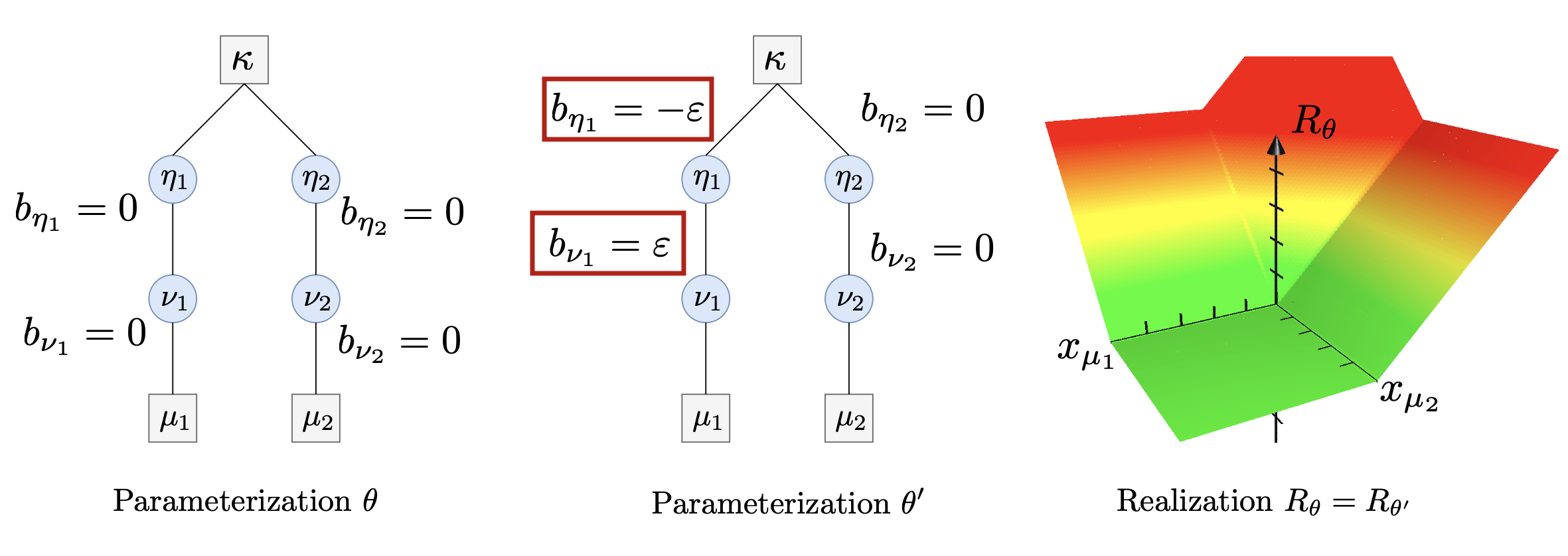}
  \caption{  \label{fig:stem}
A network with two hidden layers that is not locally S-identifiable, while having no twin hidden neurons. Weights $w_e$ and $w'_e$, $e\in E$ are set to one on the displayed edges and to zero on other edges, and are not depicted here for readability. Left: parameterization $\theta$ valuing the architecture. Center: alternative parameterization $\theta'$ such that $\theta$ and $\theta'$ are \emph{not} rescaling equivalent. Right: the two realizations $R_\theta$ and $R_{\theta'}$ coincide: for every input point $x = (x_{\mu_1}, x_{\mu_2}) \in \RR^2, R_\theta(x) = R_{\theta'}(x)$. The construction is valid for arbitrary $\varepsilon>0$.}
\end{figure}

\end{example}

Characterizing concrete conditions ensuring $\linspace{V}(\theta) = \{0\}$ is left to future work. A particular challenge is to understand  whether the condition $\linspace{V}(\theta) = \{0\}$, combined with (a possibly strengthened version of) irreducibility remains equivalent to PS-identifiability from a bounded set. We note that irreducibility in the shallow case is reminiscent of \cite[Equation (8)]{phuong2019functional}, a condition used to define so-called ``general ReLU networks'' to provide sufficient identifiability conditions in deeper settings. This may serve as a guide to identify stronger notions of irreducibility for deep networks. Preliminary investigations suggest that certain tensor products of activation vectors play a role when analyzing non-degeneracy. This is reminiscent of the tools studied by Fornasier \textit{et al.} \cite{fornasier2019robust} with two hidden layers $L=3$ in a smooth context that cannot cover ReLU networks.


\subsection{Discussion}\label{sec:discussion}

Before diving into the technical contributions in the next Sections, we discuss some topics of interest for the reader that are mostly out of the scope of this work. We refer the reader to Figure~\ref{fig:summary} for a brief summary of the results proven in this paper.

\subsubsection*{Local identifiability and optimization.}

First, we argue that studying \emph{local} (instead of global) S-identifiability is of practical interest, as discussed e.g. in~\cite{malgouyres2018multilinear,malgouyres2020stable}. Indeed, neural networks are traditionally optimized with a variant of stochastic gradient descent, or SGD \cite{lecun-gradientbased-learning-applied-1998}. Hence, (1) during training, the optimization yields parameters that are close to the previous ones and (2) the parameters obtained after convergence can be expected to be locally optimal up to natural permutation and rescaling equivalences.


\subsubsection*{Identifiability from a \emph{finite} set.}

Since we are mainly interested in the problem of recovering (the equivalence class of) $\theta$ from the knowledge of its realization $\realiz{\theta}$, we list below some questions calling for extensions of Theorem~\ref{th:MainTheorem}. Indeed, it is not always possible to recover $\theta$ from its realization.
Even when such a recovery is theoretically possible, it may involve having full access to the function $\realiz{\theta}$, which is not a concrete input to provide to any reconstruction algorithm. A more practical question is: when can we recover (the equivalence class of) $\theta$ from the knowledge of {\em finitely many samples} $\realiz{\theta}(x_{i}), 1 \leq i \leq n$ ? 
When there exists a choice (that may depend on $\theta$) of $n$ and $x_{i}$, $1 \leq i \leq n$ such that this is feasible, we also get as a byproduct a reconstruction of $\realiz{\theta}$ from the sole knowledge of its samples at these points. Hence, another question of interest is: when can the function $\realiz{\theta}$ be identified from the knowledge of finitely many of its samples ? This is possibly less demanding, as here it is not required to be able to reconstruct (the equivalence class of) $\theta$ from its realization. In both cases, since $\theta$ is not known beforehand, it is important to ensure that the choice of the sampling set is algorithmically feasible, for example if it is done iteratively at least the first sample must be chosen without any knowledge on $\theta$ or $\realiz{\theta}$.
Of course, answers to these questions lead to further ones, that we do not touch upon: if $\theta$ 
can be identified from finitely many samples, how many samples are sufficient\footnote{\autoref{lem:NoTwinImpliesLSI} partly answers this question regarding local S-identifiability for shallow networks.} (resp. necessary) ? \modiff{Can we explicit a scheme} (possibly randomized) to choose these samples ? Can we explicit an algorithm to perform reconstruction ? How stable is it to inaccuracies in the evaluation of $\realiz{\theta}(x_{i})$ or to the knowledge of $x_{i}$?

\subsubsection*{Reverse-engineering ReLU networks.}

Here, we dicuss the work of Rolnick and Kording~\cite{rolnick2019reverseengineering} more extensively than what was done in the Introduction. The goal is to position our work with respect to this interesting work. The authors present a sampling algorithm to recover a ReLU network's architecture and parameters, up to permutations and rescalings. The authors prove that their algorithm terminates except for a measure-zero set of networks and do not provide the complexity of their method in terms of number of the samples needed to recover $\realiz{\theta}$, except for recovering the first layer's parameters. They reason in terms of so-called \emph{activation} and \emph{linear} regions~\cite{hanin2019deep} and make the following assumptions. Recall that the sets $\Gamma_{\nu}(\theta)$ are introduced in Definition~\ref{def:GammaSet} for every hidden neuron $\nu$. $\Gamma_{\nu}(\theta)$ is often called the \emph{separating} or {\emph{bent} hyperplane for neuron $\nu$.

\begin{enumerate}
  \item Linear Regions assumption as stated by the authors: \say{Each [activation]\footnote{What the authors denote as linear regions are in fact known as activation regions, see~\cite{hanin2019deep}.}region represents a maximal connected component of input space on which the [realization $\realiz{\theta}$] is given by a single linear function}. In other words, the authors assume that activation regions and linear regions coincide (Section 3.2 in the original paper).
  \item All the sets $\Gamma_{\nu}(\theta)$ for $\nu \in H$ have codimension 1\footnote{This prevents the case where $\nabla R_\theta(x) = 0$ for $x \in B(x_0, r)$.} hence the name \emph{separating hyperplane} (implicitly assumed, see in particular the first paragraph of Section 3.3).
  \item For every hidden neuron $\nu$ in layer $1 \leq \ell \leq L - 1$, $\Gamma_{\nu}(\theta)$ intersects \emph{all} the sets $\Gamma_{\nu'}(\theta)$ for all neurons $\nu'$ in a previous layer $1 \leq \ell' < \ell \leq L - 1$ (Section 5.2 in the original paper).
  \item For $\nu \ne \nu'$ such that $\nu$ belong to layer $\ell$ and $\nu'$ belongs to layer $\ell' < \ell$, \say{$\Gamma_{\nu}(\theta)$ bends on $\Gamma_{\nu'}(\theta)$, but $\Gamma_{\nu}(\theta)$ and $\Gamma_{\nu'}(\theta)$ cannot both bend at their intersection} (implicitly assumed, see in particular the first paragraph of Section 3.3).
  \item For every hidden neuron $\nu \in H$, $\Gamma_{\nu}(\theta)$ is \emph{not bounded} and \emph{not disconnected} (Section 5.2 in the original paper).
\end{enumerate}

According to the authors, parameters $\theta$ that do not satisfy at least one of these assumptions constitute a measure-zero set of networks, hence the authors discard these cases from their analysis. In the remainder of this paper, we aim at more precisely characterizing this measure-null zero set. This is fully done in the shallow case, and the developed tools should be instrumental when pursuing this mathematical study in deeper settings.

%
%
%

\begin{figure}[htbp]
  \includegraphics[width=\textwidth]{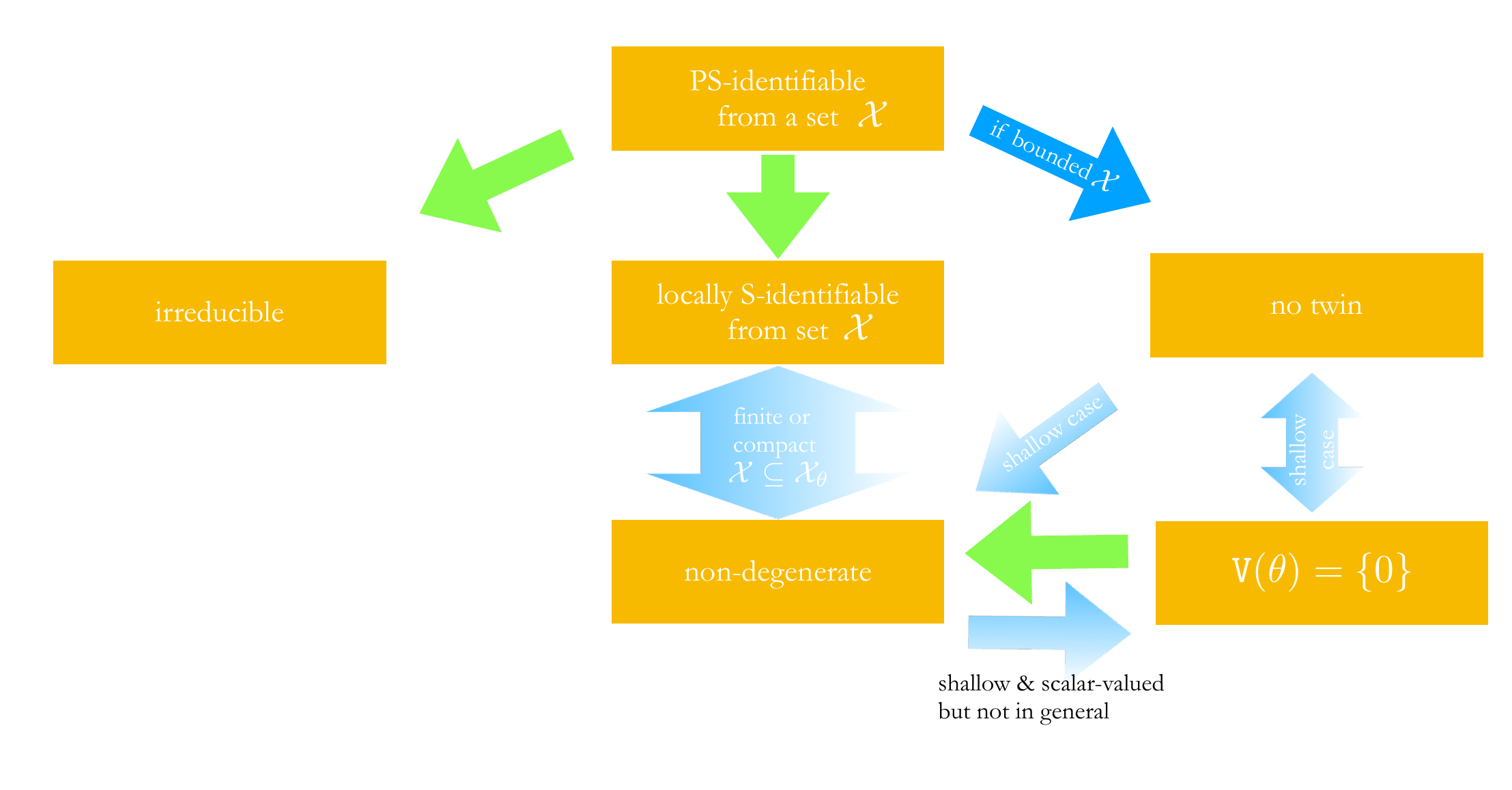}
  \caption{  \label{fig:summary}
Summary of the various results proven in the paper. 
\autoref{th:MainTheorem} further establishes that irreducibility and the absence of twins imply PS-identifiability from a bounded set in the shallow case.}
\end{figure}

%
%
%
%

\section{Rescaling invariance of the embedding}
\label{sec:embedding}
%

The proof of the main property of the embedding $\rmap(\cdot)$, \autoref{lem:EqualRepAndSignImpliesSEquiv}, exploits linear operators related to $\rmap(\cdot)$. The following definition is motivated by the obvious observation that, if $\theta$ has positive entries $\theta_{i} = e^{\alpha_{i}}$, $i \in \paramset$, then $\rmap(\theta) = e^{\mat{P}\alpha}$ where the exponential is taken componentwise. This is related to the idea of updating weights \emph{multiplicatively}, which is exploited in particular by Bernstein~\cite{bernstein_learning_2020} to investigate learning stability.
\begin{definition}
Consider $\pathmatrix: \RR^{\paramset} \to \RR^{\pathset{P}}$ the linear operator defined for $\boldsymbol{u} \in \RR^{\paramset}$ as
\begin{equation}
(\pathmatrix\vu)_{p} := 
\begin{cases}
\sum_{e \in p} u_{e}, & \text{for each full path}\ p \in \pathset{P}_{0};\\
u_{p_{\ell}} + \sum_{e \in p} u_{e}, & \text{for each partial path}\ p = (p_{\ell},\ldots,p_{L}) \in \pathset{P}_{\ell},\ 1 \leq \ell \leq L.
\end{cases}
\end{equation}
With the notations from Remark~\ref{rk:belongstop} we can also write $(\pathmatrix\vu)_{p} =\sum_{i \in p} u_{i}$. 
\end{definition}
Before proving~\autoref{lem:EqualRepAndSignImpliesSEquiv} we express a few technical lemmas.

\begin{lemma}\label{lem:supportsrmap}
For every $\theta \in \RR^{\paramset}$, with $\supp{\rmap(\theta)} = \{p \in \pathset{P}: \srmap_{p}(\theta) \neq0\}$ we have
\begin{equation}\label{eq:supportrmapincluded}
 \{i \in \paramset: \exists p \in \supp{\rmap(\theta)}, p \ni i\} \subseteq \{i \in \paramset: \theta_{i} \neq 0\} =
\supp{\theta}.
\end{equation}
If $\theta \in \RR^{\paramset}$ is admissible then we further have
\begin{equation}\label{eq:supportrmapequal}
\supp{\theta} =  \{i \in \paramset: \exists p \in \supp{\rmap(\theta)}, p \ni i\}. 
\end{equation}
\end{lemma}
\begin{proof}
For each path $p \in \pathset{P}$ denote $I_{p} = \{i \in \paramset: i \in p\}$ and observe first that the left hand side in~\eqref{eq:supportrmapincluded} is $\cup_{p \in \supp{\rmap(\theta)}} I_{p}$. Consider $p \in \supp{\rmap(\theta)}$. Since $\srmap_{p}(\theta)  = \Pi_{i \in p} \theta_{i}$, we have $\theta_{i} \neq 0$ for each $i \in I_{p}$, i.e., $I_{p} \subseteq \supp{\theta}$. As this holds for every $p \in \supp{\rmap(\theta)}$ we obtain $\cup_{p \in \supp{\rmap(\theta)}} I_{p} \subseteq \supp{\theta}$. This establishes~\eqref{eq:supportrmapincluded}.

Assuming now that $\theta$ is admissible, consider $i \in \supp{\theta}$ and distinguish three cases. If $i=\eta \in N_{L}$ is an output neuron, then $p = (\eta) \ni i$ yields $\srmap_{p}(\theta) = \theta_{\eta} \neq 0$. If $i = \nu \in H$ is a hidden neuron, then since $\theta$ is admissible there is a path $p \ni i$ with nonzero weights connecting $\nu$ to an output neuron. This path satisfies $\srmap_{p}(\theta) \neq 0$. Finally, if $i=\nu\to\nu'$ is an edge, then since $\theta$ is admissible there is a path connecting the input layer to $\nu$ and a path connecting $\nu'$ to the output layer, both with nonzero weights. Concatenating them yields a path $p \ni i$ such that $\srmap_{p}(\theta) \neq 0$. In all cases, we obtain the existence of a path $p \in \supp{\rmap(\theta)}$ such that $p \ni i$. This establishes~\eqref{eq:supportrmapequal}.
\end{proof}
\begin{corollary}\label{cor:embedimpliessupport}
Consider $\theta,\theta' \in \RR^{\paramset}$ such that $\rmap(\theta') = \rmap(\theta)$. If $\theta$ is admissible then $\supp{\theta'} = \supp{\theta}$ and $\theta'$ is also admissible.
\end{corollary}
\begin{proof}
By \autoref{lem:supportsrmap} and the equality $\rmap(\theta') = \rmap(\theta)$ we have
\begin{align*}
\supp{\theta} 
&=  \{i \in \paramset: \exists p \in \pathset{P}, \srmap_{p}(\theta) \neq 0, i \in p\}\\
&=  \{i \in \paramset: \exists p \in \pathset{P}, \srmap_{p}(\theta') \neq 0, i \in p\}
 \subseteq \supp{\theta'}.
\end{align*}
The fact that $\theta$ is admissible is a property of its support, and the inclusion $\supp{\theta} 
 \subseteq \supp{\theta'}$ implies that $\theta'$ is also admissible. It follows using  \autoref{lem:supportsrmap} again that the rightmost  inclusion above is an equality.
\end{proof}

\begin{lemma}\label{prop:bijective}
Given $\theta \in \RR^{\paramset}$ an admissible parameter, consider the spaces
    \begin{align}
    W_{\theta} & := \{\alpha \in \RR^{\paramset}, [\rmap(\theta) \odot \mat{P}\alpha]_{\pathset{P}_{0}} = 0\}\label{eq:Wthetabijective}\\
    V_{\theta} & := \{\alpha \in W_{\theta},  \alpha_{\bar{H}}=0, \supp{\alpha} \subseteq \supp{\theta}\}\label{eq:Vthetabijective}.
    \end{align}
Given $\alpha \in W_{\theta}$, define for each hidden neuron $\nu \in H$
    \begin{align}
        (\mat{S}_{\theta}\alpha)_\nu \triangleq -\sum_{e \in p} \alpha_e
    \end{align}
with $p$ any path with edges $e \in E \cap \supp{\theta}$ joining $\nu$ to an output neuron $\eta$.
     \begin{enumerate}[a)]
    \item  The linear map $\mat{S}_{\theta} \colon W_{\theta} \to \RR^{H}$  is well-defined and independent of the choice of $p$ and $\eta$;
    \item Its restriction $\mat{S}_{\theta}: V_{\theta} \to \RR^{H}$ is an isomorphism. Its inverse $\mat{S}_{\theta}^{-1}: \RR^{H} \to V_{\theta}$ is such that for any $\beta \in \RR^{H}$, $\mat{S}_{\theta}^{-1}\beta = \alpha$ where $\alpha_{\bar{H}}=0$ and  for each edge $e = \mu\to\nu\in E \cap \supp{\theta}$, 
        \begin{align}\label{eq:DefPsiInverse}
    \alpha_{e} \triangleq
        \begin{cases}
        -\beta_\mu & \text{if}\ \mu\in N_{L-1}\ (\text{and}\ \nu \in N_{L})\\
        \beta_\nu - \beta_\mu & \text{if}\ \mu\in N_{\ell}, 1 \leq \ell \leq L-2\\
        \beta_\nu & \text{if}\ \mu \in N_0.
        \end{cases}
    \end{align}
    while $\alpha_{e}=0$ for each $e \in E \backslash \supp{\theta}$.
    \end{enumerate}
\end{lemma}
 The proof is postponed to \autoref{app:proofbijective} to keep the reading flow.
\begin{proof}[Proof of~\autoref{lem:EqualRepAndSignImpliesSEquiv}]
By~\autoref{def:resc_neuron}, $\theta \sim_{S} \theta'$ if, and only if, there are $\{\lambda_{\nu}\}_{\nu \in H \cup N_{0} \cup N_{L}}$ such that 
\begin{align}
\lambda_{\nu}&>0,\quad \forall \nu \in H,\quad\text{and}\ \lambda_{\nu}=1,\quad \forall \nu \in N_{0} \cup N_{L}\\
\theta'_{e} &= \lambda_{\mu}^{-1}\theta_{e} \lambda_{\nu},\quad \forall e = \mu \to \nu \in E\quad \text{and}\ \theta'_{\nu}=\theta_{\nu}\lambda_{\nu},\quad \forall  \nu \in \bar{H}.
\end{align}
Thus, if $\theta' \sim_{S} \theta$ then $\sign{\theta'} = \sign{\theta}$, and for every path $p = (p_{0},\ldots,p_{L}) \in \pathset{P}_{0}$ we get
\begin{align*}
\srmap_{p}(\theta') &= \Pi_{k=0}^{L-1}\theta'_{p_{k}\to p_{k+1}}
= \Pi_{k=0}^{L-1}(\lambda_{p_{k}}^{-1}\theta_{p_{k}\to p_{k+1}} \lambda_{p_{k+1}})
= \Pi_{k=0}^{L-1}\theta_{p_{k}\to p_{k+1}}
= \srmap_{p}(\theta),\\
\intertext{while for $p = (p_{\ell},\ldots,p_{L}) \in \pathset{P}_{\ell}$, $1 \leq \ell \leq L$}
\srmap_{p}(\theta') &= \theta'_{p_{\ell}} \Pi_{k=\ell}^{L-1}\theta'_{p_{k}\to p_{k+1}}
= \theta_{p_{\ell}}\lambda_{p_{\ell}}\Pi_{k=\ell}^{L-1}(\lambda_{p_{k}}^{-1}\theta_{p_{k}\to p_{k+1}} \lambda_{p_{k+1}})
= \theta_{p_{\ell}}\Pi_{k=0}^{L-1}\theta_{p_{k}\to p_{k+1}}
= \srmap_{p}(\theta).
\end{align*}
This shows $\rmap(\theta') = \rmap(\theta)$.

Conversely, assume that $\theta$ is admissible and that $\rmap(\theta') = \rmap(\theta)$ and $\sign{\theta'_{E}} = \sign{\theta_{E}}$. By~\autoref{cor:embedimpliessupport}, since $\theta$ is admissible and $\rmap(\theta') = \rmap(\theta)$, we have $\supp{\theta'} = \supp{\theta}$ hence there are $\gamma_{i} \neq 0, i \in \supp{\theta}$ such that $\theta'_{i} = \gamma_{i} \theta_{i}$ for each $i \in \supp{\theta}$. Since $\sign{\theta'_{E}} = \sign{\theta_{E}}$, we have $\gamma_{e}>0$ for every $e \in E \cap \supp{\theta}$. Consider $\alpha \in \RR^{\paramset}$ such that $\alpha_{\bar{H}} = 0$, $e^{\alpha_{e}} = \gamma_{e}$ for $e \in E \cap \supp{\theta}$, and $\alpha_{e}=0$ for $e  \in E \backslash \supp{\theta}$.
For each $p \in \pathset{P}_{0}$
\[
\srmap_{p}(\theta') = \Pi_{e \in p} \theta'_{e} = \Pi_{e \in p} (\theta_{e} e^{\alpha_{e}}) = \srmap_{p}(\theta) e^{\sum_{e\in p}\alpha_{e}} = \srmap_{p}(\theta) \odot e^{(\mat{P}\alpha)_{p}}.
\]
Since $\rmap(\theta') = \rmap(\theta)$, it follows that for each $p \in \pathset{P}_{0}$ such that $\srmap_{p}(\theta) \neq 0$ we have $e^{(\mat{P}\gamma)_{p}}=1$, i.e., $(\mat{P}\alpha)_{p} = 0$. Thus, $\srmap_{p}(\theta) (\mat{P}\alpha)_{p} = 0$ for all $p \in \pathset{P}_{0}$, i.e., $[\rmap(\theta) \odot \mat{P}\alpha]_{\pathset{P}_{0}} = 0$. Since $\alpha_{\bar{H}} = 0$, we get that $\alpha$ belongs to the space $V_{\theta}$ defined in \eqref{eq:Vthetabijective} in \autoref{prop:bijective}. Since $\theta$ is admissible, the linear operator $\mat{S}_{\theta}$ defined in \autoref{prop:bijective} is a well-defined bijection from $V_{\theta}$ to $\RR^{H}$, hence $\alpha$ is related to $\beta := \mat{S}_{\theta}\alpha \in \RR^{H}$ through the relation~\eqref{eq:DefPsiInverse}. Considering $\delta \in \RR^{N_{0} \cup H \cup N_{L}}$ with $\delta_{\nu} := \beta_\nu$ for $\nu \in H$, $\delta_{\nu}=0$ for $\nu \in N_{0} \cup N_{L}$, relation~\eqref{eq:DefPsiInverse} implies
\[
\alpha_{e} = \delta_{\nu}-\delta_{\mu},\quad \forall e = \mu \to \nu \in E \cap \supp{\theta}.
\]
Setting 
$\lambda_{\nu} := e^{\delta_{\nu}}$ for each $\nu \in N_{0} \cup H \cup N_{L}$, it follows that for each $e = \mu \to \nu \in E$ we have $\theta'_{e} = \lambda_{\mu}^{-1} \theta_{e} \lambda_{\nu}$. Since $\supp{\theta} = \supp{\theta}$, this also trivially holds for $e \in E \backslash \supp{\theta}$. 

To conclude, we show that $\theta'_{\nu} = \theta_{\nu} \lambda_{\nu}$ for each $\nu \in \bar{H}$. As this holds trivially for $\nu \in \bar{H} \cap \supp{\theta}$, we focus on $\nu \in \bar{H} \cap \supp{\theta}$. First, we treat the case of $\eta \in N_{L} \cap \supp{\theta}$ by observing that, with $p = (\eta) \in \pathset{P}_{L}$ we have $\theta'_{\eta} = \srmap_{p}(\theta') = \srmap_{p}(\theta) = \theta_{\eta} = \theta_{\eta}\lambda_{\eta}$ since $\lambda_{\eta} = e^{\delta_{\eta}} = 1$ by definition of $\delta_{\eta}:=0$. Now consider $\nu \in H \cap \supp{\theta}$.
Since $\theta$ is admissible, there is a partial path $p$ connecting $\nu$ to some output neuron $\eta$ with edges in $\supp{\theta}$. Since $-\sum_{e \in p} \alpha_{e} = (\mat{S}_{\theta}\alpha)_{\nu} = \beta_{\nu} = \delta_{\nu}$ we have
\begin{align*}
\theta'_{\nu} \Pi_{e \in p} \theta'_{e} = \srmap_{p}(\theta') = \srmap_{p}(\theta) = \theta_{\nu} \Pi_{e \in p} \theta_{e} = \theta_{\nu} \Pi_{e \in p } \theta'_{e} e^{-\alpha_{e}} 
&= \theta_{\nu} (\Pi_{e \in p} \theta'_{e} ) e^{-\sum_{e \in p} \alpha_{e}}\\
& = \theta_{\nu} (\Pi_{e \in p} \theta'_{e} ) e^{\delta_{\nu}}  = \theta_{\nu} (\Pi_{e \in p} \theta'_{e} ) \lambda_{\nu}.
\end{align*}
We conclude using that $\Pi_{e \in p} \theta'_{e} \neq 0$ since all edges $e \in p$ belong to $\supp{\theta'} = \supp{\theta}$.
\end{proof}


\section{Analyzing local identifiability}
\label{sec:technical_tools}

Equipped with the rescaling-invariant embedding $\rmap(\cdot)$ we now establish the claimed local identifiability results. First, we need to introduce notations for the activation status of neurons and paths and use them to provide several expressions of the realization $\realiz{\theta}$ before providing the main result of the section, \autoref{le:SufficientConditionLocalIdentifiability}.
\subsection{Activation status of neurons and paths, and activation spaces}\label{sec:activations}
The forthcoming analysis heavily involves the activation status of each hidden neuron $\nu \in H$,  
\(
\actneuron_{\nu}(\theta,x) =  \indic{z_{\nu}(\theta,x)>0} \in \{0,1\},
\)
which gives rise to the activation status of each hidden layer $\actlayer_{\ell}(\theta,x) = (\actneuron_{\nu}(\theta,x))_{\nu \in N_{\ell}}$, $1 \leq \ell \leq L-1$, and the global activation status $\act(\theta,x) = (\actneuron_{\nu}(\theta,x))_{\nu \in H} = (\actlayer_{\ell}(\theta,x))_{1 \leq \ell \leq L-1}$. \\
\begin{definition}
The activation of  a path $p$ (full or partial) is defined as 
\[
\actpath_{p}(\theta,x) := \Pi_{\nu \in H \cap p}\actneuron_{\nu}(\theta,x) \in \{0,1\}
\]
where for $p = (p_{\ell},\ldots,p_{L}) \in \pathset{P}_{\ell}$ we denote the set of \emph{hidden} neurons visited by the path $p$ using the shorthand $H \cap p := \{\nu \in H, \exists i \in \llbracket \max(\ell,1), L-1\rrbracket, \nu = p_{i}\} \subset H$. \end{definition}
\begin{remark}
By convention, a product over an empty set is $1$. If $p$ contains no hidden neuron (e.g. , if $p  =(\eta) \in \pathset{P}_{L}$, $L \geq 1$) its activation is $\actpath_{p}(\theta,x)=1$ for every $x$.
\end{remark}
With $\pathset{Q} := \cup_{\ell=1}^{L-1} \pathset{Q}_{\ell}$ the set of all ``partial'' paths $q \in (q_{\ell},\ldots,q_{L-1})$ from a hidden layer $1 \leq \ell \leq L-1$ to the penultimate layer $L-1$ \modif{(cf \autoref{def:pathsets} for the formal definition of $\pathset{Q}_{\ell}$)}, we define
the binary-vector-valued function $\vactpath(\theta,x) := (\actpath_{q}(\theta,x))_{q \in \pathset{Q}} \in \{0,1\}^{\pathset{Q}}$.  
We also define variants that are notably useful to account for output biases
\[
 \bactlayer_{\ell}(\theta,x) = \left(\begin{matrix}\actlayer_{\ell}(\theta,x)\\1\end{matrix}\right) \in \{0,1\}^{N_{\ell}+1}
\ \text{and}\ 
\bvactpath(\theta,x) := \left(\begin{matrix}\vactpath(\theta,x)\\1\end{matrix}\right) \in \{0,1\}^{\pathset{Q}+1}
\]
where for any set $A,B$ we use the shorthand $A^{B+1} = A^{B} \times A$. 

To state the connections between non-degeneracy and local S-identifiability from finite sets, it is convenient to observe that the linear space $\linspace{V}(\theta)$ from \autoref{def:ND} can be characterized using simpler linear spaces called \emph{activation spaces}.
\begin{definition}[Activation spaces, activation dimension]\label{def:actspace}
The activation spaces associated to $\theta \in \RR^{\paramset}$ are
\begin{align}
\Aspacebias & := \linspan{\bvactpath(\theta,x),x \in \xsetcont} \subseteq \RR^{\pathset{Q}+1}. \\
\Aspace & = \linspan{\restmatrix\bvactpath(\theta,x),x \in \xsetcont} = \restmatrix \Aspacebias  \subseteq \RR^{\pathset{Q}_{1}}
\label{eq:actspace}
\end{align}
with $\mat{Q}$ as in \autoref{le:Rexpressionbis}.
We define its  \emph{activation dimension} as $\actdim(\theta) = \dim\left( \Aspacebias\right)$.
\end{definition}
\begin{remark}\label{rk:fralastlayer}
Observe that if $\theta$ and $\tilde{\theta}$ share the same $L-1$ first affine layers $(\mat{W}_{\ell},\vb_{\ell}) = (\widetilde{\mat{W}}_{\ell},\widetilde{b}_{\ell})_{\ell=0}^{L-1}$ then their activation spaces are identical. This holds \emph{even if the dimension of the output layers of $\tilde{\theta}$ and $\theta$ differ}.
\end{remark}

\begin{lemma}\label{lem:VCartesian}
Viewing $\RR^{\pathset{P}}$ as the product of $N_L \times N_{0}$ copies of $\RR^{\pathset{Q}_{1}}$ and $N_{L}$ copies of $\RR^{\pathset{Q}+1}$, $\linspace{V}(\theta) \subset \RR^{\pathset{P}}$ is the product of $N_{L} \times N_{0}$ copies of $\Aspaceperp \subseteq \RR^{\pathset{Q}_{1}}$ and $N_{L}$ copies of $\Aspacebiasperp$.
\end{lemma}
The proof is postponed to after \autoref{cor:RexpressionbisSimple} as it uses notations introduced there.

\begin{corollary}
\label{cor:EquivFRA}
$\linspace{V}(\theta) = \{0\}$ if, and only if, $\Aspacebias = \RR^{\pathset{Q}+1}$.
\end{corollary}
\begin{proof}
If $\linspace{V}(\theta) = \{0\}$ then by \autoref{lem:VCartesian} we have $\Aspacebiasperp = \{0\}$ hence $\Aspacebias = \RR^{\pathset{Q}+1}$. 
Vice-versa if $\Aspacebias = \RR^{\pathset{Q}+1}$ then $\Aspace = \restmatrix\Aspacebias  = \RR^{\pathset{Q}_{1}}$. By \autoref{lem:VCartesian} it follows that $\linspace{V}(\theta) = \{0\}$,
\end{proof}

\subsection{``Algebraic'' expressions of the realization} 
\label{sec:algebraicrealization}

We can express the realization using weight matrices, bias vectors and layerwise binary activation vectors. A similar formula
is stated without taking the biases into account in~\cite{meng_g-sgd_2019}[Lemma A.2] whereas~\cite{balduzzi2018shattered} performs analogous computations for gradient computations, still without biases.
\begin{lemma}\label{lem:RealizationAlgebraicBis}
Consider $\theta$ a network parameter of depth $L \geq 1$.
Denote $\mat{I}_{0}= \mat{Id}_{\RR^{N_{0}}}$ and for each $x$ and $1 \leq \ell \leq L-1$, $\mat{I}_{\ell} = \mathtt{diag}(\actlayer_{\ell}(\theta,x))$. The realization of $\theta$ satisfies
\begin{equation}\label{eq:RealizationAlgebraicBis}
\realiz{\theta}(x) = \left(\Pi_{\ell=1}^{L}\mat{W_{\ell}}\mat{I}_{\ell-1}\right) x 
+ \sum_{\ell'=1}^{L} \left(\Pi_{\ell=\ell'+1}^{L} \mat{W}_{\ell}\mat{I}_{\ell-1}\right)\vb_{\ell'}
\end{equation}
with the convention that a product over an empty set is the identity matrix.
\end{lemma}
The proof is in Appendix~\ref{app:ProofRealizationAlgebraic}. To conduct an analysis of the local S-identifiability of a parameter, another expression of $\realiz{\theta}$ where the embedding $\rmap(\theta)$ appears more explicitly will be useful. We rewrite~\eqref{eq:RealizationAlgebraicBis} using $\rmap(\theta)$ and the activation vector  $\bvactpath(\theta,x)$.
\begin{lemma}\label{le:Rexpressionbis}
Consider $\theta$ a network parameter of depth $L \geq 2$. For each $\eta \in N_{L}$, denote\footnote{Superscripts $\mathtt{i}$ and $\mathtt{h}$ stand for ``input'' and ``hidden'', as $\rmapi$ is associated to full paths starting from the input layer, while $\rmaph$ corresponds to partial paths starting from a hidden (or the output) layer.}
\begin{align}\label{eq:PartialRepresentations}
\rmapi_{\eta}(\theta) &:= 
\left(\srmap_{\mu\to q\to \eta}(\theta)\right)_{q \in \pathset{Q}_{1},\mu \in N_{0}} \in \RR^{\pathset{Q}_{1} \times N_{0}}\\
\rmaph_{\eta}(\theta) &:= 
\left(
\begin{matrix}
(\srmap_{q\to\eta}(\theta))_{q \in \pathset{Q}}\\
\theta_{\eta}
\end{matrix}\right)
\in \RR^{\pathset{Q}+1}
\end{align}
Up to reshaping, $\rmap(\theta)  \in \rspace$ is the concatenation of matrices $\rmapi_{\eta}(\theta) \in \RR^{\pathset{Q}_{1} \times N_{0}}$ and vectors $\rmaph_{\eta}(\theta) \in \RR^{\pathset{Q}+1}$ over all output neurons $\eta\in N_{L}$. For each output neuron $\eta\in N_{L}$ we have
\begin{align}
\realiz{\theta}(x)_{\eta} 
&= \langle \restmatrix\bvactpath(\theta,x), \rmapi_{\eta}(\theta)x\rangle
+\langle\bvactpath(\theta,x), \rmaph_{\eta}(\theta)\rangle
\label{eq:RexpressionBis}
\end{align}
where $\restmatrix: \RR^{\pathset{Q}+1} \to \RR^{\pathset{Q}_{1}}$ is the canonical restriction to $\pathset{Q}_{1} \subset \pathset{Q}$. 
\end{lemma}
The proof of \autoref{le:Rexpressionbis} is in Appendix~\ref{app:ProofRealizationAlgebraic}. It yields an expression of $\realiz{\theta}$ that perfectly fits the upcoming analysis of local S-identifiability. A more abstract (but probably somewhat more digestible) version of the same result implies Property~\eqref{eq:LocallyLinearRealization} as claimed.
\begin{corollary}\label{cor:RexpressionbisSimple}
Consider $\theta$ a network parameter of depth $L \geq 2$. For each $x \in \RR^{N_{0}}$ let $\mat{L}_{\theta,x}$ be the linear form on $\RR^{\pathset{Q}_{1} \times N_{0}} \times  \RR^{\pathset{Q}+1}$ defined as 
\[
\mat{L}_{\theta,x} \{(\mat{M},\vv)\} := \langle\restmatrix\bvactpath(\theta,x),\mat{M}x\rangle+\langle\bvactpath(\theta,x),\vv\rangle,
\quad\mat{M} \in \RR^{\pathset{Q}_{1} \times N_{0}},\quad \vv \in \RR^{\pathset{Q}+1}
\]
Define $\linopreal_{\theta,x} \in \RR^{N_{L} \times \pathset{P}}$ the matrix associated to the linear operator mapping each $\vec{\phi} \in \RR^{\pathset{P}}$, seen as a reshaped concatenation of matrices $\vec{\phi}_{\eta}^{\mathtt{i}} \in \RR^{\pathset{Q}_{1} \times N_{0}}$ and vectors $\vec{\phi}_{\eta}^{\mathtt{h}} \in \RR^{\pathset{Q}+1}$ as in \autoref{le:Rexpressionbis}, to $\vec{r} := (r_{\eta})_{\eta \in N_{L}}$, with $r_{\eta} = \mat{L}_{\theta,x} \{(\vec{\phi}_{\eta}^{\mathtt{i}},\vec{\phi}_{\eta}^{\mathtt{h}})\}.$ 
We have
\begin{equation}
\realiz{\theta}(x) = \linopreal_{\theta,x} \cdot \rmap(\theta)
\end{equation}
\end{corollary}

We are now equipped with the notations needed to prove~\autoref{lem:VCartesian}. The proof also relies on the following alternative characterization of the set $\xsetcont$ from \autoref{def:characXset} that we did not find elsewhere. It is proved in Appendix~\ref{app:ProofXset}.
\begin{lemma}\label{lem:characXset}
Given a parameter $\theta$, consider the open set of input variables $x$ such that $(\theta',z) \mapsto \act(\theta',z)$ is locally constant in some neighborhood of $(\theta,x)$.
\begin{align}\label{eq:DefXSetCont}
\xsetcont' 
:= \{ 
x \in \RR^{N_{0}}: 
\exists \epsilon,r>0,\ \forall (\theta,z) \in B(\theta,\epsilon) \times B(x,r),\ 
\act(\theta',z) = \act(\theta,x)
\}.
\end{align} 
with the convention that $\xsetcont' = \RR^{N_{0}}$ if the network depth is $L=1$. This set coincides exactly with the set $\xsetcont$ from Definition~\ref{def:characXset}. 
\end{lemma}

\begin{proof}[Proof of \autoref{lem:VCartesian}]
Consider a vector $\vec{\phi} \in \RR^{\pathset{P}}$ and its representation as $\vec{\phi}_{\eta}^{\mathtt{i}} \in \RR^{\pathset{Q}_{1} \times N_{0}}$, $\vec{\phi}_{\eta}^{\mathtt{h}} \in \RR^{\pathset{Q}}$, $\eta \in N_{L}$. By definition $\vec{\phi} \in \linspace{V}(\theta)$ if, and only if, $\linopreal_{\theta,x} \vec{\phi} = 0,\ \forall x \in \xsetcont$, i.e., for each $\eta \in N_{L}$ we have
\begin{align}
\label{eq:TmpActSpace0}
\langle \restmatrix \bvactpath(\theta,x),\vec{\phi}^{\mathtt{i}}_{\eta}x\rangle + \langle \bvactpath(\theta,x),\vec{\phi}^{\mathtt{h}}_{\eta}\rangle = 0,\ \forall x \in \xsetcont.
\end{align} 
By \autoref{lem:characXset}, $x' \mapsto \bvactpath(\theta,x')$ is locally constant in the neighborhood of each $x \in \xsetcont$, hence the left-hand-side in~\eqref{eq:TmpActSpace0} is locally affine with respect to $x$, and~\eqref{eq:TmpActSpace0} is thus equivalent to
\begin{align}
\label{eq:TmpActSpace1}
\begin{cases}
[\restmatrix \bvactpath(\theta,x)]^{\top}\vec{\phi}^{\mathtt{i}}_{\eta} &= \vec{0}_{1 \times N_{0}}\\
\langle \bvactpath(\theta,x),\vec{\phi}^{\mathtt{h}}_{\eta}\rangle &= 0
\end{cases},\ \forall x \in \xsetcont,
\end{align} 
that is to say each column of $\vec{\phi}^{\mathtt{i}}_{\eta}$, is orthogonal to $\restmatrix \bvactpath(\theta,x)$, and $\vec{\phi}^{\mathtt{h}}_{\eta}$ is orthogonal to  $\bvactpath(\theta,x)$ for every $x \in \xsetcont$. We conclude using the definition of $\Aspace,\Aspacebias$.
\end{proof}

%

\subsection{Non-degeneracy and local S-identifiability}



We can now state the main result of this section.

\begin{theorem} \label{le:SufficientConditionLocalIdentifiability}
Consider $\theta \in \Theta \subseteq \RR^{\paramset}$. The following are equivalent:
\begin{enumerate}[i)]
\item \label{it:NDimplies} $\theta$ is non-degenerate with respect to $\Theta$;
\item \label{it:FiniteLSIimplies} there is a finite $F \subset \xsetcont$ such that $\theta$ is locally S-identifiable from $F$ with respect to $\Theta$. 
\item \label{it:CompactLSIimplies} there is a compact $K \subset \xsetcont$ such that $\theta$ is locally S-identifiable from $K$ wrt $\Theta$. 
\end{enumerate}
When they hold, the finite set $F$ can be chosen such that 
\begin{equation}
\modiff{\card{F} \leq (|N_{0}|+1) \actdim(\theta)}.
\end{equation} 
\end{theorem}
\begin{remark}\label{rmk:NDvsLSI}
We exhibit in \autoref{ex:LSInotND} a PS-identifiable (hence locally S-identifiable) parameter $\theta$ that is degenerate, i.e., not locally S-identifiable from any compact $K \subseteq \xsetcont$. 
\end{remark}

\begin{proof}   
{\bf \ref{it:NDimplies}) $\Rightarrow$ \ref{it:FiniteLSIimplies})}
Consider $\epsilon>0$ such that $\theta$ is $\epsilon$-non-degenerate with respect to $\Theta$. To establish the existence of $F$ such that $\theta$ is locally S-identifiable from $F$ with respect to $\Theta$, we use a Lemma which proof is postponed.

\begin{lemma}\label{le:FunEquivImpliesOrthoConds}
Consider $\theta \in \RR^{\paramset}$.
\begin{enumerate}[a)]
\item \label{it:ExistFiniteLSI} There exists $\epsilon>0$ and a set $F \subset \xsetcont$ of cardinality at most $(N_{0}+1) \actdim(\theta)$ such that: for each $\theta' \in  B(\theta,\epsilon)$, if $\realiz{\theta'} = \realiz{\theta}$ on $F$, then $\rmap(\theta')-\rmap(\theta) \in \linspace{V}(\theta)$.
\item \label{it:ForAllCompactLSI} For every compact set $K \subset \xsetcont$, there exists $\epsilon'>0$ such that: for each $\theta' \in B(\theta,\epsilon')$, if $\rmap(\theta')-\rmap(\theta) \in \linspace{V}(\theta)$, then $\realiz{\theta'}(x) = \realiz{\theta}(x)$ for all $x \in K$.
\end{enumerate}
\end{lemma}
Let $\epsilon_{0}$, $F$ be given by \autoref{le:FunEquivImpliesOrthoConds}-\ref{it:ExistFiniteLSI} and set $\epsilon_{1} := \min(\epsilon_{0},\epsilon,\eta/2)$ where $\eta := \min_{i \in \supp{\theta}} |\theta_{i}|$. We will show that $\theta$ is $\epsilon_{1}$-locally S-identifiable from $F$. For this, consider $\theta' \in \Theta \cap B(\theta,\epsilon_{1})$ and assume that $\realiz{\theta'} = \realiz{\theta}$ on $K$. By \autoref{le:FunEquivImpliesOrthoConds}-\ref{it:ExistFiniteLSI}, since $\theta' \in B(\theta,\epsilon_{0})$, we have $\rmap(\theta')-\rmap(\theta) \in \linspace{V}(\theta)$. Since $\theta' \in B(\theta,\epsilon)$ and $\theta$ is $\epsilon$-non-degenerate, this implies $\rmap(\theta') = \rmap(\theta)$ hence (recall that, since $\theta$ is non-degenerate, it is admissible by definition) by \autoref{lem:supportsrmap} we have $\supp{\theta'}=\supp{\theta}$. Since $\theta' \in B(\theta,\eta/2)$ we further have $\sign{\theta'_{i}} = \sign{\theta_{i}}$ for every $i \in \supp{\theta}$, hence $\sign{\theta'} = \sign{\theta}$. By \autoref{lem:EqualRepAndSignImpliesSEquiv} we obtain $\theta' \sim_{S} \theta$. 

{\bf \ref{it:FiniteLSIimplies}) $\Rightarrow$ \ref{it:CompactLSIimplies})} Simply observe that a finite set is compact.

{\bf  \ref{it:CompactLSIimplies}) $\Rightarrow$ \ref{it:NDimplies})}
Consider $\epsilon>0$ such that $\theta$ is $\epsilon$-locally identifiable from $K$ with respect to $\Theta$.
By \autoref{le:FunEquivImpliesOrthoConds}-\ref{it:ForAllCompactLSI} for the compact set $K$, there is $\epsilon_{0}>0$ such that: for each $\theta' \in B(\theta,\epsilon_{0})$,  $\rmap(\theta')-\rmap(\theta) \in \linspace{V}(\theta) \Rightarrow (\realiz{\theta'}(x)=\realiz{\theta}(x), \forall x \in K)$. Set $\epsilon_{1} := \min(\epsilon,\epsilon_{0})$. We will show that $\theta$ is $\epsilon_{1}$-non-degenerate with respect to $\Theta$. 
Considering $\theta' \in \Theta \cap B(\theta,\epsilon_{1})$ such that  $\rmap(\theta')-\rmap(\theta) \in \linspace{V}(\theta)$ we now show that $\rmap(\theta') = \rmap(\theta)$. Since $\theta' \in B(\theta,\epsilon_{0})$ and  $\rmap(\theta')-\rmap(\theta) \in \linspace{V}(\theta)$, we have $\realiz{\theta'}(x)=\realiz{\theta}(x)$ for all $x \in K$. Since $\theta' \in \Theta \cap B(\theta,\epsilon)$ and $\theta$ is locally S-identifiable from $K$ with respect to $\Theta$  this implies $\theta' \sim_{S} \theta$, hence by \autoref{lem:EqualRepAndSignImpliesSEquiv} we have $\rmap(\theta')=\rmap(\theta)$.
\end{proof}

 \begin{proof}[Proof of \autoref{le:FunEquivImpliesOrthoConds}]
We begin with some preliminaries. Since $\Aspacebias \subseteq \RR^{\pathset{Q}+1}$, it is finite dimensional hence there is a finite set $\zset \subset \xsetcont$ such that $\card{\zset} = \actdim(\theta)$ and 
\begin{align}
\Aspacebias = \linspan{\bvactpath(\theta,z),z \in \zset}. 
\label{eq:spacesfinite}
\end{align}
By \modiff{Lemma~\ref{lem:characXset}}, for each  $z \in \zset$ there exists $\epsilon(z),r(z)>0$ such that, for every $\theta' \in B(\theta,\epsilon(z))$ and $x \in B(z,r(z))$, we have $\act(\theta',x)=\act(\theta,z)$, hence $\bvactpath(\theta',x)=\bvactpath(\theta,z)$. Since $\zset$ is finite, 
\[
\epsilon := \min_{z \in \zset} \epsilon(z)>0. 
\]
Consider $\theta' \in B(\theta,\epsilon)$,  $z \in \zset$, $x \in B(z,r(z))$. Since $\bvactpath(\theta',x) = \bvactpath(\theta,z)$, we have for each output neuron $\eta \in N_{L}$
\begin{align*}
\begin{cases}
[\restmatrix
 \bvactpath(\theta',x)]^{\top}
\rmapi_{\eta}(\theta')
&=  
[\restmatrix
\bvactpath(\theta',z)]^{\top}
\rmapi_{\eta}(\theta')\\
\bvactpath(\theta',x)^{\top}
\rmaph_{\eta}(\theta')
&= 
\bvactpath(\theta',z)^{\top}
\rmaph_{\eta}(\theta')
\end{cases}
\end{align*}
hence using~\eqref{eq:RexpressionBis} we get 
\begin{align}
\realiz{\theta'}(x)_{\eta}- \realiz{\theta}(x)_{\eta} 
=& [\restmatrix
\bvactpath(\theta,z)]^{\top} \left( \rmapi_{\eta}(\theta')-\rmapi_{\eta}(\theta)\right)x
  + \bvactpath^{\top}(\theta,z) \left( \rmaph_{\eta}(\theta')-\rmaph_{\eta}(\theta)\right).
  \label{eq:TmpCompact}
\end{align}

Considering $z \in \zset$, define $F_{z} := \{x_{i}\}_{i=0}^{N_{0}} \subset B(z,r(z)) \subset \RR^{N_{0}}$ where $x_{0}=z$ and for $1 \leq i \leq N_{0}$,  $x_{i} = z+\tfrac{r(z)}{2} \delta_{i}$ with $\delta_{i}$ the $i$-th vector of the canonical basis. Observe that if $\vu \in \RR^{N_{0}}, b \in \RR$ are such that $\vu^{\top}x+b = 0$ for every $x \in F_{z}$, then $\vu = \vec{0}$ (since $r(z)\vu^{\top}\delta_{i} = \vu^{\top}(x_{i}-x_{0}) = \vu^{\top}x_{i}+b-(\vu^{\top}x_{0}+b)=0$ for every $i$), and therefore $b=0$ too.

{\bf a)} The finite set $F: = \cup_{z \in \zset} F_{z}$ satisfies $F \subset \cup_{z \in \zset} B(z,r(z)) \subset \xsetcont$. Assume that $\realiz{\theta'}=\realiz{\theta}$ on  $F$ where $\theta' \in B(\theta,\epsilon)$. By the preliminaries, this implies that  the right hand side in~\eqref{eq:TmpCompact} is zero for each $\eta \in N_{L}$, $z \in \zset$, $x \in  F_{z}$, hence
\begin{align*}
\begin{cases}
[\restmatrix
\bvactpath(\theta,z)]^{\top}
\left(\rmapi_{\eta}(\theta')-\rmapi_{\eta}(\theta)\right) = \vec{0}_{1 \times N_{0}}\\
\bvactpath^{\top}(\theta,z) \left(\rmaph_{\eta}(\theta')-\rmaph_{\eta}(\theta)\right)   = 0.
\end{cases}
\end{align*}
Since this holds for every $\eta \in N_{L}$, $z \in \zset$, in light of~\eqref{eq:spacesfinite} this establishes
 that
 \begin{align}
\forall \eta \in N_{L},
\begin{cases}
\rmapi_{\eta}(\theta')-\rmapi_{\eta}(\theta) & \in \underbrace{\Aspaceperp \times \ldots \times \Aspaceperp}_{N_{0}\ \text{times}}\\
\rmaph_{\eta}(\theta')-\rmaph_{\eta}(\theta) & \in \Aspacebiasperp
\end{cases}
\label{eq:orthoexplicit}
\end{align}
and we conclude using \autoref{lem:VCartesian} and the fact that $\card{F} \leq \card{\zset} \times (N_{0}+1)$.

{\bf b)} 
Since $K \subset \xsetcont$, for each $z \in K$ there are $\epsilon(z),r(z)>0$ such that: for each $\theta' \in B(\theta,\epsilon(z))$, $x \in B(z,r(z))$, $\bvactpath(\theta',x)= \bvactpath(\theta,z)$. 
Since $K$ is compact and $K \subset \cup_{z \in K} B(z,r(z))$, there is a finite set $\mathcal{Z} \subset K$ such that $K \subset \cup_{z \in \mathcal{Z}} B(z,r(z))$. Denote $\epsilon' := \min_{z \in \mathcal{Z}} \epsilon(z)>0$. 
Considering $\theta' \in B(\theta,\epsilon')$ such that $\rmap(\theta')-\rmap(\theta) \in \linspace{V}(\theta)$, we now show that $\realiz{\theta'}(x)=\realiz{\theta}(x)$ for each $x \in K$.
Given $x \in K$, since there is $z \in \zset$ such that $x \in B(z,r(z))$, we have
\begin{align}
\bvactpath(\theta',x) &= \bvactpath(\theta,z) = \bvactpath(\theta,x).\label{eq:activequal}
\end{align}
For each $\eta \in N_{L}$, since by \autoref{lem:VCartesian} $\rmap(\theta')-\rmap(\theta) \in \linspace{V}(\theta)$ is equivalent to~\eqref{eq:orthoexplicit}, we get
\begin{align*}
[\restmatrix
\bvactpath(\theta',x)]^{\top}\rmapi_{\eta}(\theta')
&\stackrel{\eqref{eq:activequal}}{=} 
[\restmatrix\bvactpath(\theta,x)]^{\top}\rmapi_{\eta}(\theta')
\stackrel{\eqref{eq:orthoexplicit}}{=}  
[\restmatrix\bvactpath(\theta,x)]^{\top}\rmapi_{\eta}(\theta)\\
\bvactpath^{\top}(\theta',x)\rmaph_{\eta}(\theta')
&\stackrel{\eqref{eq:activequal}}{=}   
\bvactpath^{\top}(\theta,x)\rmaph_{\eta}(\theta')
\stackrel{\eqref{eq:orthoexplicit}}{=} 
\bvactpath^{\top}(\theta,x)\rmaph_{\eta}(\theta).
\end{align*}
Using~\eqref{eq:RexpressionBis} we conclude that $\realiz{\theta'}(x)_{\eta} =  \realiz{\theta}(x)_{\eta}$ for all $\eta \in N_{L}$, i.e., $\realiz{\theta'}(x) =  \realiz{\theta}(x)$.
\end{proof}

\section{Identifiability for shallow neural networks}\label{sec:shallowcase}
In this section we focus on shallow networks, for which the set $\pathset{Q} =\pathset{Q}_{1}$ of paths is in bijection with the set $H = N_{1}$ of hidden neurons. Identifying these sets the activation vectors also coincide $\vactpath(\theta,x) = \va_{1}(\theta,x) \in \RR^{\pathset{Q}}=\RR^{N_{1}}=\RR^{H}$. After giving a complete characterization of the activation space $\Aspacebias$ using the notion of twin neurons, we show that the absence of twin neurons implies non-degeneracy (hence local S-identifiability), and that its combination with irreducibility implies PS-identifiability. Finally, we discuss what happens in the presence of twin neurons.

\subsection{Activation spaces and twin neurons}

Whenever $\theta$ is admissible, each hidden neuron $\nu\in H$ is not dead, i.e. $\vw_{\bullet \to \nu} \neq 0$ and $\vw_{\nu \to \bullet} \neq 0$. According to Definition~\ref{def:shallowtwins}, neurons are twins if their extended vectors $(\vw_{\bullet \to \nu},b_{\nu})$ are colinear.
This defines an equivalence relation, and the hidden layer $H = N_{1}$ can be partitioned into \emph{equivalence classes of twin neurons}, denoted 
\[
T_{c} \subset H, 1 \leq c \leq C.
\]
Each equivalence class $T_{c}$ is partitioned into $I_{c},J_{c}$, where all neurons
 in $I_{c}$ are positive twins, all neurons in $J_{c}$ are positive twins, and $\nu \in I_{c},\nu' \in J_{c}$ 
 are negative twins. By convention $I_{c}$ is always non-empty, while $J_{c}$ may be empty if there are no negative twins in $T_{c}$. 
 For each class, we can define  a {\em class signature} vector
 \[
 \twinsignature{c} = \indic{I_{c}}-\indic{J_{c}} \in \RR^{H},
 \]
 which is zero out of $T_{c}$, with $\pm 1$ entries on $T_{c}$, and has at least one $+1$ entry. When $T_{c}$ contains both positive and negative twins, $\twinsignature{c}$ is only defined up to a global sign. An equivalence class is said to be nontrivial if its cardinal is at least two. Equipped with these notions, we prove in Appendix~\ref{app:proofAthetaShallow} the following characterization of activation spaces.

\begin{lemma}\label{le:AthetaShallow}
Consider an admissible parameter $\theta$ on a shallow network architecture. Using the notations introduced above, its activation spaces are
\begin{align}
\label{eq:AthetaShallowNoBias}
\Aspace= \linspan{\indic{H},\twinsignature{c},1 \leq c \leq C} &\subseteq \RR^{H}\\
\label{eq:AthetaShallowBias}
\Aspacebias= \linspan{(\indic{H},2),(\twinsignature{c},0),1 \leq c \leq C} & \subseteq \RR^{H+1}.
\end{align}
\end{lemma}

\subsection{Proof of~\autoref{lem:NoTwinImpliesLSI}: no twins implies non-degeneracy}
\label{sec:proofMainTheoremLSI}

\autoref{lem:NoTwinImpliesLSI} is a direct consequence of the combination of \autoref{le:SufficientConditionLocalIdentifiability} with the following two results.

\begin{lemma}\label{le:trivialND}
On any network architecture, if $\theta \in \RR^{\paramset}$ is admissible and $\Aspacebias = \RR^{\pathset{Q}+1}$ then\footnote{The converse does not hold: there are non-degenerate parameters with $\Aspacebias \neq \RR^{\pathset{Q}+1}$, see  \autoref{lem:Case2a}.} $\theta$ is non-degenerate with respect to any $\Theta \subset \RR^{\paramset}$ that contains it. 
\end{lemma}
\begin{proof} 
Since $\Aspacebias = \RR^{\pathset{Q}+1}$,
by \autoref{cor:EquivFRA}
$\mathtt{V}(\theta) = \{0\}$, hence $\rmap(\theta')-\rmap(\theta) \in \mathtt{V}(\theta)$ is equivalent to $\rmap(\theta') = \rmap(\theta)$. 
Since $\theta$ is admissible, this shows that $\theta$ is non-degenerate.
\end{proof}
\begin{lemma}\label{lem:shallowFRAisNoTwin}
Consider a shallow architecture and $\theta \in \RR^{\paramset}$. The equality $\Aspacebias = \RR^{\pathset{Q}+1}$ holds if, and only if, there is no twin.
When this holds, $\actdim(\theta) = |H|+1=|N_{1}|+1$.
\end{lemma}
\begin{proof}
Equivalence classes of twin neurons form a partition of $H$, hence $|C| \leq |H|=|\pathset{Q}|$. By \autoref{le:AthetaShallow}, $\Aspacebias$ is the span of $|C|+1$ vectors, hence its dimension is at most $|C|+1$. In the presence of twins we get $|C| < |H|$ hence $\Aspacebias \neq \RR^{\pathset{Q}+1}$.
In the absence of twins, each equivalence class $T_{c}$ is trivial, i.e. $|T_{c}|=1$. We obtain that $|C|=|H|$, that each signature vector $\twinsignature{c}$ is a distinct canonical vector $\delta_{c}$, and obtain $\Aspacebias = \RR^{\pathset{Q}+1}$ by \autoref{le:AthetaShallow}.
\end{proof}

\subsection{Proof of~\autoref{th:MainTheorem}: irreducibility and no twins implies PS-identifiability}\label{sec:proofMainTheorem}
By \autoref{lem:PSIFromBoundedImpliesNoTwin}, PS-identifiability from a bounded set with respect to $\Theta = \RR^{\paramset}$ implies that $\theta$ has no twins,
hence by \autoref{lem:nonlocaldegeneracy}), it is irreducible (hence admissible), and local S-identifiable, by \autoref{th:IdentImpliesLocaIdent}. 
For shallow networks, we show that conversely, irreducibility and the absence of twins imply PS-identifiability from a bounded set. 

\begin{theorem}\label{thm:PSIarchitecture}
Consider $N_{1}$, $N'_{1}$ two finite sets of indices, empty or not\footnote{We use the convention: $\sum_\emptyset = 0$.}, and integers $d,k \geq 1$.  Consider $\vc \in\RR^{k}$ and for each $\nu \in N_{1}$, let $\vv_\nu\in\RR^{k}$, $\vw_{\nu}\in \RR^{d}$ and $b_\nu\in\RR$. Define
  \[
  \vec{\varphi}(x)  = \sum_{\nu\in{N_{1}}}\vv_\nu\ReLU(\langle \vw_{\nu},x\rangle+b_{\nu}) + \vc,\qquad x \in \RR^{d}.
  \] 
  Similarly define $\vec{\psi}(x)$ with $\vv'_\nu \in \RR^{k}$, $\vw'_\nu \in \RR^{d}$, $b'_\nu \in \RR$ for $\nu \in N'_{1}$, and $\vc' \in \RR^{k}$.  
  \begin{enumerate}[a)]
  \item \label{it:PSIarchi} Assume that
  \begin{itemize}
  \item $\{(\vw_{\nu},b_\nu)\}_{\nu \in N_{1}}$ are pairwise not collinear, and $\vv_{\nu},\vw_{\nu} \neq 0$;
  \item $\{(\vw'_{\nu},b'_\nu)\}_{\nu \in N'_{1}}$ are pairwise not collinear, and $\vv'_{\nu},\vw'_{\nu} \neq 0$.
  \end{itemize}
  If $\vec{\varphi}(x)=\vec{\psi}(x)$ for every $x \in \RR^{d}$ then $\card{N_{1}} = \card{N'_{1}}$.
\item \label{it:PSIirred} Assume that  $\{(\vw_{\nu},b_\nu)\}_{i \in N_{1}}$ are pairwise not collinear, and 
   \begin{equation}\label{eq:Irred1D}
    \sum_{\nu \in{T}} \vv_\nu \vw_{\nu}^{\top} \ne 0,\quad \text{for all non-empty}\ {T} \subset N_{1}.
  \end{equation}
  There exists a bounded set $\xset \subseteq \RR^{N_{0}}$ (which depends on $\theta$) such that: if  $N'_{1} = N_{1}$ and $\vec{\varphi}(x)=\vec{\psi}(x)$ for every $x \in \xset$, then\footnote{Let us emphasize that here no further assumption is made on $\vw'_{\nu},b'_\nu,\vv'_{\nu}$, $\nu \in N_{1}$.} $\vc=\vc'$ and there exists a permutation $\pi$ of $N_{1}$ and $\lambda_{\nu}>0$, $\nu \in N_{1}$ such that
  \begin{align}\label{eq:mainThmCclPSI}
  \forall \nu \in N_{1}: \vv'_{\pi(\nu)} = \lambda_{\nu}^{-1} \vv_{\nu};\ \vw'_{\pi(\nu)} = \lambda_{\nu}\vw_{\nu}\ \text{and}\ b'_{\pi(\nu)}=\lambda_{\nu}b_{\nu   }.
  \end{align}
  \end{enumerate}
\end{theorem}
\begin{proof}
As a preliminary, consider $\nu \in N_{1}$ and denote $\mathcal{V}_{\nu} := \{x \in \RR^{N_{0}}: \langle \vw_{\nu},x\rangle+b_{\nu}=0\}$.  Since $\vw_{\nu} \neq 0$, the set $\mathcal{V}_{\nu}$ is a hyperplane which matches the set $\Gamma_{\nu}(\theta)$ from \autoref{def:characXset} when considering $\theta$ such that $\vec{\varphi} = \realiz{\theta}$. 
As none of the $(\vw_{\nu},b_{\nu})$ is collinear to another, the hyperplanes associated to $\nu \neq \nu' \in N_{1}$ are distinct. As $\vv_{\nu} \neq 0$ for every $\nu \in N_{1}$ and $\vec{\varphi}$ is continuous and piecewise affine, this function is differentiable exactly on the complement of  $\mathcal{T} := \cup_{\nu \in{N_{1}}} \mathcal{V}_{\nu}$, which is a union of $\card{N_{1}}$ distinct hyperplanes. 

{\bf a)}  Similarly, since none of the $(\vw'_{\nu},b'_{\nu})$ is collinear to another and $\vv'_{\nu} \neq 0, \vw'_{\nu} \neq 0$ for each $\nu \in N'_{1}$, the function $\vec{\psi}$ is differentiable exactly on the complement of a union of  $\card{N'_{1}}$ distinct hyperplanes, $\mathcal{T}' = \cup_{\nu \in{N}'} \mathcal{V}'_{\nu}$, where $\mathcal{V}'_{\nu} := \{x \in \RR^{N_{0}}: \langle \vw'_{\nu},x\rangle+b'_{\nu}=0\}$. Note that $\mathcal{T}$ may be empty since $N_{1}$ may be empty, and similarly for $\mathcal{T}'$. Since $\vec{\varphi}= \vec{\psi}$, we have $\mathcal{T} = \mathcal{T}'$ hence $\card{N_{1}}$ = $\card{N'_{1}}$, otherwise there would exist one point $x \in \RR^{d}$ where one function would be differentiable and the other not.  

{\bf b)} We now assume $N'_{1}$ = $N_{1}$, but make no specific assumption on $\vv'_{\nu} \in \RR^{k}$, $\vw'_{\nu} \in \RR^{d}$, $b'_{\nu} \in \RR$ for $\nu \in N_{1}$ or on $\vc' \in \RR^{k}$. 
By~\eqref{eq:Irred1D} with $T = \{\nu\}$ we have $\vv_{\nu}\vw_{\nu}^{\top} \neq 0$ hence, as in the preliminary, $\mathcal{V}_{\nu}$, $\nu \in N_{1}$ are pairwise distinct hyperplanes. Consider an arbitrary hidden neuron $\nu \in N_{1}$. As the hyperplanes $\{\mathcal{V}_{\mu}\}_{\mu \in N_{1}}$ are pairwise distinct, there exist $x_{\nu} \in \mathcal{V}_{\nu}$ and $\epsilon_{\nu}>0$ such that $\Omega_{\nu} := B(x_{\nu},\epsilon_{\nu})$ satisfies $\Omega_{\nu} \cap \mathcal{T} = \Omega_{\nu} \cap \mathcal{V}_{\nu}$. We will show that the result holds with  $\xset := \cup_{\nu \in N_{1}} \Omega_{\nu}$, which is easily seen to be bounded. 

From now, assume that $\vec{\psi}(x) = \vec{\varphi}(x)$ for every $x \in \xset$.

For each $\nu \in \widehat{N}_{1} := \{\nu \in N_{1}: \vv'_{\nu} \neq 0, \vw'_{\nu} \neq 0\}$, since $\vw'_{\nu} \neq0$, the set $\mathcal{V}'_{\nu}$ is a hyperplane. Consider the equivalence relation on $\widehat{N}_{1}$ defined by: $\nu \sim \mu \Leftrightarrow \mathcal{V}'_{\nu}=\mathcal{V}'_{\mu}$, and the resulting quotient set $\bar{N_{1}} = \widehat{N}_{1}/\!\!\!\sim$. For each equivalence class $\bar{\nu} \in \bar{N_{1}}$, denote $\mathcal{V}'_{\bar{\nu}}$ the common hyperplane associated to every $\nu \in \bar{\nu}$, and set $\bar{\mathcal{T}} = \cup_{\bar{\nu} \in \bar{N}_{1}} \mathcal{V}'_{\bar{\nu}}$. We will prove below that there exists an injective map $\pi: N_{1} \to \bar{N}_{1}$ such that $\mathcal{V}_{\nu} = \mathcal{V}'_{\pi(\nu)}$ for every $\nu \in N_{1}$. This will imply that $\card{\bar{N}_{1}} \geq \card{N_{1}}$, and since
$\card{\bar{N}_{1}} \leq \card{\widehat{N}_{1}} \leq \card{N_{1}}$, it will follow that $\widehat{N}_{1}=N_{1}$ (hence $\vv'_{\nu} \neq 0$, $\vw'_{\nu} \neq 0$ for every $\nu \in N_{1}$) and that each equivalence class $\bar{\nu}$ is a singleton. In other words, $\pi$ is indeed a permutation of $N_{1}$, and the hyperplanes $\mathcal{V}'_{\{\nu\}}$, $\nu \in N_{1}$ are pairwise distinct. 

To build $\pi$, consider a hidden neuron $\nu \in N_{1}$. For the sake of contradiction, assume that $\mathcal{V}'_{\bar{\mu}} \neq \mathcal{V}_{\nu}$ for every $\bar{\mu} \in \bar{N}_{1}$. This implies the existence of $x'_{\nu} \in \Omega_{\nu} \cap \mathcal{V}_{\nu}$ and of $\epsilon'_{\nu}>0$ such that $\Omega'_{\nu} := B(x'_{\nu},\epsilon'_{\nu}) \subseteq \Omega_{\nu}$ and $\Omega'_{\nu} \cap \bar{\mathcal{T}'} = \emptyset$ and $\Omega'_{\nu} \cap \mathcal{V}_{\nu} = \mathcal{V}_{\nu}$. Since $\Omega'_{\nu} \cap \bar{\mathcal{T}} = \emptyset$, the function $\vec{\psi}$ is affine linear on $\Omega'_{\nu}$, hence it has constant Jacobian on $\Omega'_{\nu}$. 
Denote $\Omega_{\nu}^{+} := \{x \in \Omega'_{\nu}: \langle \vw_{\nu},x\rangle+b_{\nu}>0\}$, $\Omega_{\nu}^{-} := \{x \in \Omega'_{\nu}: \langle \vw_{\nu},x\rangle+b_{\nu}<0\}$, and observe that both sets are non-empty.
For any $x\in\Omega'_{\nu}\backslash\mathcal{V}_{\nu} = \Omega_{\nu}^{+} \cup \Omega_{\nu}^{-}$, the function $\vec{\varphi}$ is differentiable and its Jacobian is
$\vec{\varphi}'(x) = \vv_\nu \vw_{\nu}^{\top} H(\langle \vw_{\nu},x\rangle+b_{\nu}) + \vd$
 where $\vd \in \RR^{k}$ and
\begin{align*}
 H(t) := \begin{cases}
  1,& \text{if}\ t>0\\
  0,& \text{otherwise}.
  \end{cases}  
\end{align*}
For each $x_{\nu}^{+} \in \Omega_{\nu}^{+},x_{\nu}^{-} \in \Omega_{\nu}^{-}$ we have $H(\langle \vw_{\nu},x_{\nu}^{+}\rangle+b_{\nu})-H(\langle \vw_{\nu},x_{\nu}^{-}\rangle+b_{\nu})
  =1$, hence $\vec{\varphi}'(x_\nu^+) - \vec{\varphi}'(x_\nu^-) = \vv_{\nu}\vw_{\nu}^{\top}$. As $\vec{\psi} = \vec{\varphi}$ on $\xset \supseteq \Omega_{\nu} \supseteq \Omega'_{\nu}$ and $\vec{\psi}$ has constant Jacobian on $\Omega'_{\nu}$, it follows that $\vv_{\nu}\vw_{\nu}=0$, which contradicts our assumptions.
Hence, there is $\bar{\mu} \in \bar{N}_{1}$ such that $\mathcal{V}'_{\bar{\mu}} = \mathcal{V}_{\nu}$. Since the hyperplanes $\{\mathcal{V}'_{\bar{\nu}}\}_{\bar{\nu} \in \bar{N}_{1}}$ are pairwise disjoint by construction, such a $\bar{\mu}$ is unique and we define $\pi(\nu) := \bar{\mu}$. Since this holds for every $\nu \in N_{1}$, we can define the map $\pi: N_{1} \to \bar{N}_{1}$ with $\pi(\nu) := \bar{\mu}$. For $\nu \neq \nu'$ we have $\mathcal{V}'_{\pi(\nu')} = \mathcal{V}_{\nu'} \neq \mathcal{V}_{\nu} = \mathcal{V}'_{\pi(\nu)}$ since the hyperplanes $\{\mathcal{V}_{\nu}\}_{\nu \in N_{1}}$ are pairwise distinct. This proves the injectivity of $\pi$. As we have seen, this means that indeed $\pi$ is a permutation of $N_{1}$. Without loss of generality, to simplify notations, we assume from now on that $\pi$ is the identity. 

\modif{For each $\nu \in N_{1}$, since $\mathcal{V}'_{\nu} = \mathcal{V}'_{\pi(\nu)} = \mathcal{V}_{\nu}$} there is a nonzero $\lambda_{\nu} \in \RR$ such that 
\[
(\vw'_{\nu},b'_{\nu}) = \lambda_{\nu} (\vw_{\nu},b_{\nu}).
\]
Reasoning as above, with $\Omega_{\nu}^{\pm}$ defined using $\Omega'_{\nu}:= \Omega_{\nu}$, we obtain that 
\[
\vec{\varphi}'(x_\nu^+) - \vec{\varphi}'(x_\nu^-) = \vv_{\nu}\vw_{\nu}^{\top}.
\] 
for each $x_{\nu}^{+} \in \Omega_{\nu}^{+},x_{\nu}^{-} \in \Omega_{\nu}^{-}$, and that for each  $x\in\Omega_{\nu}\backslash\mathcal{V}_{\nu}$, the Jacobian of $\psi$ satisfies $\vec{\psi}'(x) = \vv'_{\nu}(\vw'_\nu)^{\top} H(\langle \vw'_\nu,x\rangle+b'_{\nu}) + \vd'$ with some $\vd' \in \RR^{k}$, 
hence for each $x_{\nu}^{+} \in \Omega_{\nu}^{+},x_{\nu}^{-} \in \Omega_{\nu}^{-}$ 
\[
\vec{\psi}'(x_\nu^+) - \vec{\psi}'(x_\nu^-) = \vv'_{\nu}(\vw'_{\nu})^{\top} \left(H(\langle \vw'_{\nu},x_{\nu}^{+}\rangle+b'_{\nu})-H(\langle \vw'_{\nu},x_{\nu}^{-}\rangle+b'_{\nu})\right).
\] 
Since $(\vw'_{\nu},b'_{\nu}) = \lambda_{\nu} (\vw_{\nu},b_{\nu})$ and $\sign{\langle \vw_{\nu},x_{\nu}^{\pm}\rangle+b_{\nu}} = \pm 1$, we have
\begin{align*}
  H(\langle \vw'_{\nu},x_{\nu}^{+}\rangle+b'_{\nu})-H(\langle \vw'_{\nu},x_{\nu}^{-}\rangle+b'_{\nu})&=\sign{\lambda_{\nu}}.
\end{align*}
Moreover, as  $\vec{\psi}=\vec{\varphi}$ on $\xset$, we have $\vec{\varphi}'(x_\nu^+) - \vec{\varphi}'(x_\nu^-)= \vec{\psi}'(x_\nu^+) - \vec{\psi}'(x_\nu^-)$, hence
  \begin{equation*}
    \vv_\nu\vw_\nu^{\top} = \vv'_\nu(\vw'_\nu)^{\top} \sign{\lambda_{\nu}}.
  \end{equation*} 
  Since $\vw'_{\nu} = \lambda_{\nu} \vw_{\nu}$, this simplifies to 
  \begin{equation}\label{eq:MainThmPfStepX}
    \vv_\nu\vw_\nu^{\top} = \vv'_\nu(\vw'_\nu)^{\top} \sign{\lambda_{\nu}} = 
    \vv'_\nu\vw_\nu^{\top} \lambda_{\nu}\sign{\lambda_{\nu}}
    = 
\abs{\lambda_{\nu}}    \vv'_\nu\vw_\nu^{\top} 
  \end{equation}
Hence, $\vv'_{\nu} = \vv_{\nu}/\abs{\lambda_{\nu}}$ for each  $\nu \in N_{1}$.

 To conclude, it is enough to prove that $\lambda_{\nu}>0$ for every $\nu \in N_{\nu}$. Using~\eqref{eq:MainThmPfStepX}, we can re-write the equality $\vec{\varphi}(x) = \vec{\psi}(x)$ for every $x$ as follows:
  \begin{equation}\label{eq:simp}
    \sum_{\nu\in{N_{1}}} \vv_\nu\Big[\ReLU(\langle \vw_{\nu},x\rangle+b_{\nu}) - \abs{\lambda_{\nu}}^{-1} \ReLU(\underbrace{\langle \vw'_{\nu},x\rangle+b'_{\nu}}_{\lambda_{\nu}(\langle \vw_{\nu},x\rangle+b_{\nu})})\Big]+ \vc - \vc' = 0.
  \end{equation}
 Now, we observe that 
  \begin{equation}\label{eq:cases}
    \ReLU(\langle \vw_{\nu},x\rangle+b_{\nu}) - \abs{\lambda_{\nu}}^{-1} \ReLU(\lambda_{\nu}(\langle \vw_{\nu},x\rangle+b_{\nu})) = 
    \begin{cases}
      0 & \text{if } \sign{\lambda_{\nu}} = 1 \\
     \langle \vw_{\nu},x\rangle+b_{\nu} & \text{if } \sign{\lambda_\nu} = -1
    \end{cases}
  \end{equation}
  We now show that ${T} := \{\nu\in{N_{1}} \mid  \sign{\lambda_{\nu}} = -1\} = \emptyset$. Using~\eqref{eq:cases}, we re-write~\eqref{eq:simp} as:
  \begin{equation}
    \sum_{\nu \in{T}} \vv_{\nu} (\langle \vw_{\nu},x\rangle + b_\nu) + \vc - \vc' = 0.
  \end{equation} 
  Since this is valid for all $x \in \RR^{N_{0}}$ we get $\vc=\vc'$ and $\sum_{\nu\in{T}} \vv_\nu\vw_\nu^{\top} =0$. In light of~\eqref{eq:Irred1D} the latter implies ${T} = \emptyset$, hence $\sign{\lambda_{\nu}}=1$ for all $\nu \in {N_{1}}$. 
\end{proof}

\subsection{Local S-identifiability despite the presence of twins}
It is natural to wonder if there exists shallow networks with twins that are nevertheless either non-degenerate, or locally S-identifiable, or PS-identifiable. Positive twins are excluded (for any network depth) by \autoref{lem:PSIFromBoundedImpliesNoTwin}, hence we can focus on the case where  there are $K \geq 1$ nontrivial classes of twins, each made of a single pair of (distinct) negative twins (as any equivalence class with at least three twins necessarily contains two positive ones). We detail here the case $K=1$ and leave to future work a more detailed analysis of what happens for $K \geq 2$.
\begin{lemma}[{\bf Single pair of negative twins}]\label{lem:Case2a}
Consider a shallow network architecture. If  $\theta \in \Theta \subseteq \RR^{\paramset}$ is admissible with a single pair of negative twins, $\{\nu_{1},\nu_{2}\} \subseteq H$, then
\begin{equation}
\Aspaceperp = \{0\}\ \text{and}\ \Aspacebiasperp = \linspan{\delta_{\nu_{1}}+\delta_{\nu_{2}}-\delta_{\star}}\neq \{0\}
\end{equation}
where $\delta_{\nu} \in \RR^{H+1}, \nu \in H$ is the $\nu$-th canonical eigenvector, and $\delta_{\star} = (\vec{0}_{H},1)$. 

Moreover, if at least one of the following conditions holds:
\begin{enumerate}[i)]
\item $\vw_{\nu_{1}\to\bullet}$, $\vw_{\nu_{2}\to\bullet}$ are linearly independent (which is only possible if $|N_{2}| \geq 2$); or 
\item $\Theta$ is contained in the set of parameters with zero output bias;
\end{enumerate}
then $\theta$ is non-degenerate with respect to $\Theta$. Conversely, if 
\begin{enumerate}[i)]
\item[iii)] \label{it:caseiii} $\vw_{\nu_{1}\to\bullet}$, $\vw_{\nu_{2}\to\bullet}$ are linearly dependent and $\theta$ belongs to the interior of $\Theta$, 
\end{enumerate}
then $\theta$ is degenerate with respect to $\Theta$.
\end{lemma}
\begin{remark}
Inspecting the proof shows that the assumption in~iii) that $\theta$ is in the interior of $\Theta$ can be relaxed to: for small enough $\epsilon$, each parameter $\theta' \in B(\theta,\epsilon)$ differing from $\theta$ only in terms of biases belongs to $\Theta \cap B(\theta,\epsilon)$. 
\end{remark}
The proof is in Appendix~\ref{app:proofWithTwins}. We are now equipped to show with an example that non-degeneracy and local-identifiability are distinct concepts.
\begin{example}[Absolute value]\label{ex:LSInotND}
Consider a shallow architecture with scalar input and output and two hidden neurons. The absolute value can be written as $\abs{x} = \ReLU(x)+\ReLU(-x) = \realiz{\theta}$ where $\theta = (w_{\mu\to\nu_{1}}=1,w_{\mu\to\nu_{2}}=-1,b_{\nu_{1}}=b_{\nu_{2}}=0,w_{\nu_{1}\to\eta}=w_{\nu_{2}\to\eta}=1, b_{\eta}=0)$ has a single pair of negative twins. 
This parameter $\theta$ satisfies the following properties
\begin{enumerate}[i)]
\item it is not PS-identifiable from any bounded set $\xset \subset \RR$ (by  \autoref{lem:PSIFromBoundedImpliesNoTwin});
\item it is not locally S-identifiable from any finite $F \subset \xsetcont$ (i.e., it is degenerate, see below); 
\item it is PS-identifiable (hence locally S-identifiable) \emph{from $\xset = \RR$};
\item it is locally S-identifiable from $F \cup \{0\}$ for some finite set $F \subset \xsetcont$;
\end{enumerate}
The last two points are detailed in \autoref{app:AbsReLU}. Let us detail ii) here. Since $|N_{2}|=1$, by \autoref{lem:Case2a}-iii) we get that $\theta$ is degenerate with respect to $\Theta = \RR^{\paramset}$, i.e. not locally S-identifiable from any finite $F \subseteq \xsetcont$. Indeed, if $F \subseteq\xsetcont = \RR \backslash\{0\}$  is finite then $F \subset (-\infty,-t] \cup [t,+\infty)$ for some $t>0$, and $\mathtt{abs}$ coincides on $F$ with (see \autoref{fig:exReLU}-(c))
\[
\ReLU(x-t)+\ReLU(-(x+t))+t = 
\begin{cases}
-x, & x \leq -t\\
t, & |x| \leq t\\
x, & x \geq t.
\end{cases}
= \realiz{\theta'}(x)
\]
where $\theta'$ has nonzero biases, so that $\theta' \not \sim_{PS} \theta$.

\end{example}
\begin{example}[Revisiting the identity function from  \autoref{ex:1twinpair}]
\label{ex:1twinpairbis}
The identity function from \autoref{ex:1twinpair} is another example with a single pair of twin neurons. With $\Theta = \RR^{\paramset}$ the parameter $\theta_{0}$ is not locally S-identifiable (from $\xset = \RR$) as already explained in \autoref{ex:1twinpair}.
With $\Theta = \Theta_{0} \subsetneq \RR^{\paramset}$ the set of parameters with zero output bias, $\theta_{0}$ is on the contrary PS-identifiable from $\RR$ (see details in \autoref{app:1twinpairbis}). It can also be shown that 
$\theta_{0}$ is 
non-degenerate with respect to $\Theta_{0}$, using arguments similar to those used in \autoref{app:AbsReLU} to prove item iv) of 
\autoref{ex:LSInotND}. This illustrates the fact that, in the presence of a pair of negative twins, many things can happen: the parameter can be PS-identifiable and non- degenerate, or not even locally S-identifiable. 
\end{example}

\subsection{Discussion of the role of activation spaces}
For shallow irreducible networks, PS-identifiability from a bounded set is equivalent (\autoref{th:MainTheorem}) to the absence of twin neurons, which corresponds (\autoref{lem:shallowFRAisNoTwin}) to a completeness property of the activation space that reads $\Aspacebiasperp = \{0\}$.} The property $\Aspacebiasperp = \{0\}$ also implies non-degeneracy (\autoref{le:trivialND}), yet a consequence of \autoref{lem:Case2a} is that the converse does not generally hold 
(and that the weaker assumption $\Aspaceperp = \{0\}$ is no longer sufficient to imply non-degeneracy). An exception occurs for scalar-valued shallow networks.


\begin{lemma}\label{le:necessaryconditionLNDshallowscalar}
Consider a \emph{scalar-valued} shallow architecture ($|N_{2}|=1$). If $\theta$ belongs to the interior of $\Theta \subseteq \RR^{\paramset}$ and is non-degenerate with respect to $\Theta$ then 
$\Aspacebiasperp = \{0\}$.
\end{lemma}
This exception is a consequence of the following result.
\begin{lemma}\label{le:UisOpen}
Consider a  scalar-valued shallow network architecture. If $\theta$ is admissible then there is $0<C<\infty$ such that: for each 
$\vz \in \RR^{\pathset{Q}+1}$, there exists  $\theta' \in B(\theta,C\|\vz\|_{\infty})$ 
\begin{align}
\rmapi_{\eta}(\theta') -\rmapi_{\eta}(\theta) &= \vec{0}_{\pathset{Q}_{1} \times N_{0}},\\
\rmaph_{\eta}(\theta')  - \rmaph_{\eta}(\theta) & = \vz
\end{align}
where $\eta$ is the single output neuron constituting the output layer $N_{L}$.
The parameters $\theta$ and $\theta'$ differ only in terms of biases.
 \end{lemma}
 \begin{proof} 
 Write $\vz = (\vy,\gamma)$ with $\vy \in \RR^{\pathset{Q}} = \RR^{H}$ and $\gamma \in \RR$.  For each hidden neuron $\nu \in H = N_{1}$,
denote $v_{\nu} = v_{\nu\to\eta}$ the unique weight from neuron $\nu$ to the single output neuron.
 For each input neuron $\mu \in N_{0}$, the $\mu$-th column of $\rmapi_{\eta}(\theta)$ is $\rmapi_{\mu\to\eta}(\theta) := (w_{\mu\to \nu}v_{\nu})_{\nu \in H}$, and
$\rmaph_{\eta}(\theta) = 
\left((b_{\nu}v_{\nu})_{\nu \in H},
b_{\eta}\right)^{\top}$.
To prove the result we define $\theta'$ with identical weights as $\theta$, $v'_{\nu}:=v_{\nu}$, $w'_{\mu\to\nu} := w_{\mu\to\nu}$, and set the output bias to $b'_{\eta}:=b_{\eta} + \gamma$. This implies $w'_{\mu\to \nu}v'_{\nu}=w_{\mu\to \nu}v_{\nu}$ for every $\nu \in H, \mu \in N_{0}$, hence $\rmapi_{\eta}(\theta') =\rmapi_{\eta}(\theta)$. We now seek $b'_{\nu}$ such that $b'_{\nu}v'_{\nu}=b_{\nu}v_{\nu}+y_{\nu}$, for each $\nu \in H$. Since $\theta$ is admissible, $v_{\nu} \neq 0$ for all $\nu \in H$, hence we can choose $b'_{\nu}:=b_{\nu}+y_{\nu}/v_{\nu}$. We conclude with $C_{\theta}$ driven by $\min_{\nu}1/|v_{\nu}|$. 
 \end{proof}
\begin{proof}[Proof of \autoref{le:necessaryconditionLNDshallowscalar}]
We prove the contraposition. Assume that $\theta$ is admissible, that it belongs to the interior of $\Theta \subset \RR^{\paramset}$, and that 
$\Aspacebiasperp \neq \{0\}$. 
Since $\theta$ is in the interior of $\Theta$, there is $\eta>0$ such that $B(\theta,\eta) \subseteq \Theta$. For each $0<\epsilon<\eta$ there exists $\vz \in \Aspacebiasperp$ with norm $\|\vz\|_{\infty}=\epsilon/C$ where $C$ is the constant from \autoref{le:UisOpen}. Since $\theta$ is admissible, by \autoref{le:UisOpen} and the characterization of $\mathtt{V}(\theta)$ (\autoref{lem:VCartesian}) there exists $\theta' \in B(\theta,C \|\vz\|_{\infty}) = B(\theta,\epsilon) = \Theta \cap B(\theta,\epsilon)$ such that $\vec{0} \neq \rmap(\theta')-\rmap(\theta) \in \mathtt{V}(\theta)$. This shows that $\theta$ is degenerate with respect to $\Theta$. 
\end{proof}
\begin{remark}
The assumption that $\theta$ is in the interior of $\Theta$ can be relaxed to: each parameter $\theta' \in B(\theta,\epsilon)$ differing from $\theta$ only in terms of biases belongs to $\Theta \cap B(\theta,\epsilon)$.
\end{remark}

\section*{Acknowledgements}

The authors are thankful to Fran{\c{c}}ois Malgouyres for the interesting discussions on invariant embeddings of linear and ReLU networks we had at different stages of advancement of this work and Joachim Bona-Pellissier for his technical comments on early versions of this article. The authors thank Elisa Riccietti for her feedback that helped a lot improve the readability of this paper. The authors are thankful to Hervé Jégou and Benjamin Graham for their continuous support since the genesis of this project a few years ago. This project was supported in part by the AllegroAssai ANR project ANR-19-CHIA-0009.

\bibliography{main}

\begin{thebibliography}{10}

\bibitem{krizhevsky2012deep}
Alex Krizhevsky, Ilya Sutskever, and Geoffrey~E Hinton.
\newblock Imagenet classification with deep convolutional neural networks.
\newblock In {\em Advances in Neural Information Processing Systems}, 2012.

\bibitem{daubechies2019nonlinear}
I.~Daubechies, R.~DeVore, S.~Foucart, B.~Hanin, and G.~Petrova.
\newblock Nonlinear approximation and (deep) relu networks, 2019.

\bibitem{devore2020neural}
Ronald DeVore, Boris Hanin, and Guergana Petrova.
\newblock Neural network approximation, 2020.

\bibitem{sussman}
H{\'{e}}ctor~J. Sussmann.
\newblock Uniqueness of the weights for minimal feedforward nets with a given
  input-output map.
\newblock {\em Neural Networks}, 1992.

\bibitem{kainen1994uniqueness}
Paul Kainen, Vera Kurková, Vladik Kreinovich, and Ongard Sirisengtaksin.
\newblock Uniqueness of network parameterizations and faster learning.
\newblock {\em Preprint}, 1994.

\bibitem{kurkova}
V\v{e}ra K\r{u}rkov\'{a} and Paul~C. Kainen.
\newblock Functionally equivalent feedforward neural networks.
\newblock {\em Neural Comput.}, 1993.

\bibitem{albertini}
Francesca Albertini, Eduardo~D. Sontag, and Vincent Maillot.
\newblock Uniqueness of weights for neural networks.
\newblock In {\em Artificial Neural Networks with Applications in Speech and
  Vision}, 1993.

\bibitem{Fefferman1994}
Charles Fefferman.
\newblock Reconstructing a neural net from its output.
\newblock {\em Revista Matemática Iberoamericana}, 1994.

\bibitem{rolnick2019reverseengineering}
David Rolnick and Konrad~P. Kording.
\newblock Reverse-engineering deep relu networks, 2019.

\bibitem{fornasier2019robust}
Massimo Fornasier, Timo Klock, and Michael Rauchensteiner.
\newblock Robust and resource efficient identification of two hidden layer
  neural networks, 2019.

\bibitem{phuong_functional_2020}
Mary Phuong and Christoph~H Lampert.
\newblock Functional vs. parametric equivalence of {ReLU} networks.
\newblock {\em ICLR}, 2020.

\bibitem{malgouyres2018multilinear}
Francois Malgouyres and Joseph Landsberg.
\newblock Multilinear compressive sensing and an application to convolutional
  linear networks.
\newblock {\em SIAM}, 2018.

\bibitem{malgouyres2020stable}
Francois Malgouyres.
\newblock On the stable recovery of deep structured linear networks under
  sparsity constraints.
\newblock {\em Proceedings of Machine Learning Research}, 2020.

\bibitem{neyshabur_norm-based_2015}
Behnam Neyshabur, Ryota Tomioka, and Nathan Srebro.
\newblock Norm-{Based} {Capacity} {Control} in {Neural} {Networks}.
\newblock {\em Journal of Machine Learning Research}, 2015.

\bibitem{neyshabur_path-sgd_2015}
Behnam Neyshabur, Ruslan Salakhutdinov, and Nathan Srebro.
\newblock Path-{SGD} - {Path}-{Normalized} {Optimization} in {Deep} {Neural}
  {Networks}.
\newblock {\em NIPS}, 2015.

\bibitem{meng_g-sgd_2019}
Qi~Meng, Shuxin Zheng, Huishuai Zhang, Wei~Chen 0034, Qiwei Ye, Zhi-Ming Ma,
  Nenghai Yu, and Tie-Yan Liu.
\newblock G-{SGD} - {Optimizing} {ReLU} {Neural} {Networks} in its {Positively}
  {Scale}-{Invariant} {Space}.
\newblock {\em ICLR}, 2019.

\bibitem{yi_positively_2019}
Mingyang Yi, Qi~Meng, Wei Chen, Zhi-ming Ma, and Tie-Yan Liu.
\newblock Positively {Scale}-{Invariant} {Flatness} of {ReLU} {Neural}
  {Networks}.
\newblock {\em arXiv:1903.02237 [cs, stat]}, March 2019.
\newblock arXiv: 1903.02237.

\bibitem{stock_equi-normalization_2019}
Pierre Stock, Benjamin Graham, Remi Gribonval, and Hervé Jégou.
\newblock Equi-normalization of {Neural} {Networks}.
\newblock In {\em {ICLR} 2019 - {Seventh} {International} {Conference} on
  {Learning} {Representations}}, pages 1--20, New Orleans, United States, May
  2019.

\bibitem{yuan_scaling-based_2019}
Qunyong Yuan and Nanfeng Xiao.
\newblock Scaling-{Based} {Weight} {Normalization} for {Deep} {Neural}
  {Networks}.
\newblock {\em IEEE Access}, 7:7286--7295, January 2019.
\newblock Publisher: IEEE.

\bibitem{meller_same_2019}
Eldad Meller, Alexander Finkelstein, Uri Almog, and Mark Grobman.
\newblock Same, {Same} {But} {Different} - {Recovering} {Neural} {Network}
  {Quantization} {Error} {Through} {Weight} {Factorization}.
\newblock {\em arXiv:1902.01917 [cs, stat]}, February 2019.
\newblock arXiv: 1902.01917.

\bibitem{nagel_data-free_2019}
Markus Nagel, Mart van Baalen, Tijmen Blankevoort, and Max Welling.
\newblock Data-{Free} {Quantization} {Through} {Weight} {Equalization} and
  {Bias} {Correction}.
\newblock {\em arXiv:1906.04721 [cs, stat]}, November 2019.
\newblock arXiv: 1906.04721.

\bibitem{carlini2020cryptanalytic}
Nicholas Carlini, Matthew Jagielski, and Ilya Mironov.
\newblock Cryptanalytic extraction of neural network models, 2020.

\bibitem{neyshabur2015pathsgd}
Behnam Neyshabur, Ruslan Salakhutdinov, and Nathan Srebro.
\newblock Path-sgd: Path-normalized optimization in deep neural networks.
\newblock {\em arXiv preprint arXiv:1506.02617}, 2015.

\bibitem{nagel2019datafree}
Markus Nagel, Mart van Baalen, Tijmen Blankevoort, and Max Welling.
\newblock Data-free quantization through weight equalization and bias
  correction, 2019.

\bibitem{meller2019different}
Eldad Meller, Alexander Finkelstein, Uri Almog, and Mark Grobman.
\newblock Same, same but different - recovering neural network quantization
  error through weight factorization, 2019.

\bibitem{yi2019positively}
Mingyang Yi, Qi~Meng, Wei Chen, Zhi ming Ma, and Tie-Yan Liu.
\newblock Positively scale-invariant flatness of {ReLU} neural networks, 2019.

\bibitem{meng2018mathcalgsgd}
Qi~Meng, Shuxin Zheng, Huishuai Zhang, Wei Chen, Zhi-Ming Ma, and Tie-Yan Liu.
\newblock $\mathcal{G}$-sgd: Optimizing relu neural networks in its positively
  scale-invariant space, 2018.

\bibitem{hanin2019deep}
Boris Hanin and David Rolnick.
\newblock Deep relu networks have surprisingly few activation patterns, 2019.

\bibitem{pascanu2013number}
Razvan Pascanu, Guido Montufar, and Yoshua Bengio.
\newblock On the number of response regions of deep feed forward networks with
  piece-wise linear activations, 2013.

\bibitem{montufar}
Guido Mont\'{u}far, Razvan Pascanu, Kyunghyun Cho, and Yoshua Bengio.
\newblock On the number of linear regions of deep neural networks.
\newblock In {\em Proceedings of the 27th International Conference on Neural
  Information Processing Systems - Volume 2}, 2014.

\bibitem{raghu}
Maithra Raghu, Ben Poole, Jon Kleinberg, Surya Ganguli, and Jascha
  Sohl-Dickstein.
\newblock On the expressive power of deep neural networks.
\newblock In {\em Proceedings of the 34th International Conference on Machine
  Learning}, 2017.

\bibitem{phuong2019functional}
Mary Phuong and Christoph~H Lampert.
\newblock Functional vs. parametric equivalence of relu networks.
\newblock In {\em International Conference on Learning Representations}, 2019.

\bibitem{lecun-gradientbased-learning-applied-1998}
Yann LeCun, Léon Bottou, Yoshua Bengio, and Patrick Haffner.
\newblock Gradient-based learning applied to document recognition.
\newblock In {\em Proceedings of the IEEE}, volume~86, pages 2278--2324, 1998.

\bibitem{bernstein_learning_2020}
Jeremy Bernstein, Jiawei Zhao, Markus Meister, Ming-Yu Liu, Anima Anandkumar,
  and Yisong Yue.
\newblock Learning compositional functions via multiplicative weight updates.
\newblock {\em Advances in neural information processing systems}, 33, 2020.

\bibitem{balduzzi2018shattered}
David Balduzzi, Marcus Frean, Lennox Leary, JP~Lewis, Kurt Wan-Duo Ma, and
  Brian McWilliams.
\newblock The shattered gradients problem: If resnets are the answer, then what
  is the question?, 2018.

\end{thebibliography}
\bibliographystyle{unsrt}
\appendix

\section{Proof of \autoref{prop:bijective}}\label{app:proofbijective}
To lighten notations we omit the dependence of $W,V$ and $\mat{S}$ on $\theta$. The proof follows three steps.
  \paragraph{1) $\mat{S}$ is well-defined.} Consider $\alpha \in W$ and $\nu \in H$. First, since $\theta$ is admissible, there exists at least one path $p$ connecting $\nu$ to some output neuron $\eta$ through edges $e \in E \cap \supp{\theta}$. We wish to show that if $p$ and $p'$ are two such partial paths then
  \(
    \sum_{e \in p} \alpha_e = \sum_{e \in p'} \alpha_e.
  \)
  Since $\theta$ is admissible, there exists a partial path $\bar{p}$ going from some input neuron $\mu$ to $\nu$ through edges $e \in E \cap \supp{\theta}$. Since $\bar{p}_{\ell}=p_{\ell}=p'_{\ell}$, define by concatenation the full paths $q = (\bar{p}_{0},\ldots,\bar{p}_{\ell},p_{\ell+1},\ldots,p_{L})$ and $q' = (\bar{p}_{0},\ldots,\bar{p}_{\ell},p'_{\ell+1},\ldots,p'_{L})$. As $q$ and $q'$ have all their edges in $\supp{\theta}$, we have $\srmap_{q}(\theta) \neq 0 \neq \srmap_{q'}(\theta)$. Since $\alpha \in W$ and $q,q' \in \pathset{P}_{0}$, we have $\srmap_{q}(\theta) \cdot (\mat{P}\alpha)_{q}=\srmap_{q'}(\theta) \cdot (\mat{P}\alpha)_{q'}=0$, hence  $(\mat{P}\alpha)_{q}=(\mat{P}\alpha)_{q'}=0$ and
  \[ 
  \sum_{e\in \overline p}\alpha_e + \sum_{e\in p}\alpha_e = \sum_{e\in q}\alpha_e = (\mat{P} \alpha)_q = 0
  =  (\mat{P} \alpha)_{q'} = \sum_{e\in \overline p}\alpha_e + \sum_{e\in p'}\alpha_e.
  \]
  Thus,
  $
    \sum_{e \in q} \alpha_e = \sum_{e \in q'} \alpha_e
  $ and $\mat{S}$ is well-defined.
  \paragraph{2) $\mat{S}$ is injective on $V$.}
    Note that $\mat{S}$ is linear hence it is sufficient to show that its kernel is reduced to zero.
  Let $\alpha \in V$ such that $\mat{S}\alpha = 0$. Since $\alpha_{\bar{H}}=0$ and $\supp{\alpha} \subseteq \supp{\theta}$, we have $\alpha_{e} = 0$ for each $e \in E \backslash \supp{\theta}$, hence it is sufficient to show $\alpha_{e}=0$ for any edge $e = \mu\to \nu \in E \cap \supp{\theta}$.
  Since $\theta$ is admissible, there is a partial path $\overline p$ going from $\nu$ to some output neuron $\eta$ through edges $e' \in E \cap \supp{\theta}$. Since $e \in E \cap \supp{\theta}$, the extended path $p := \mu \to \overline p$ also has all its edges in $E \cap \supp{\theta}$ and joins $\mu$ to an output neuron.  We distinguish three cases: in the first case, $\mu,\nu$ are two hidden neurons, and
   \[
   \alpha_e= \sum_{e' \in p} \alpha_{e'}- \sum_{e' \in \overline p} \alpha_{e'} 
   =
   - (\mat{S}\alpha)_\mu + (\mat{S}\alpha)_\nu =  0.
  \] 
In the second case, $\mu \in N_{0}$ is an input neuron, hence $p$ is a full path with edges $e' \in E \cap \supp{\theta}$, so that $\srmap_{p}(\theta) \neq 0$. Since $\alpha \in W$ it follows that $\sum_{e' \in p} \alpha_{e'} = (\mat{P}\alpha)_{p}= 0$ and we also obtain $\alpha_{e} = 0$. Finally, in the third case, $\nu \in N_{L}$ is an output neuron hence $\bar{p} = (\nu)$ contains no edge, so that $\sum_{e' \in \bar{p}} \alpha_{e'} = 0$ and we get $\alpha_{e}=0$ as well.
 
  \paragraph{c) $\mat{S}$ is surjective.} Consider $\beta \in \RR^H$, and $\alpha \in \RR^{\paramset}$ as defined around~\eqref{eq:DefPsiInverse}. Consider $p = (p_\ell, \dots, p_L) \in \pathset{P}$ (with $0 \leq \ell \leq L-1$) a full or partial path going from some neuron $\mu = p_{\ell}  \in N_{0} \cup H$  to an arbitrary output neuron $\eta = p_L \in N_{L}$ through edges $e \in E \cap \supp{\theta}$. 
In the case of a full path, $\mu \in N_{0}$ and
  \begin{align*}
   \sum_{e\in p}\alpha_e
    &= -\beta_{p_{L-1}} + \sum_{j=1}^{L-1} (\beta_{p_j} - \beta_{p_{j-1}}) + \beta_{p_0} = 0.
  \end{align*}
  As this holds for any full path with edges $e \in \supp{\theta}$, and since $\alpha_{\bar{H}}=0$, we get $\alpha \in W$. Since $\alpha_{e}=0$ for $e \in E \backslash \supp{\theta}$ we have indeed $\supp{\alpha} \subseteq \supp{\theta}$ hence $\alpha \in V$.
  
    In the case of a partial path ($\ell \geq 1$), we have $\mu \in H$ and similarly
  \begin{align*}
   \sum_{e\in p}\alpha_e
    &= -\beta_{p_{L-1}} + \sum_{j=\ell+1}^{L-1} (\beta_{p_j} - \beta_{p_{j-1}}) = -\beta_{p_{\ell}} = -\beta_{\mu}
  \end{align*}
  hence $(\mat{S}\alpha)_{\mu} = \beta_{\mu}$. As this holds for any $\mu \in H$, this shows that $(\mat{S}\alpha) = \beta$.

\section{Proof of \autoref{th:IdentImpliesLocaIdent}}
\label{sec:global2local}
Denote $\eta = \min_{i \in \supp{\theta}}|\theta_{i}|$. Assume by contradiction that $\theta$ is not locally S-identifiable from $\xset$ with respect to $\Theta$. This implies that for each $n \geq 1$ there is $\theta_{n} \in \Theta \cap B(\theta,\min(\eta,1/n))$ which is not scaling-equivalent to $\theta$ such that $\realiz{\theta_{n}}=\realiz{\theta}$  on $\xset$.
For $n \geq 1/\epsilon$, since $\theta$ is PS-identifiable from $\xset$ with respect to $\Theta$ and since $\theta_{n} \in \Theta$ satisfies $\realiz{\theta_{n}}=\realiz{\theta}$ on $\xset$, we have $\theta_{n} \sim_{PS} \theta$, hence there is a permutation $\pi_{n} \in \permset{G}$ such that $\pi_{n} \circ \theta_{n} \sim_{S} \theta$, hence by \autoref{lem:EqualRepAndSignImpliesSEquiv}
\begin{align}
\sign{\pi_{n} \circ \theta_{n}} = \sign{\theta},\qquad \text{and}\qquad
\rmap(\pi_{n} \circ \theta_{n}) &= \rmap(\theta) \label{eq:tmp2}.
\end{align}
Since the set of permutations is finite, there exists $\pi \in \permset{G}$, 
and an increasing subsequence $n_{k}$ such that $\pi_{n_{k}} = \pi$ 
for each $k \geq 1$. By~\eqref{eq:tmp2}, for every $k$ we have 
\begin{align}
\sign{\pi \circ \theta_{n_{k}}} = \sign{\theta}\qquad \text{and} \qquad
\rmap(\pi \circ \theta_{n_{k}}) = \rmap(\theta).\label{eq:tmp3}
\end{align}
There is a permutation matrix $\permmat \in \RR^{\pathset{P} \times \pathset{P}}$ such that $\rmap(\pi \circ \theta') = \permmat \rmap(\theta')$ for all $\theta' \in \RR^{\paramset}$.
%
Since $\lim_{k \to \infty}\theta_{n_{k}}=\theta$, by continuity of $\rmap$ we obtain 
$\permmat \rmap(\theta) = \rmap(\pi \circ \theta) = \rmap(\theta)$ hence for every $k \geq 1$
\begin{align}
\label{eq:tmp1}
\rmap(\theta) = \boldsymbol{\Pi}^{-1}\rmap(\theta) \stackrel{\eqref{eq:tmp3}}{=}  \boldsymbol{\Pi}^{-1} \boldsymbol{\Pi}\rmap(\theta'_{n_{k}}) = \rmap(\theta_{n_{k}}).
\end{align}
Since $\theta$ is admissible, by \autoref{cor:embedimpliessupport}, the equality $\rmap(\theta_{n_{k}}) = \rmap(\theta)$ implies $\supp{\theta_{n_{k}}} = \supp{\theta}$. This implies that $\sign{(\theta_{n_{k}})_{i}} = 0 = \sign{\theta_{i}}$ for each $i \notin \supp{\theta}$. Since $\theta_{n_{k}} \in B(\theta,\eta)$ for each $k$, we also have $\sign{(\theta'_{n_{k}})_{i}} = \sign{\theta_{i}} \in \{-1,1\}$ for every $i \in \supp{\theta}$, hence $\sign{\theta_{n_{k}}} = \sign{\theta}$ for every $k$.  Since $\theta$ is admissible, by \autoref{lem:EqualRepAndSignImpliesSEquiv}, the fact that $\sign{\theta_{n_{k}}} = \sign{\theta}$ and $\rmap(\theta_{n_{k}}) = \rmap(\theta)$ implies $\theta'_{n} \sim_{S} \theta$.
This contradicts our assumption that $\theta'_{n} \not\sim_{S} \theta$ for every $n$.

\section{Proof of \autoref{lem:PSIFromBoundedImpliesNoTwin}}\label{app:PSIFromBoundedImpliesNoTwin}
We will prove the contraposition using the following observation.
\begin{fact}\label{fact:TwinProof}
Consider $\alpha,\beta \in \RR$ and $M>0$. For any $t \in [-M,\infty)$ we have
\[
\alpha \ReLU(t)+\beta \ReLU(-t) 
= \begin{cases}
\alpha t, & t \geq 0\\
-\beta t, & t \leq 0
\end{cases}
\quad =
(\alpha+\beta) \ReLU(t)-\beta \ReLU(t+M)+M.
\]
\end{fact}
Assuming that $\theta$ has twins, consider a hidden layer $1 \leq \ell \leq L-1$ and $T \subseteq N_{\ell}$ a pair of twin neurons $T = \{\nu_{1},\nu_{2}\}$. Denote $\vw_{i} = \vw_{\bullet\to \nu_{i}}$, $b_{i} = b_{\nu_{i}}$, $\vv_{i} =\vw_{\nu_{i}\to \bullet}$. As these neurons are twins, there is $\lambda \in \RR$ such that for every $x \in \RR^{N_{0}}$,
\[
z_{\nu_{2}}(\theta,x) = \langle \vw_{2},y_{\ell-1}(\theta,x)\rangle+b_{2} = \lambda\left( \langle \vw_{1},y_{\ell-1}(\theta,x)\rangle+b_{1} \right) = \lambda z_{\nu_{1}}(\theta,x)
\]
In the case of positive twins we have $\lambda>0$, hence $y_{\nu_{1}}(\theta,x) = \lambda y_{\nu_{2}}(\theta,x)$ for every $x$. Given $\epsilon>0$ consider $\theta'(\epsilon)$ obtained by keeping unchanged all weights and biases in $\theta$ except the weights outgoing from neurons $\nu_{i}$, $i=1,2$:  $\vv'_{1} = \vv_{1}+\lambda\epsilon \indic{N_{\ell+1}}$, $\vv'_{2} = \vv_{2}-\epsilon\indic{N_{\ell+1}}$.
Since the linear layers and biases of hidden neurons up to layer $\ell$ are unchanged, we have  $\vec{y}_{\ell}(\theta,x) = \vec{y}_{\ell}(\theta',x)$ for all $x$. For every neuron $\nu \in N_{\ell} \backslash T$, since the outgoing weights are unchanged, we get for every $x$
\[
 y_{\nu}(\theta',x)\vw'_{\nu\to\bullet} = y_{\nu}(\theta,x) \vw_{\nu\to\bullet}.
\]
Moreover, since $y_{\nu_{2}}(\theta,x) = \lambda y_{\nu_{1}}(\theta,x)$  and $\vec{y}_{\ell}(\theta,x) = \vec{y}_{\ell}(\theta',x)$, we obtain for every $x$
\begin{align*}
 y_{\nu_{1}}(\theta',x)\vw'_{\nu_{1}\to\bullet}+ y_{\nu_{2}}(\theta',x)\vw'_{\nu_{2}\to\bullet}
&=
 y_{\nu_{1}}(\theta,x)\vv'_{1}+ y_{\nu_{2}}(\theta,x)\vv'_{2}\\
&=
 y_{\nu_{1}}(\theta,x)(\vv_{1}+\lambda\epsilon \indic{N_{\ell+1}})+ \lambda y_{\nu_{1}}(\theta,x)(\vv_{2}-\epsilon \indic{N_{\ell+1}})\\
&=
y_{\nu_{1}}(\theta,x)\vv_{1}+ \lambda y_{\nu_{1}}(\theta,x)\vv_{2}\\
&= 
 y_{\nu_{1}}(\theta,x)\vw_{\nu_{1}\to\bullet}+ y_{\nu_{2}}(\theta,x)\vw_{\nu_{2}\to\bullet}.
\end{align*}
 Summing over all hidden neurons we obtain $\vec{z}_{\ell+1}(\theta,x) = \vec{z}_{\ell+1}(\theta',x)$ for every $x$, and since all the next affine layers are unchanged, we obtain  $\realiz{\theta'}=\realiz{\theta}$. It is not difficult to check that $\theta'=\theta'(\epsilon)$ is not scaling equivalent to $\theta$ and can be made arbitrarily close to it. This shows that $\theta$ is not locally S-identifiable from $\RR^{N_{0}}$.
 By contraposition, if $\theta$ is locally S-identifiable from $\RR^{N_{0}}$ then it has no positive twins.

In the case of negative twins we have $y_{\nu_{1}}(\theta,x) = \ReLU(t)$ and $y_{\nu_{2}}(\theta,x) = \abs{\lambda} \ReLU(-t)$ with $t = t(x) := z_{\nu_{1}}(\theta,x) = \langle \vw_{1},y_{\ell-1}(\theta,x)\rangle+b_{1}$. Since $\xset$ is bounded there is some finite $M>0$ such that $\abs{z_{\nu_{1}}(\theta,x)} \leq M$ for every $x \in \xset$. Consider $\theta'$ obtained by keeping all weights and biases unchanged from $\theta$ except the incoming and outgoing weights of $\nu_{1},\nu_{2}$, their biases, and the biases of the neurons of the next layer, $\eta \in N_{\ell+1}$, which are set as:
\begin{itemize}
\item $\vw'_{\bullet \to \nu_{1}} = \vw'_{\bullet \to \nu_{2}} = \vw_{\bullet \to \nu_{1}}$;
\item $b'_{\nu_{1}} = b_{\nu_{1}}$; $b'_{\nu_{2}} = b_{\nu_{1}}+M$;
\item $\vw'_{\nu_{1}\to\bullet} = \vw_{\nu_{1}\to\bullet} + \abs{\lambda}\vw_{\nu_{2}\to\bullet}$;
$\vw'_{\nu_{2}\to\bullet} = -\abs{\lambda} \vw_{\nu_{2}\to\bullet}$;
\item $b'_{\eta} = b_{\eta}+M$
\end{itemize}
For each $\eta \in N_{\ell+1}$, using \autoref{fact:TwinProof} for $\alpha = w_{\nu_{1}\to\eta}$, $\beta = \abs{\lambda}w_{\nu_{2}\to\eta}$ we obtain
\begin{align*}
 y_{\nu_{1}}(\theta,x)w_{\nu_{1}\to\eta}+ y_{\nu_{2}}(\theta,x)w_{\nu_{2}\to\eta}
&=
 \alpha \ReLU(t)+\beta \ReLU(-t)\\
 &=
 (\alpha+\beta) \ReLU(t) -\beta \ReLU(t+M)+M\\
&=
 (w_{\nu_{1}\to\eta}+\abs{\lambda}w_{\nu_{2}\to\eta}) \ReLU(t) -\abs{\lambda}w_{\nu_{2}\to\eta} \ReLU(t+M)+M\\
&=
 w'_{\nu_{1}\to\eta} \ReLU(t)+w'_{\nu_{2}\to\eta} \ReLU(t+M)+M\\
 &=
 w'_{\nu_{1}\to\eta} \ReLU(z_{\nu_{1}}(\theta',x))+w'_{\nu_{2}\to\eta} \ReLU(z_{\nu_{2}}(\theta',x))+M\\
 &=
 w'_{\nu_{1}\to\eta} y_{\nu_{1}}(\theta',x)+w'_{\nu_{2}\to\eta} y_{\nu_{2}}(\theta',x)+M.
 \end{align*}
Reasoning as in the case of positive twins we obtain $\vz_{\ell+1}(\theta',x) = \vz_{\ell+1}(\theta,x)$ and eventually $\realiz{\theta'}(x) = \realiz{\theta}(x)$ for every $x$ in the bounded set $\xset$. Since the sign of $\vw_{\nu_{2}\to\bullet}$ has changed, $\theta'$ is not PS-equivalent to $\theta$. This shows that $\theta$ is not PS-identifiable from $\xset$ with respect to $\Theta = \RR^{\paramset}$. 

By contraposition, assuming that $\theta$ is PS-identifiable from a bounded set $\xset$ with respect to $\Theta = \RR^{\paramset}$, there is no negative twin. Besides, by~\autoref{th:IdentImpliesLocaIdent}, such a $\theta$ is also locally S-identifiable from  $\xset$ with respect to $\Theta = \RR^{\paramset}$, hence it is locally S-identifiable from  $\RR^{N_{0}}$ with respect to $\Theta = \RR^{\paramset}$. By the first part of the lemma, we conclude that $\theta$ has no positive twins either. Hence, it has no twins.

\section{Proof of \autoref{lem:nonlocaldegeneracy}}\label{app:irreduciblenecessary}
We will use the following observation.
\begin{fact}\label{fact:IrrProof}
 for $\chi \in \{0,1\}$ and $e = (-1)^{\chi}$ we have
$\ReLU(t) = \chi t+e\ReLU(et)$ for every $t \in \RR$.
\end{fact}
Assume for the sake of contradiction that $\theta$ is not irreducible:  $\mat{W}_{\ell+1}\mat{I}_{T}\mat{W}_{\ell} = 0$ for some non-empty $T \subset N_{\ell}$ with some $1 \leq \ell \leq L-1$. Denote $\theta'$ a network with the same weights and biases as $\theta$ except on layers $\ell$ and $\ell+1$, where $\mat{W}'_{\ell},\mat{W}'_{\ell+1}$ and $\vb'_{\ell},\vb'_{\ell+1}$ will soon be described. By an easy induction we have $\vy_{\ell'}(\theta',x)=\vy_{\ell'}(\theta,x)$ for $0 \leq \ell' \leq \ell-1$.

Defining $\mat{J}_{T} = \mathtt{diag}(e_{\nu})_{\nu \in N_{\ell}}$ with $e_{\nu} = -1$ if $\nu \in T$ and $e_{\nu}=1$ otherwise, we obtain from \autoref{fact:IrrProof} that for every vector $\vz_{\ell} \in \RR^{N_{\ell}}$, $\ReLU(\vz_{\ell}) = \mat{I}_{T} \vz_{\ell}+\mat{J}_{T}\  \ReLU(\mat{J}_{T}\vz_{\ell})$. Define $\mat{W}'_{\ell} = \mat{J}_{T}\mat{W}_{\ell}$, $\mat{W}'_{\ell+1} = \mat{W}_{\ell+1}\mat{J}_{T}$, $\vb'_{\ell} = \mat{J}_{T}\vb_{\ell}$. For each $x \in \RR^{N_{0}}$, since $\vy_{\ell-1}(\theta,x) = \vy_{\ell-1}(\theta',x)$ and $\mat{J}_{T}^{2} = \mat{Id}_{\RR^{N_{\ell}}}$ we get using the shorthands $\vz_{i} = \vz_{i}(\theta,x)$, $\vz'_{i} = \vz_{i}(\theta',x)$, $i \in \{\ell,\ell+1\}$
\begin{align*}
\vz_{\ell} = & \mat{W}_{\ell}\vy_{\ell-1}(\theta,x)+\vb_{\ell} = \mat{J}_{T}\left(\mat{J}_{T}\mat{W}_{\ell}\vy_{\ell-1}(\theta,x)+\mat{J}_{T}\vb_{\ell}\right)
= \mat{J}_{T}\vz'_{\ell}\\
\vz_{\ell+1} = & \mat{W}_{\ell+1}\ \ReLU(\vz_{\ell})+\vb_{\ell+1}
 =  \mat{W}_{\ell+1}\left(\mat{I}_{T}\vz_{\ell}+\mat{J}_{T}\ \ReLU(\mat{J}_{T}\vz_{\ell})\right) + \vb_{\ell+1}\\
 = & \underbrace{\mat{W}_{\ell+1}\mat{I}_{T}\mat{W}_{\ell}}_{=0} \vy_{\ell-1}(\theta,x) 
 + \mat{W}_{\ell+1}\mat{I}_{T}\vb_{\ell} + \mat{W}'_{\ell+1}\ \ReLU(\vz'_{\ell}) + \vb_{\ell+1}\\
= & \mat{W}'_{\ell+1}\ \ReLU(\vz'_{\ell}) + (\mat{W}_{\ell+1}\mat{I}_{T}\vb_{\ell} + \vb_{\ell+1}).
\end{align*}
Defining $\vb'_{\ell+1} := \mat{W}_{\ell+1}\mat{I}_{T}\vb_{\ell}+\vb_{\ell+1}$, we get $\vz_{\ell+1}(\theta,x) = \vz'_{\ell+1}(\theta',x)$ for all $x$. Since all other layers of $\theta$ and $\theta'$ are identical, an easy induction yields $\realiz{\theta}= \realiz{\theta'}$, where $\theta' \in \RR^{\paramset} = \Theta$. To conclude, we prove below that $\theta'$ is not PS-equivalent to $\theta$: this contradicts the assumption that $\theta$ is PS-identifiable and concludes the proof.

For the sake of (yet another) contradiction, assume that $\theta' \sim_{PS}$, so that there exists diagonal matrices $\mat{\Lambda}_{\ell'} \in \RR^{N_{\ell'} \times N_{\ell'}}$ with positive entries and permutation matrices $\mat{\Pi}_{\ell'} 
 \in \RR^{N_{\ell'} \times N_{\ell'}}$, $0 \leq \ell' \leq L$, such that $\mat{\Lambda}_{0} = \mat{\Pi}_{0} = \mat{I}_{N_{0}}$, $\mat{\Lambda}_{L} = \mat{\Pi}_{L} = \mat{I}_{N_{L}}$, $\mat{W}'_{\ell'} = \mat{\Pi}_{\ell'}\mat{\Lambda}_{\ell'} \mat{W}_{\ell'} \mat{\Lambda}^{-1}_{\ell'-1}\mat{\Pi}_{\ell'-1}^{-1}$, and $\vb'_{\ell'} = \mat{\Pi}_{\ell'}\mat{\Lambda}_{\ell'}\vb_{\ell'}$ for every $1 \leq \ell' \leq L$. We show by induction that $\mat{\Lambda}_{\ell'}=\mat{\Pi}_{\ell'} = \mat{I}_{N_{\ell'}}$ for every $0 \leq \ell' < \ell$. This trivially holds for $\ell'=0$. If it holds for some $\ell' < \ell-1$ then, as $(\mat{W}'_{\ell'+1},\vb'_{\ell'+1}) = (\mat{W}_{\ell'+1},\vb_{\ell'+1})$ by construction of $\theta'$, we have
\begin{align*}
(\mat{W}_{\ell'+1},\vb_{\ell'+1})
=
(\mat{W}'_{\ell'+1},\vb'_{\ell'+1})
&=
(\mat{\Pi}_{\ell'+1}\mat{\Lambda}_{\ell'+1}\mat{W}_{\ell'+1}\mat{\Lambda}_{\ell'}^{-1}\mat{\Pi}_{\ell'}^{-1},\mat{\Pi}_{\ell'+1}\mat{\Lambda}_{\ell'+1}\vb_{\ell'+1})\\
& =
\mat{\Pi}_{\ell'+1}\mat{\Lambda}_{\ell'+1}
(\mat{W}_{\ell'+1},\vb_{\ell'+1}),
\end{align*}
i.e., $(\vw_{\bullet \to \nu},b_{\nu}) = \lambda_{\pi(\nu)} (\vw_{\bullet \to \pi(\nu)},b_{\pi(\nu)})$ for every $\nu \in N_{\ell'+1}$, with $\pi$ the permutation of $N_{\ell'+1}$ associated to $\mat{\Pi}_{\ell'+1}$ and $\mat{\Lambda}_{\ell'+1} = \mathtt{diag}(\lambda_{\nu})_{\nu \in N_{\ell'+1}}$. Since $\theta$ has no twin, it follows that $\pi$ is the identity and $\lambda_{\nu}=1$ for every $\nu \in N_{\ell'+1}$, which concludes the induction. Now, since $(\mat{W}'_{\ell},\vb'_{\ell}) = \mat{J}_{T}(\mat{W}_{\ell},\vb_{\ell})$ by construction of $\theta'$, we have
\begin{align*}
\mat{J}_{T}(\mat{W}_{\ell},\vb_{\ell})
=
(\mat{W}'_{\ell},\vb'_{\ell})
=
(\mat{\Pi}_{\ell}\mat{\Lambda}_{\ell}\mat{W}_{\ell}\mat{\Lambda}_{\ell-1}^{-1}\mat{\Pi}_{\ell-1}^{-1},\mat{\Pi}_{\ell}\mat{\Lambda}_{\ell}\vb_{\ell})
 =
\mat{\Pi}_{\ell}\mat{\Lambda}_{\ell}
(\mat{W}_{\ell},\vb_{\ell}).
\end{align*}
As a result, for each $\nu \in T \neq \emptyset$ we have  $-(\vw_{\bullet \to \nu},b_{\nu}) = \lambda_{\pi(\nu)}(\vw_{\bullet \to \pi(\nu)},b_{\pi(\nu)})$ where $\pi$ is the permutation of $N_{\ell}$ associated to $\mat{\Pi}_{\ell}$ and $\mat{\Lambda}_{\ell} = \mathtt{diag}(\lambda_{\nu})_{\nu \in N_{\ell}}$. However, since $\theta$ has no twin, $(\vw_{\bullet \to \nu},b_{\nu})$ is not collinear to any $(\vw_{\bullet \to \nu'},b_{\nu'})$, $\nu' \in N_{\ell}$, $\nu' \neq \nu$, hence $\pi(\nu) = \nu$. It follows the 
$-(\vw_{\bullet \to \nu},b_{\nu})=\lambda_{\nu}(\vw_{\bullet \to \nu},b_{\nu})$, and as $\lambda_{\nu}>0$ we obtain $(\vw_{\bullet \to \nu},b_{\nu})=0$, therefore $\theta$ is not admissible. However, by~\autoref{lem:admnec}, since $\theta$ is PS-identifiable with respect to $\Theta = \RR^{\paramset}$, it is admissible. Hence the desired contradiction.


\section{Proof of \autoref{lem:RealizationAlgebraicBis} and \autoref{le:Rexpressionbis}}\label{app:ProofRealizationAlgebraic}

\begin{proof}[Proof of \autoref{lem:RealizationAlgebraicBis}]
The proof is by induction on $L$. 
For $L=1$, since $\mat{I}_{0}$ is the identity
\begin{align*}
\realiz{\theta}(x) = \vz_{1}(\theta,x) = \mat{W}_{1} x + \vb_{1} = \mat{W}_{1}\mat{I}_{0} x +\vb_{1}.
\end{align*}
With the convention that a product of matrices over an empty index set is the identity, this establishes~\eqref{eq:RealizationAlgebraicBis} for $L=1$.
%
%
Now, assuming that~\eqref{eq:RealizationAlgebraicBis} holds for every network of depth $L$, let us prove it for $\theta$ of depth $L+1$. For this, observe that $\vz_{L}(\theta,x)$ is the realization of a network $\underline{\theta}$ of depth $L$ made of the first $L$ affine layers of $\theta$, hence by the induction hypothesis we can use~\eqref{eq:RealizationAlgebraicBis} to get 
\[
\vz_{L}(\theta,x) = \realiz{\underline{\theta}}(x) = 
\left(\Pi_{\ell=1}^{L}\mat{W}_{\ell}\mat{I}_{\ell-1}\right) x 
+
\sum_{\ell'=1}^{L-1} \left(\Pi_{\ell=\ell'+1}^{L} 
\mat{W}_{\ell}\mat{I}_{\ell}\right) \vb_{\ell'}
\]
Since $\vy_{L}(\theta,x) = \va_{L}(\theta,x) \odot \vz_{L}(\theta,x) =  \mat{I}_{L}\vz_{L}(\theta,x)$ we get
\begin{align*}
\realiz{\theta}(x) &= \vz_{L+1}(\theta,x) = \mat{W}_{L+1} \vy_{L}(\theta,x) + \vb_{L+1}\\
& = 
\mat{W}_{L+1}\mat{I}_{L}
\left(
\left(\Pi_{\ell=1}^{L} \mat{W}_{\ell}\mat{I}_{\ell-1}\right)x
 + \sum_{\ell'=1}^{L} 
 \left(\Pi_{\ell=\ell'+1}^{L} \mat{W}_{\ell}\mat{I}_{\ell-1}\right) \vb_{\ell'}
 \right) +  \vb_{L+1}.
\end{align*}
To conclude simply observe that
\begin{align*}
\mat{W}_{L+1}\mat{I}_{L}\left(\Pi_{\ell=1}^{L} \mat{W}_{\ell}\mat{I}_{\ell-1}\right) 
&= \left(\Pi_{\ell=1}^{L+1} \mat{W}_{\ell}\mat{I}_{\ell-1}\right),\\
\mat{W}_{L+1}\mat{I}_{L}
\left(
 \sum_{\ell'=1}^{L} 
 \left(\Pi_{\ell=\ell'+1}^{L} \mat{W}_{\ell}\mat{I}_{\ell-1}\right) \vb_{\ell'}
 \right)
&=
 \sum_{\ell'=1}^{L} 
 \left(\Pi_{\ell=\ell'+1}^{L+1} \mat{W}_{\ell}\mat{I}_{\ell-1}\right) \vb_{\ell'}
 \\
\text{and that with}\ \ell' = L+1,\quad \vb_{L+1} &= \vb_{\ell'} = \left(\Pi_{\ell=\ell'+1}^{L+1} \mat{W}_{\ell}\mat{I}_{\ell-1}\right)\vb_{\ell'}.\qedhere
\end{align*}
\end{proof}

\begin{proof}[Proof of \autoref{le:Rexpressionbis}]
With $\pathset{P}_{H} := \cup_{\ell=1}^{L-1} \pathset{P}_{\ell}$ the set of all paths from a hidden neuron to an output neuron, for each output neuron $\eta \in N_{L}$ we prove at the end of this section that
\begin{align}\label{eq:RexpressionOld}
\realiz{\theta}(x)_{\eta} 
&= \sum_{p \in \pathset{P}_{0}, p_{L}=\eta}\actpath_{p}(\theta,x) \srmap_{p}(\theta)  x_{p_{0}}
+ \sum_{p \in \pathset{P}_{H}, p_{L}=\eta}  \actpath_{p}(\theta,x)  \srmap_{p}(\theta)
+\theta_{\eta}
\end{align}
Any $p \in \pathset{P}_{0}$ is uniquely written $p = \mu \to q \to \eta$ with $\mu = p_{0}$ its input neuron, $\eta \in N_{L}$ its output neuron, and $q = (p_{1},\ldots,p_{L-1}) \in \pathset{Q}_{1}$ a partial path from the first layer to the penultimate layer, and $\actpath_{\mu \to q \to \eta}(\theta,x) = \actpath_{q}(\theta,x)$ for every $q \in \pathset{Q}_{1}$ and any $\mu \in N_{0}$,$\eta \in N_{L}$.\\
Similarly, for $1 \leq \ell \leq L-1$, every  partial path $p \in \pathset{P}_{\ell}$ starting from the $\ell$-th hidden layer and ending at the output layer can be written as $p = q\to\eta$ where $\eta = p_{L} \in N_{L}$ and $q \in \pathset{Q}_{\ell}$ starts from the $\ell$-th hidden layer and ends at the penultimate layer, and we have $\actpath_{q\to\eta}(\theta,x) = \actpath_{q}(\theta,x)$ for all $\theta,x$. Therefore,~\eqref{eq:RexpressionOld} can be rewritten as
\begin{align*}
\realiz{\theta}(x)_{\eta}
&=
 \sum_{q \in \pathset{Q}_{1}} \actpath_{q}(\theta,x) \sum_{\mu \in N_{0}}\srmap_{\mu\to q\to\eta}(\theta)x_{\mu}  
+ \sum_{q \in \pathset{Q}}  \actpath_{q}(\theta,x) \srmap_{q\to\eta}(\theta)
+\theta_{\eta}\\
&=
\sum_{q \in \pathset{Q}_{1}} \actpath_{q}(\theta,x) [\rmapi_{\eta}(\theta)x]_{q}
+ \sum_{q \in \pathset{Q}+1}  [\bvactpath(\theta,x)]_{q} [\rmaph_{\eta}(\theta)]_{q}\\
&= 
\langle \mat{Q}\bvactpath(\theta,x),\rmapi_{\eta}(\theta)x\rangle + \langle \bvactpath(\theta,x),\rmaph_{\eta}(\theta)\rangle.
\end{align*}
where we used that $\bvactpath(\theta,x) := (\vactpath(\theta,x),1)$ with $\vactpath(\theta,x) := (\actpath_{q}(\theta,x))_{q \in \pathset{Q}}$, and $\mat{Q}$ is the canonical restriction from $\pathset{Q}+1$ to $\pathset{Q}_{1}$.
\end{proof}

\begin{proof}[Proof of Equation~\eqref{eq:RexpressionOld}]
  We prove the result by induction on the number of layers $L$.

For $L = 1$ we have $\pathset{P}_{0} = \{(\mu, \eta)\}_{\mu \in N_0, \eta \in N_1}$ and $\pathset{P}_1 =  \{(\eta)\}_{\eta \in N_1}$. Since $H = \emptyset$, $\actpath_p(\theta, x) = 1$ for all $p \in \pathset{P}$ and $\pathset{P}_{H} = \emptyset$. We have $\srmap_p(\theta) = w_{\mu\to\eta}$ for all $p = (\mu, \eta) \in \pathset{P}_{0}$ and $\srmap_{p}(\theta) = b_{\eta}$ for each $p  = (\eta) \in \pathset{P}_{1}$. It follows that
    \begin{align*}
      \sum_{\substack{p \in \pathset{P}_{0} \\ p_L =\eta}} \actpath_p(\theta, x) \srmap_p(\theta)  x_{p_0}
      +
      \sum_{\substack{p\in \pathset{P}_{H} \\ p_{L} = \eta}}  \actpath_p(\theta, x)\srmap_p(\theta) + \theta_{\eta}
      &= \sum_{\substack{\mu \in N_0}} w_{\mu\to\eta}x_{\mu} + b_{\eta}
      = \left(\mat{W}_{1}x+\vb_{1}\right)_\eta = (\realiz{\theta}(x))_{\eta}.
    \end{align*}
    This establishes~\eqref{eq:RexpressionOld} for $L=1$. Assume now that~\eqref{eq:RexpressionOld} holds for networks of depth $L \geq 1$. With $\theta$ a network of depth $L+1$, observe that $\vz_{L}(\theta,x) = \realiz{\tilde{\theta}}(x)$ with $\tilde{\theta}$ the network made of the first $L$ affine layers of $\theta$. Using the induction hypothesis, we get, for $\nu \in N_{L-1}$,
    \begin{align}\label{rec:L-1_b}
      z_\nu(\theta,x)
      &=
      \sum_{\substack{\tilde{p} \in \widetilde{\pathset{P}}_{0} \\ \tilde{p}_{L-1} =\nu}} \actpath_{\tilde{p}}(\tilde{\theta}, x) \srmap_{\tilde{p}}(\tilde{\theta})  x_{\tilde{p}_0}
      +
	\sum_{\substack{\tilde{p}\in \widetilde{\pathset{P}}_{H} \\ \tilde{p}_{L-1} = \nu}} \actpath_{\tilde{p}}(\tilde{\theta}, x) \srmap_{\tilde{p}}(\tilde{\theta}) 
    \end{align}
    with $\widetilde{\pathset{P}}_{0}= \{(p_0, \dots, p_{L-1}) \mid p\in \pathset{P}_{0}\}$, $\widetilde{\pathset{P}_{H}} = \{(p_\ell, \dots, p_{L-1})\mid p=(p_{\ell},\ldots,p_{L}) \in  \pathset{P}_{H}\}$. Since $\ReLU\left(z_{\nu}(\theta,x)\right) = \actneuron_{\nu}(\theta,x) z_{\nu}(\theta,x)$ we get
        \begin{align*}
      (\realiz{\theta}(x))_\eta
      &= \sum_{\nu \in N_{L-1}}y_\nu(\theta, x)w_{\nu\to\eta} + b_\eta
      = \sum_{\nu \in N_{L-1}}\ReLU\left(z_{\nu}(\theta,x)\right)w_{\nu\to\eta} + b_\eta\label{rec:L_b}\\
      &= \sum_{\nu \in N_{L-1}} \actneuron_{\nu}(\theta,x) z_{\nu}(\theta,x) w_{\nu\to\eta} + b_\eta\\
      &= 
      \sum_{\nu \in N_{L-1}}
      \sum_{\substack{\tilde p \in \widetilde{\pathset{P}}_{0} \\ \tilde p_{L-1} =\nu}}
  \actneuron_{\nu}(\theta,x) 
      \actpath_{\tilde p}(\tilde{\theta}, x)     
      \srmap_{\tilde p}(\tilde{\theta}) 
      w_{\nu\to\eta}
      x_{\tilde p_0}  
      \nonumber \\
      &+ \sum_{\nu \in N_{L-1}}
      \sum_{\substack{\tilde p\in \widetilde{\pathset{P}}_H \\ \tilde p_{L-1} = \nu}} 
       \actneuron_{\nu}(\theta, x) \actpath_{\tilde p}(\tilde{\theta}, x)
       \srmap_{\tilde p}(\tilde{\theta})w_{\nu\to\eta}\nonumber \\
      &+ b_\eta
    \end{align*}
    For each path such that $\tilde{p}_{L-1}=\nu \in N_{L-1}$ we have $ \actneuron_{\nu}(\theta, x) \actpath_{\tilde p}(\tilde{\theta}, x)      \srmap_{\tilde p}(\tilde{\theta})w_{\nu\to\eta} =  \actpath_{\tilde{p}\to\eta}(\theta, x)
       \srmap_{\tilde{p}\to\eta}(\theta)$, and $p := \tilde{p}\to \eta$ belongs to $\pathset{P}_{0}$ (resp. to $\pathset{P}_{H}$) if, and only if, $\tilde{p} \in \tilde{\pathset{P}}_{0}$ (resp. $\tilde{p} \in \tilde{\pathset{P}}_{H}$). Thus,
           \begin{align*}
      (\realiz{\theta}(x))_\eta
  &=
     \sum_{\nu \in N_{L-1}}
      \sum_{\substack{\tilde p \in \widetilde{\pathset{P}}_{0} \\ \tilde p_{L-1} =\nu}}
      \actpath_{\tilde p \to \eta}(\theta, x)  \srmap_{\tilde p \to \eta}(\theta) 
      x_{\tilde p_0}  
      + \sum_{\nu \in N_{L-1}}
      \sum_{\substack{\tilde p\in \widetilde{\pathset{P}}_H \\ \tilde p_{L-1} = \nu}} 
	\actpath_{\tilde p \to \eta}(\theta, x)
       \srmap_{\tilde p \to \eta}(\theta)+ b_\eta\\
       &=
      \sum_{\substack{p \in \pathset{P}_{0} \\ p_{L} =\eta}}
      \actpath_{p}(\theta, x)  \srmap_{p}(\theta) 
      x_{p_0}  
      + 
      \sum_{\substack{p\in \pathset{P}_H \\ p_{L} = \eta}} 
	\actpath_{p}(\theta, x)
       \srmap_{p}(\theta)+ b_\eta. \qedhere
  \end{align*}

\end{proof}

\section{Proof of \autoref{lem:characXset}}\label{app:ProofXset}
The result is proved by induction on the network's depth. The case $L=1$ is trivial with the convention that a union over an empty family is empty.
For any depth, since 
\[
\left(\cup_{\nu \in H} \Gamma_{\nu}(\theta)\right)^{c} = 
\cap_{\nu \in H} \Gamma_{\nu}^{c}(\theta)
= \cap_{\ell=1}^{L-1} \left(\cap_{\nu \in N_{\ell}} \Gamma_{\nu}^{c}(\theta)\right)
= \cap_{\ell=1}^{L-1} \left(\cup_{\nu \in N_{\ell}} \Gamma_{\nu}(\theta)\right)^{c}
\]
the result is equivalent to
$\xsetcont' = \cap_{\ell=1}^{L-1} \left(\cup_{\nu \in N_{\ell}} \Gamma_{\nu}(\theta)\right)^{c}$,
which is the quantity manipulated in the induction.
%
Assume that the result is valid for all parameters of depth $L$ and consider $\theta$ a parameter of depth $L +1 \geq 2$. Denoting $\underline{\theta} = g(\theta)$ its restriction to its first $L$ layers, we will show that $\xsetcont' = \xset_{\underline{\theta}}' \cap \left(\cup_{\nu \in N_{L}}\Gamma_{\nu}(\theta) \right)^{c}$.
First we prove $(\xset_{\underline{\theta}}')^{c} \cup  \left(\cup_{\nu \in N_{L}}\Gamma_{\nu}\right) \subset (\xsetcont')^{c}$.
\begin{itemize}
\item if $x \notin \xset_{\underline{\theta}}'$ then (by definition of $\xset_{\underline{\theta}}'$) the function $(\underline{\theta}',x') \mapsto \act(\underline{\theta}',x')$ is not locally constant around $(\underline{\theta},x)$ hence there exists $1 \leq \ell \leq L-1$ and $\nu \in N_{\ell}$ such that $\actneuron_{\nu}(\underline{\theta}',x')$ is not locally constant around $(\underline{\theta},x)$. Since $\ell \leq L-1$, for every $\theta',x'$ we have $ \actneuron_{\nu}(\theta',x') = \actneuron_{\nu}(\underline{\theta'},x')$ with $\underline{\theta}' = g(\theta')$ the restriction of $\theta'$ to its first $L$ layers. We obtain that $\actneuron_{\nu}(\theta',x')$ is not locally constant around $(\theta,x)$, showing that $x \notin \xsetcont'$.
\item  If $x \in \cup_{\nu \in N_{L}} \Gamma_{\nu}(\theta)$, there exists $\nu \in N_{L}$ such that $x \in \Gamma_{\nu}(\theta)$ hence $z_{\nu}(\theta,x) = 0$, the gradient is well-defined, and $\nabla z_{\nu}(\theta,x) \neq 0$. This implies that the sign of $z_{\nu}(\theta,x')$ is not locally constant around $x$, hence $x' \mapsto \actneuron_{\nu}(\theta,x')$ is not locally constant around $x$, therefore $(\theta',x') \mapsto \act(\theta',x')$ is not locally constant around $(\theta,x)$. Thus, $x \notin \xsetcont'$.
\end{itemize}
This establishes equivalently that $\xsetcont' \subset \xset_{\underline{\theta}}' \cap \left(\cup_{\nu \in N_{L}}\Gamma_{\nu} \right)^{c}$.

Vice-versa, consider $x \in \xset_{\underline{\theta}}' \cap \left(\cup_{\nu \in N_{L}}\Gamma_{\nu}(\theta) \right)^{c}$.
Since $x \in \xset_{\underline{\theta}}'$, $\actlayer(\underline{\theta}',x')$ is locally constant around $(\underline{\theta},x)$, hence $(\theta',x') \mapsto \actlayer_{\ell}(g(\theta'),x') = \actlayer_{\ell}(\theta',x')$ is locally constant around $(\theta,x)$ for each $1 \leq \ell \leq L-1$. There remains to show that $\actlayer_{L}(\theta',x')$ is locally constant around $(\theta,x)$.
Indeed, since $x \notin \cup_{\nu \in N_{L}} \Gamma_{\nu}$, we have $z_{\nu}(\theta,x) \neq 0$ for every $\nu \in N_{L}$. By continuity of $(\theta',x') \mapsto \vz_{L}(\theta',x')$, there exists a neighborhood of $(\theta,x)$ on which $\mathtt{sign}(z_{\nu}(\theta',x'))$ is constant for every $\nu \in N_{L}$, hence $\actlayer_{L}(\theta',x')$ is locally constant around $(\theta,x)$. Overall, we get that $(\theta,x) \mapsto \actlayer(\theta',x')$ is locally constant around $(\theta,x)$, i.e. $x \in \xsetcont'$. This concludes the proof that $\xset_{\underline{\theta}}' \cap \left(\cup_{\nu \in N_{L}}\Gamma_{\nu}(\theta) \right)^{c} \subset \xsetcont'$, hence the equality $\xsetcont' = \xset_{\underline{\theta}}' \cap \left(\cup_{\nu \in N_{L}}\Gamma_{\nu}(\theta) \right)^{c}$.

\section{Proof of \autoref{le:AthetaShallow}}
\label{app:proofAthetaShallow}

First we prove that $\twinsignature{c} \in \Aspace$ and $(\twinsignature{c},0) \in \Aspacebias$ for each $c$.
Since $\theta$ is admissible, one can check (cf Definition~\ref{def:shallowtwins}) that two hidden neurons $\nu,\nu' \in H$  of a shallow network are:
\begin{itemize}
\item positive twins if, and only if, $\actneuron_{\nu}(\theta,x) = a_{\nu'}(\theta,x)$ for all $x \in \xsetcont$;
\item negative twins if, and only if,  $\actneuron_{\nu}(\theta,x) = 1-\actneuron_{\nu'}(\theta,x)$ for all $x \in \xsetcont$;
\end{itemize}
Since we are on a shallow architecture, we identify $\pathset{Q}=\pathset{Q}_{1}$ with $H$ and $\actlayer(\theta,x)$ with $\vactpath(\theta,x)$. Considering the $c$-th equivalence class $T_c$ of twins, it follows that for every $x$ there is $\epsilon_{c}(x) \in \{-1,+1\}$ such that 
\begin{align}\label{eq:TmpActivationExpression}
2\vactpath_{T_{c}}(\theta,x) = 2\actlayer_{T_{c}}(\theta,x)  =  \indic{T_{c}} + \epsilon_{c}(x) \cdot \twinsignature{c}
\end{align}
where for any $\vu \in \RR^{H}$, $\vu_{T} \in \RR^{H}$ is its restriction to $T$ (which matches $\vu$ on its coordinates indexed by $T$ and is zero elsewhere), and $\indic{H} \in \mathtt{R}^{H}$ is the vector with all entries equal to one, while $\indic{T}$ is its restriction to $T$. To continue we use the following result.
\begin{lemma}\label{le:XpmTwins}
Consider a shallow network with parameter $\theta$, and $T \subset H$ an equivalence class of twin neurons.
There are $x_{T}^{+},x_{T}^{-} \in \xsetcont$ such that 
\begin{equation}\label{eq:XpmTwins}
|\actneuron_{\nu}(\theta,x^{+}_{T})-\actneuron_{\nu}(\theta,x^{-}_{T})| 
= \begin{cases}1,\ \text{if}\ \nu \in T\\ 0,\ \text{otherwise}\end{cases}
\end{equation}
\end{lemma}
\begin{proof}
For each $\nu \in H$ denote $\mathcal{V}_{\nu} = \{x \in \RR^{N_{0}}: \langle \vw_{\bullet\to\nu},x\rangle+b_{\nu}=0\}$.  Since $\theta$ is admissible, $\vw_{\bullet\to\nu}\neq 0$ for each $\nu \in H$, hence these linear spaces are hyperplanes. The hyperplanes associated to two neurons coincide if, and only if, these neurons are twins. 
Choose an arbitrary $\nu \in T$. Since $\mathcal{V}_{\nu}$ is distinct from each of the (finitely many) $\mathcal{V}_{\nu'}$, $\nu' \notin T$, there exists $x_{0} \in \mathcal{V}_{\nu}$ that belongs to the complement of $\cup_{\nu' \notin T} \mathcal{V}_{\nu'}$. As this complement is open, there exists $\epsilon>0$ such that $B(x_{0},\epsilon\|\vw_{\bullet\to\nu}\|_{2})$ does not intersect any of the hyperplanes $\mathcal{V}_{\nu'}$, $\nu' \notin T$. Since $x_{T}^{\pm}:=x_{0}\pm\vw_{\bullet\to\nu}\epsilon/2 \in B(x_{0},\epsilon\|\vw_{\bullet\to\nu}\|_{2})$ we obtain: $a_{\nu'}(\theta,x^{+}_{T}) = a_{\nu'}(\theta,x^{-}_{T})$ for every $\nu' \notin T$, and $\sign{\langle \vw_{\bullet\to\nu},x_{T}^{\pm}\rangle+b_{\nu}} = \pm 1$ hence $a_{\nu}(\theta,x^{+}_{T}) = 1-a_{\nu}(\theta,x^{-}_{T})$. The latter extends to each $\nu' \in T$ by the twin property, and yields the conclusion.
\end{proof}
By \autoref{le:XpmTwins} there are $x^{+}_{c},x^{-}_{c} \in \xsetcont$ such that
\begin{equation}
|\actneuron_{\nu}(\theta,x^{+}_{c})-\actneuron_{\nu}(\theta,x^{-}_{c})| 
= \begin{cases}1,\ \text{if}\ \nu \in T_{c}\\ 0,\ \text{if}\ \nu \in H \backslash T_{c}\end{cases}
\end{equation}
It follows that $\vactpath(\theta,x^{+}_{c})-\vactpath(\theta,x^{-}_{c}) = \vactpath_{T_{c}}(\theta,x^{+}_{c})-\vactpath_{T_{c}}(\theta,x^{-}_{c}) = \pm \twinsignature{c}$. As a result
\begin{align*}
\twinsignature{c} &= \pm \left(\vactpath(\theta,x^{+}_{c})-\vactpath(\theta,x^{-}_{c})\right)  \in \Aspace\\
(\twinsignature{c},0) & = \pm \left(\bvactpath(\theta,x^{+}_{c})-\bvactpath(\theta,x^{-}_{c})\right)  \in \Aspacebias
\end{align*}
as claimed.
Using~\eqref{eq:TmpActivationExpression} and the partition of $H$ into $T_{1},\ldots, T_{C}$ we have for any $x \in \xsetcont$
\begin{align}\label{eq:TmpDecompTwins}
2\vactpath(\theta,x) 
&=
\sum_{c}
2\vactpath_{T_{c}}(\theta,x) 
=
\sum_{c}
(\indic{T_{c}}+\epsilon_{c}(x) \cdot \twinsignature{c}) = \indic{H}+ \sum_{c} \epsilon_{c}(x) \twinsignature{c}
\end{align}
We obtain 
\begin{align}\label{eq:TmpDecompTwins2}
\indic{H} = 2\vactpath(\theta,x) - \sum_{c} \epsilon_{c}(x) \twinsignature{c},
\end{align}
and since $\vactpath(\theta,x) \in \Aspace$ and $\twinsignature{c} \in \Aspace$ for all $c$, it follows that 
$\indic{H} \in \Aspace$. This proves $\linspan{\indic{H},\twinsignature{c},1 \leq c \leq C} \subseteq \Aspace$. Vice-versa, \eqref{eq:TmpDecompTwins} shows $\vactpath(\theta,x) \in \linspan{\indic{H},\twinsignature{c},1 \leq c \leq C}$ for every $x \in \xsetcont$, hence $\Aspace \subseteq \linspan{\indic{H},\twinsignature{c},1 \leq c \leq C}$.
By~\eqref{eq:TmpDecompTwins2} we also get
\[
(\indic{H},2) = 2(\vactpath(\theta,x),1) - \sum_{c} \epsilon_{c}(x) (\twinsignature{c},0),
\]
and since $(\vactpath(\theta,x),1)=\bvactpath(\theta,x) \in \Aspacebias$ and $(\twinsignature{c},0) \in \Aspacebias$, we get $(\indic{H},2) \in \Aspacebias$. This proves $\linspan{(\indic{H},2),(\twinsignature{c},0),1 \leq c \leq C} \subseteq \Aspacebias$, and also implies
\[
2\bvactpath(\theta,x) = (\indic{H},2)+ \sum_{c} \epsilon_{c}(x) (\twinsignature{c},0)
\]
hence $\Aspacebias \subseteq \linspan{(\indic{H},2),(\twinsignature{c},0),1 \leq c \leq C}$.

\section{Proof of \autoref{lem:Case2a}}
\label{app:proofWithTwins}
We use the shorthands $\vw_{\nu} = \vw_{\bullet\to\nu}$, $\vv_{\nu} = \vw_{\nu \to \bullet}$.

Given the assumption there are $C = |H|-1$ classes of twin neurons, all being trivial except one made of a pair of negative twins $\{\nu,\nu'\}$. Without loss of generality we enumerate the neurons and their classes such that $T_{1} = \{\nu_{1},\nu_{2}\} = \{\nu,\nu'\}$ and $T_{c} = \{\nu_{c+1}\}$, $2 \leq c \leq C = |H|-1$. 
First we establish that, with this numbering,
\begin{equation}\label{eq:TmpShallow1NegTwins}
\Aspacebiasperp = \linspan{(1,1,0,\ldots,0,-1)}\ \text{and}\ \Aspaceperp = \{0\}.
\end{equation}
The signatures of the classes are $\twinsignature{1} = \delta_{1}-\delta_{2}$ and $\twinsignature{c} = \delta_{c+1}$, $2 \leq c \leq C$. By \autoref{le:AthetaShallow} we have $\Aspace = \linspan{\indic{H},\twinsignature{c},1 \leq c \leq C}$. It is not difficult to check
\footnote{If, instead of a single pair of negative twins, we consider a single pair of 
{\em positive} twins, then $\twinsignature{1} = \delta_{1}+\delta_{2}$ and the spanning vectors of $\Aspace$ become linearly \emph{dependent}, with 
$\Aspaceperp = \linspan{(1,-1,0,\ldots,0)} \neq \{0\}$.
} 
that the $C+1 = |H|$ spanning vectors are linearly independent, hence $\Aspace = \RR^{H}$ and $\Aspaceperp = \{0\}$. Now, consider $\vv = (v_{1},\ldots,v_{|H|+1}) \in \Aspacebiasperp$. By \autoref{le:AthetaShallow}, this vector is orthogonal to each $(\twinsignature{c},0)$, $1 \leq c \leq C$, and to $(\indic{H},2)$.
For $2 \leq c \leq C$, orthogonality to $(\twinsignature{c},0) = (\delta_{c+1},0)$  implies $v_{c+1}=0$, hence $\vv = (\alpha,\beta,0,\ldots,0,\gamma)$ for some $\alpha,\beta,\gamma \in \RR$. Orthogonality to $(\twinsignature{1},0) = (1,-1,0,\ldots,0)$ implies $\beta=\alpha$, and orthogonality to $(\indic{H},2)$ implies $\gamma = -\alpha$, hence $\vv$ is proportional to $(1,1,0,\ldots,0,-1)$ as claimed. Since $\Aspacebias$ is spanned by $C+1 = |H|$ vectors, its dimension is at most $|H|$, hence the dimension of $\Aspacebiasperp$ is at least one. This concludes the proof that $\Aspacebiasperp = \linspan{(1,1,0,\ldots,0,-1)}$.

Since $\theta$ is admissible, there is an input neuron $\mu \in N_{0}$ such that $w_{\mu \to \nu_{1}} \neq0$. Since $\nu_{1}$,$\nu_{2}$ are twins, we also have $w_{\mu \to \nu_{2}} \neq 0$. Let $\epsilon_{0} := \min_{1 \leq j \leq 2} \abs{w_{\mu \to \nu_{j}}}/2$. Consider  $\theta' \in B(\theta,\epsilon_{0})$ such that $\rmap(\theta')-\rmap(\theta) \in \linspace{V}(\theta)$. First, observe that $w'_{\mu\to\nu_{j}} \neq 0$ for $j=1,2$. Then, in light of \autoref{lem:VCartesian} and~\eqref{eq:TmpShallow1NegTwins}, for every $\eta \in N_{2}$, we have $\rmapi_{\eta}(\theta')=\rmapi_{\eta}(\theta)$ and
\begin{equation}\label{eq:TmpShallow1NegTwins2}
\rmaph_{\eta}(\theta')-\rmaph_{\eta}(\theta) \in \linspan{(1,1,0,\ldots,0,-1)},
\end{equation}
hence $b'_{\nu_{c+1}}w'_{\nu_{c+1}\to\eta}=b_{\nu_{c+1}}w_{\nu_{c+1}\to\eta}$ for $2 \leq c \leq C$, and there are scalars $\lambda_{\eta} \in \RR$ such that 
\begin{align}
b'_{\nu_{j}}w'_{\nu_{j}\to\eta}-b_{\nu_{j}} w_{\nu_{j}\to\eta} & = \lambda_{\eta},\ \forall j \in \{1,2\}
\qquad\text{and}\ b'_{\eta}-b_{\eta} = -\lambda_{\eta}\label{eq:TmpShallow1NegTwins3}.
\end{align}

When $\Theta$ is the set of parameters with zero output biases, the fact that $\theta,\theta' \in \Theta$ implies $b'_{\eta}=b_{\eta}=0$, hence $\lambda_{\eta}=0$ and $\rmaph_{\eta}(\theta')=\rmaph_{\eta}(\theta)$  for every $\eta \in N_{2}$. We show below that the same holds for arbitrary $\Theta$ when $\vw_{\nu_{1}\to \bullet}$ and $\vw_{\nu_{2}\to\bullet}$ are linearly independent. This implies $\rmap(\theta')=\rmap(\theta)$, hence $\theta$ is then $\epsilon$-non-degenerate with respect to $\Theta$.

Indeed, the equality $\rmapi_{\eta}(\theta')=\rmapi_{\eta}(\theta)$ for all $\eta \in N_{2}$ implies that  for $1 \leq j \leq 2$,
\begin{align}
w'_{\mu\to\nu_{j}}w'_{\nu_{j}\to\eta}&=w_{\mu\to\nu_{j}} w_{\nu_{j}\to\eta},\ \forall \eta \in N_{2},\label{eq:TmpShallow1NegTwins5}
\end{align}
and since $w'_{\mu\to\nu_{j}} \neq 0$ for $j=1,2$, we obtain
from~\eqref{eq:TmpShallow1NegTwins3} and~\eqref{eq:TmpShallow1NegTwins5} that for each $\eta \in N_{2}$, 
\begin{align*}
\lambda_{\eta} &
= b'_{\nu_{j}} w'_{\nu_{j}\to\eta}-b_{\nu_{j}}w_{\nu_{j}\to\eta}
= b'_{\nu_{j}} \frac{w'_{\mu\to\nu_{j}}w'_{\nu_{j}\to\eta}}{w'_{\mu\to\nu_{j}}}-b_{\nu_{j}}w_{\nu_{j}\to\eta}\\
&= b'_{\nu_{j}} \frac{w_{\mu\to\nu_{j}}w_{\nu_{j}\to\eta}}{w'_{\mu\to\nu_{j}}}-b_{\nu_{j}}w_{\nu_{j}\to\eta}
= \left(b'_{\nu_{j}} \frac{w_{\mu\to\nu_{j}}}{w'_{\mu\to\nu_{j}}}-b_{\nu_{j}}\right)w_{\nu_{j}\to\eta}
\end{align*}
We obtain $\vec{\lambda} = x_{j} \vw_{\nu_{j}\to\bullet}$, $j=1,2$ where $\vec{\lambda} := (\lambda_{\eta})_{\eta \in N_{2}}$ and $x_{j} := b'_{\nu_{j}} \frac{w_{\mu\to\nu_{j}}}{w'_{\mu\to\nu_{j}}}-b_{\nu_{j}}$. Since $\vw_{\nu_{1}\to \bullet}$ and $\vw_{\nu_{2}\to\bullet}$ are linearly independent, it follows that $x_{1}=x_{2}=0$, hence $\vec{\lambda}=\vec{0}$.

Assume now that $\vw_{\nu_{1}\to \bullet}$ and $\vw_{\nu_{2}\to\bullet}$ are linearly \emph{dependent}, and recall that since $\theta$ is admissible they are both nonzero vectors, hence $\vw_{\nu_{2}\to\bullet}=\alpha \vw_{\nu_{2}\to\bullet}$ for some $\alpha \neq 0$. Consider $0<\epsilon<\epsilon_{0}$ and set $\theta'$ as follows: 
\begin{align*}
\mat{W}'_{\ell}&=\mat{W}_{\ell},\ 1 \leq \ell \leq 2;\\
b'_{\nu}&=b_{\nu},\ \nu \in H \backslash \{\nu_{1},\nu_{2}\};\\
b'_{\nu_{1}} &= b_{\nu_{1}}+\gamma\epsilon;\\
b'_{\nu_{2}} &= b_{\nu_{2}}+\gamma\epsilon/\alpha, j =1,2;\\
b'_{\eta}&=b_{\eta}-w_{\nu_{1}\to\eta}\gamma\epsilon, \eta \in N_{2},
\end{align*}
with $0<\gamma< \min(1,|\alpha|,1/\|\vw_{\nu_{1}\to\bullet}\|_{\infty})$ so that $\theta' \in B(\theta,\epsilon)$.
Since the weights of $\theta'$ and $\theta$ coincide we have $\rmapi_{\eta}(\theta')=\rmapi_{\eta}(\theta)$ for every $\eta \in N_{2}$. It is not difficult to check that, with $\lambda_{\eta} := w_{\nu_{1}\to\eta}\epsilon$, we also have $\rmaph_{\eta}(\theta')-\rmaph_{\eta}(\theta) = \lambda_{\eta} (1,1,0,\ldots,0,-1)$, hence 
$\rmap(\theta')-\rmap(\theta) \in \linspace{V}(\theta)$. Yet, $\rmap(\theta') \neq \rmap(\theta)$ since $\vec{\lambda} := (\lambda_{\eta})_{\eta \in N_{2}} = \epsilon \vw_{\nu_{1}\to \bullet} \neq \vec{0}$. Assuming that $\theta$ belongs to the interior of $\Theta$, we have $\theta' \in \Theta \cap B(\theta,\epsilon)$ for small enough $\epsilon$. It follows that $\theta$ is degenerate.

\section{Details on \autoref{ex:LSInotND}} \label{app:AbsReLU} 

{\bf $\theta$ is PS-identifiable from $\xset = \RR$.} 
Consider an arbitrary $\theta' \in \Theta = \RR^{\paramset}$. If $\realiz{\theta'}(x) = \realiz{\theta}(x) = \abs{x}$ on $\RR$ then $\theta'$ is admissible (otherwise its realization would be, up to an additive constant, proportional to a single shifted version of the ReLU, which would prevent it from being equal to $\realiz{\theta} = \mathtt{abs}$) hence $w'_{\nu\to\nu_{i}}\neq0$, $i=1,2$. Writing $\alpha_{i}= \abs{w'_{\mu\to\nu_{i}}}\ w'_{\nu_{i}\to\eta}$ and $\beta_{i}=-b'_{\nu_{i}}/\abs{w'_{\mu\to\nu_{i}}}$, and $s_{i} = \sign{w'_{\mu\to\nu_{i}}} \in \{-1,+1\}$ for $i=1,2$, we have $\alpha_{i}\neq 0$ and
\[
\realiz{\theta'}(x) = \alpha_{1}\ReLU(s_{1}(x-s_{1}\beta_{1}))+ \alpha_{2}\ReLU(s_{2}(x-s_{2}\beta_{2}))+b'_{\eta},\ \forall x \in \RR.
\]
If we had $s_{1}\beta_{1} \neq s_{2} \beta_{2}$, $\realiz{\theta'}$ would be non-differentiable at two distinct points $s_{1}\beta_{1},s_{2}\beta_{2}$. However $\realiz{\theta'}=\realiz{\theta} = \mathtt{abs}$ is differentiable on $\RR \backslash\{0\}$, hence $s_{1}\beta_{1}=s_{2}\beta_{2}$, and a similar reasoning yields $s_{1}\beta_{1}=s_{2}\beta_{2}=0$. Since $\abs{s_{i}} = 1$, we get $\beta_{1}=\beta_{2}=0$ and
\[
\realiz{\theta'}(x) = \alpha_{1}\ReLU(s_{1}x)+ \alpha_{2}\ReLU(s_{2}x)+b'_{\eta},\ \forall x \in \RR.
\]
If we had $s_{1}=s_{2}$, the realization would be $(\alpha_{1}+\alpha_{2})\ReLU(s_{1}x)+b'_{\eta}$, which cannot match $\mathtt{abs}$, hence $s_{2}=-s_{1}$. Without loss of generality (up to a permutation of indices of the hidden layer) $s_{1}=1$, $s_{2}=-1$. Now, for $x<0$ we have $-x = \abs{x} = \realiz{\theta'}(x) = -\alpha_{2}x+b'_{\eta}$ while for $x>0$ we get $x = \abs{x} = \realiz{\theta'}(x)  = \alpha_{1}x+ b_{\eta}$, hence $\alpha_{1}=1$, $\alpha_{2}=1$, $b'_{\eta}=0$. Overall, up to the possible permutation of the hidden layer, we obtain $\sign{\theta'} = \sign{\theta}$ and $\alpha_{1}s_{1}=1$, $\alpha_{2}s_{2}=-1$, $\alpha_{1}\beta_{1}=\alpha_{2}\beta_{2}=0$, $b'_{\eta}=0$, hence $\rmap(\theta') = \rmap(\theta)$. Since $\theta$ is admissible, it follows by \autoref{lem:EqualRepAndSignImpliesSEquiv} that $\theta' \sim_{PS} \theta$. Since this holds for any $\theta'$ such that $\realiz{\theta'} = \realiz{\theta}$, this shows that $\theta$ is PS-identifiable from $\xset = \RR$ with respect to $\Theta = \RR^{\paramset}$.

{\bf $\theta$ is locally S-identifiable from some finite set $F \subset \RR$ (with $0 \in F$)}

With the same notations as above, observe that there is $\epsilon>0$ such that for every  $\theta' \in B(\theta,\epsilon)$ we have $s_{i}:= \sign{w'_{\mu \to \nu_{i}}}= \sign{w_{\mu\to\nu_{i}}}$, $i=1,2$, and $\max(|\alpha_{1}-1|,|\alpha_{2}-1|,|\beta_{1}|,|\beta_{2}|,|b'_{\eta}|) \leq 1/2$. Consider $\theta' \in B(\theta,\epsilon)$ such that $\realiz{\theta'}=\realiz{\theta'}$ on $F = \{-3,-2,-1,0,1,2,3\}$. 
We have $s_{1}=+1$, $s_{2}=-1$ hence
\[
\realiz{\theta'}(x) = \alpha_{1} \ReLU(x-\beta_{1})+\alpha_{2}\ReLU(-x-\beta_{2}) + b'_{\eta}.
\]
Since $|\beta_{i}| \leq 1/2$ for $i=1,2$, we have
$\realiz{\theta'}(x) = \alpha_{1}(x-\beta_{1}) + b'_{\eta}$
for every $x \geq 1/2$, hence $\alpha_{1}(y-x)=\realiz{\theta'}(y)-\realiz{\theta'}(x) = \realiz{\theta}(y)-\realiz{\theta}(x) = \abs{y}-\abs{x} = y-x$ for $(x,y) = (1,2)$. Therefore $\alpha_{1}=1$. A similar reasoning
with $(x,y) = (-2,-1)$ shows that $\alpha_{2}=1$, hence 
\[
\realiz{\theta'}(x) = \ReLU(x-\beta_{1})+\ReLU(-x-\beta_{2}) + b'_{\eta}
\]
Specializing to $x=1$, since $x-\beta_{1}>0$ and $-x-\beta_{2}<0$ we get $1 = \abs{x} = \realiz{\theta'}(x) = x-\beta_{1}+b'_{\eta} = 1-\beta_{1}+b'_{\eta}$ hence $b'_{\eta}=\beta_{1}$. Similarly, with $x=-1$, we get $b'_{\eta}=\beta_{2}$ hence 
\[
\realiz{\theta'}(x) = \ReLU(x-b'_{\eta})+\ReLU(-x-b'_{\eta})+b'_{\eta}.
\]
Specializing to $x = 0 \in F$ yields 
\[
0 = \abs{x} = \realiz{\theta'}(x) = b'_{\eta}+2 \ReLU(-b'_{\eta}) = 
\begin{cases}
b'_{\eta}& \text{if}\ b'_{\eta} \geq 0\\
-b'_{\eta}& \text{if}\ b'_{\eta} \leq 0
\end{cases}
= \abs{b'_{\eta}}
\] 
hence $b'_{\eta}=0$. Overall we have shown that for every $\theta' \in B(\theta,\epsilon)$ such that 
$\realiz{\theta'} = \realiz{\theta}$ on $F = \{-2,-1,0,1,2\}$ we have $\alpha_{1}=\alpha_{2}=1$, $s_{1}=1,s_{2}=-1$, $\beta_{1}=\beta_{2}=b'_{\eta}=0$. These imply $\rmap(\theta')=\rmap(\theta)$ and $\sign{\theta'}=\sign{\theta}$ hence $\theta' \sim_{S} \theta$. In other words, $\theta$ is locally S-identifiable from $F$.

\section{Details on \autoref{ex:1twinpairbis}}\label{app:1twinpairbis}

Here we show, as claimed in  \autoref{ex:1twinpairbis} that the parameter $\theta_{0} \in \RR^{\paramset}$ from \autoref{ex:1twinpair} is PS-identifiable from $\xset = \RR$ with respect to the set $\Theta_{0}$ of parameters with zero output biases. Consider an arbitrary $\theta' \in \Theta_{0}$. If $\realiz{\theta'}(x) = \realiz{\theta_{0}}(x) = x$ on $\xset$ then $\theta'$ is admissible (otherwise its realization would be, up to an additive constant, proportional to a single shifted version of the ReLU, which would prevent it from being equal to $\realiz{\theta} = \mathtt{id}$) hence $w'_{\nu\to\nu_{i}}\neq0$, $i=1,2$. Writing $\alpha_{i}= \abs{w'_{\mu\to\nu_{i}}}\ w'_{\nu_{i}\to\eta}$ and $\beta_{i}=-b'_{\nu_{i}}/\abs{w'_{\mu\to\nu_{i}}}$, and $s_{i} = \sign{w'_{\mu\to\nu_{i}}} \in \{-1,+1\}$ for $i=1,2$, we have $\alpha_{i}\neq 0$ and since the output bias is zero
\[
\realiz{\theta'}(x) = \alpha_{1}\ReLU(s_{1}(x-s_{1}\beta_{1}))+ \alpha_{2}\ReLU(s_{2}(x-s_{2}\beta_{2})),\ \forall x \in \xset.
\]
If we had $s_{1}\beta_{1} \neq s_{2} \beta_{2}$, $\realiz{\theta'}$ would be non-differentiable at two distinct points $s_{1}\beta_{1},s_{2}\beta_{2}$. However $\realiz{\theta'}=\realiz{\theta_{0}} = \mathtt{id}$ is differentiable on $\RR$, hence $s_{1}\beta_{1}=s_{2}\beta_{2}$. It follows that $s_{2}\beta_{1}= s_{1}\beta_{1} = \mathtt{id}(s_{1}\beta_{1}) = \realiz{\theta'}(s_{1}\beta_{1}) = 0$. Since $\abs{s_{i}} = 1$, we get $\beta_{1}=\beta_{2}=0$ and
\[
\realiz{\theta'}(x) = \alpha_{1}\ReLU(s_{1}x)+ \alpha_{2}\ReLU(s_{2}x),\ \forall x \in \xset.
\]
If we had $s_{1}=s_{2}$, the realization would be $(\alpha_{1}+\alpha_{2})\ReLU(s_{1}x)$, which cannot match $\mathtt{id}$ on $\xset$, hence $s_{2}=-s_{1}$. Without loss of generality (up to a permutation of indices of the hidden layer) $s_{1}=1$, $s_{2}=-1$. Now, for $x<0$ we have $x = \mathtt{id}(x) = \realiz{\theta'}(x) = -\alpha_{2}x$ while for $x>0$ we get $x = \mathtt{id}(x) = \realiz{\theta'}(x)  = \alpha_{1}x$, hence $\alpha_{1}=1$, $\alpha_{2}=-1$ (and $b'_{\eta}=0$ because $\theta' \in \Theta_{0}$). Overall, up to the possible permutation of the hidden layer, we obtain $\sign{\theta'} = \sign{\theta_{0}}$ and $\alpha_{1}s_{1}=1$, $\alpha_{2}s_{2}=-1$, $\alpha_{1}\beta_{1}=\alpha_{2}\beta_{2}=0$, $b'_{\eta}=0$, hence $\rmap(\theta') = \rmap(\theta_{0})$. Since $\theta_{0}$ is admissible, it follows by \autoref{lem:EqualRepAndSignImpliesSEquiv} that $\theta' \sim_{PS} \theta_{0}$. Since this holds for any $\theta' \in \Theta_{0}$ such that $\realiz{\theta'} = \realiz{\theta_{0}}$, this shows that $\theta$ is PS-identifiable from $\xset = \RR$ with respect to $\Theta_{0}$.

\end{document}